\documentclass{article}
\usepackage{amsmath,amsfonts,bm}
\usepackage{xspace}

\newcommand{\method}{{\sf USD3}\xspace}
\newcommand{\methodce}{{\sf USD3-CE}\xspace}
\newcommand{\methodvlb}{{\sf USD3-VLB}\xspace}
\newcommand{\methodsim}{{\sf USD3$^\ast$}\xspace}

\newcommand{\tauldr}{{$\tau$-LDR}\xspace}
\newcommand{\sddm}{{SDDM}\xspace}
\newcommand{\dpm}{{D3PM}\xspace}
\newcommand{\rdm}{{RDM}\xspace}
\newcommand{\sedd}{{SEDD}\xspace}

\newcommand{\pone}{{\sc Piano-P}\xspace}
\newcommand{\ptwo}{{\sc Piano}\xspace}
\newcommand{\cifar}{{\sc VQCIFAR10}\xspace}

\newcommand{\beq}{\begin{equation}}
	\newcommand{\eeq}{\end{equation}}
 
\def\figref#1{Fig.~\ref{#1}}

\def\eqref#1{Eq.~(\ref{#1})}

\def\1{\bm{1}}

\def\rvb{{\mathbf{b}}}

\def\rvm{{\mathbf{m}}}

\def\rvx{{\mathbf{x}}}

\def\rvz{{\mathbf{z}}}

\def\ve{{\bm{e}}}

\def\vm{{\bm{m}}}

\def\vp{{\bm{p}}}

\def\vx{{\bm{x}}}
\def\vy{{\bm{y}}}
\def\vz{{\bm{z}}}

\DeclareMathAlphabet{\mathsfit}{\encodingdefault}{\sfdefault}{m}{sl}
\SetMathAlphabet{\mathsfit}{bold}{\encodingdefault}{\sfdefault}{bx}{n}

\def\gM{{\mathcal{M}}}

\def\gS{{\mathcal{S}}}

\newcommand{\E}{\mathbb{E}}
\newcommand{\Ls}{\mathcal{L}}
\newcommand{\R}{\mathbb{R}}

\newcommand{\KL}{D_{\mathrm{KL}}}

\newcommand\ci{\perp\!\!\!\perp}

    \PassOptionsToPackage{compress}{natbib}

    \usepackage[preprint]{neurips_2024}

\usepackage[utf8]{inputenc} 
\usepackage[T1]{fontenc}    
\usepackage{hyperref}       
\hypersetup{
    colorlinks=true,
    linkcolor=blue,
    citecolor=blue,
    filecolor=magenta,      
    urlcolor=cyan,
}
\usepackage{url}            
\usepackage{booktabs}       
\usepackage{amsfonts}       
\usepackage{nicefrac}       
\usepackage{microtype}      
\usepackage{xcolor}         
\usepackage{graphicx}

\usepackage{amssymb,amsmath,amsthm,enumitem}
\usepackage{mathtools}
\theoremstyle{plain}
\newtheorem{theorem}{Theorem}[section]
\newtheorem{proposition}[theorem]{Proposition}

\theoremstyle{definition}

\theoremstyle{remark}

\usepackage{algorithm}
\usepackage{algorithmic}
\usepackage{titlesec}
\titlespacing*{\subsection}{0pt}{0\baselineskip}{0\baselineskip}
\titlespacing*{\subsubsection}{0pt}{0\baselineskip}{0\baselineskip}
\titlespacing*{\section}{0pt}{0\baselineskip}{0\baselineskip}
\everydisplay{
\setlength{\abovecaptionskip}{2pt}
\setlength{\abovedisplayskip}{2pt}
\setlength{\belowdisplayskip}{3pt}
\setlength{\belowcaptionskip}{0pt}
}
\usepackage{caption}
\usepackage{subcaption}
\usepackage{multirow}

\newcommand{\first}[1]{\textbf{\textcolor{red}{#1}}}
\newcommand{\second}[1]{\textbf{\textcolor{violet}{#1}}}

\title{Unified Discrete Diffusion for Categorical Data}

\author{%
  Lingxiao Zhao,\  Xueying Ding \thanks{Equal contribution.}\\
  \small Carnegie Mellon University\\
  \texttt{\small lingxiaozlx@gmail.com}\\
  \texttt{\small xding2@andrew.cmu.edu}\\
  \And
  Lijun Yu\\
  \small Carnegie Mellon University\\
  \texttt{\small lijun@lj-y.com} \\ 
  \And
  Leman Akoglu\\
 \small  Carnegie Mellon University\\
  \texttt{\small lakoglu@andrew.cmu.edu} \\ 
}

\begin{document}

\maketitle

\begin{abstract}
Discrete diffusion models have attracted significant attention for their application to naturally discrete data, such as language and graphs. While discrete-time discrete diffusion has been established for some time, it was only recently that \citet{campbell2022continuous} introduced the first framework for continuous-time discrete diffusion. However, their training and backward sampling processes significantly differ from those of the discrete-time version, requiring nontrivial approximations for tractability. In this paper, we first introduce a series of generalizations and simplifications of the variational lower bound (VLB) that facilitate more accurate and easier optimization both discrete- and continuous-time discrete diffusion. 
We further establish a unification of discrete- and continuous-time discrete diffusion through shared forward process and backward parameterization. Thanks to this unification, the continuous-time diffusion can now utilize the exact and efficient backward process developed for the discrete-time case, avoiding the need for costly and inexact approximations. Similarly, the discrete-time diffusion now also employ the MCMC corrector, which was previously exclusive to the continuous-time case. Extensive experiments and ablations demonstrate the significant improvement.
\end{abstract}

\section{Introduction}


{Deep generative models have taken the world by storm, 
capturing complex data distributions and producing realistic data, from human-like text \citep{brown2020language,li2022diffusion,OpenAI_GPT4_2023} and natural looking images \citep{dhariwal2021diffusion,ramesh2022hierarchical,zhang2023text} to  novel compounds like molecules and drugs \citep{kang2018conditional,li2021structure} and video synthesis \citep{ho2022imagen}.  
Denoising diffusion models \citep{ho2020denoising}, a powerful class of generative models, 
are trained through a forward diffusion process that gradually adds noise to the training samples, and a backward process that denoises these diffusion trajectories. New data are then generated by sampling from the noise distribution and employing the trained model for recursive denoising.

Discrete diffusion for categorical data has two modeling paradigms: discrete-time and continuous-time. 
The former discretizes time such that  backward denoising 
 is learned only at pre-specified time points. This limits generation, which can only  ``jump back'' through fixed points. In contrast, continuous-time case 
 allows a path  through any point in range,  
 and often yields higher sample quality. 
Current literature on discrete-time discrete diffusion is relatively established, while only recently 
\citet{campbell2022continuous} introduced the first continuous-time discrete diffusion framework. While groundbreaking, their  
their loss requires multiple network evaluations during training.
Moreover, the exact sampling through their learned 
backward process is extremely tedious for multi-dimensional variables.
Due to mathematically complicated and computationally demanding formulations,  \citet{campbell2022continuous} propose nontrivial
approximations for tractability with unknown potential errors.

In this paper, we introduce a series of improvements to two critical aspects of discrete diffusion: loss computation and backward sampling, for both discrete- and continuous-time scenarios. What is more, for the first time, we show that the discrete-time and continuous-time discrete diffusion can be unified together such that they share exactly the same forward diffusion and backward sampling process. We elaborate on the details and their benefits as follows:
\begin{itemize}[nosep, left=5pt,itemsep=4pt,leftmargin=1em,labelwidth=*,align=left]
    \item \textbf{Loss simplifcation and generalization}:  We derive significantly simplified yet exact VLB calculations for both discrete\&continuous-time discrete diffusion,  taking into account the properties of categorical data. Our simplified formulations allow both forward and backward probabilities to accommodate any noise distribution element-wisely, which is particularly attractive for multi-element objects where each element can exhibit a  different noise distribution. 
    
    \item \textbf{Efficient backward sampling}: We present a closed-form backward probability \( p_\theta(\vx_s | \vx_t) \) for all \( s < t \) with reconstruction-based model parameterization, paving the way for VLB simplification and  accelerated backward sampling, initially derived in the discrete-time case. Furthermore, we demonstrate that this improved backward process can be applied to continuous-time case, eliminating the need for the costly and inexact sampling approximations in \cite{campbell2022continuous}.
    
    \item \textbf{Unification:} We demonstrate that continuous-time diffusion can share the exact same forward process as discrete-time diffusion for \textit{any} noise distribution. This unified forward process, along with shared backward parameterization, leads to a unification of both forward and backward processes for discrete-\& continuous-time cases. This offers mutual benefits: discrete-time diffusion can utilize continuous time's MCMC corrector to enhance performance, while continuous-time diffusion can employ fast and exact backward sampling derived in discrete-time case.
\end{itemize}

We present both discrete-time (\S\ref{sec:dd}) and continuous-time (\S\ref{sec:cd}) discrete diffusion in a self-contained manner, accompanied by easy-to-use open-source code at {\small\url{https://github.com/LingxiaoShawn/USD3}}. 
We explicitly mark our novel contributions in (sub)sections with summarization.

\section{Discrete-time Discrete Diffusion}
\label{sec:dd}
\textbf{Notation:} Let $\rvx_0 \sim p_{\text{data}}(\rvx_0)$ be the random variable of observed data with underlying distribution $p_{\text{data}}(\rvx_0)$. Let $\rvx_t \sim q({\rvx_t})$ be the latent variable at time $t$ of a single-element object, like a pixel of an image or a node/edge of a graph, with maximum time $T$. 
Let $\rvx_{t|s} \sim q(\rvx_t| \rvx_s)$ be the conditional random variable. 
We model the \textit{forward diffusion} process independently for each element of the object, while the \textit{backward denoising} process is modeled jointly for all elements of the object. For simplicity and clarity of presentation, we first assume that the object only has 1 element and extend to multi-element object later. Let $\rvx_0^{1:D}$ denote the object with $D$ elements, and  $\rvx^i_t$ be the $i$-th element of latent object at time $t$. We assume all random variables take categorical values from $\{1,2,...,K\}$. Let $\ve_k \in \{0,1\}^K $ be the one-hot encoding of category $k$. For a random variable $\rvx$, we use $\vx$ denoting its one-hot encoded sample where $\vx \in \{\ve_1,...,\ve_K \}$. Also, we interchangeably use $q(\rvx_t| \rvx_s)$, $q(\rvx_t=\vx_t| \rvx_s=\vx_s)$, and $q_{t|s}(\vx_t| \vx_s)$ when no ambiguity is introduced.
Let $\langle \cdot, \cdot \rangle $ denote inner product. All vectors are column-wise vectors. 

\subsection{Graphical Model View of Diffusion Models}
\label{ssec:prelim}

Diffusion models
can be represented by latent variable graphical models (see Appx. \figref{fig:pgm}).
We can write the joint probability as $p_{\theta}(\rvx_{0:T}):= p_{\theta}(\rvx_0, \rvx_1, ..., \rvx_T) = p_{\theta}(\rvx_T)\prod_{t=1}^T p_{\theta}(\rvx_{t-1} | \rvx_{t})$ using the Markov condition.  Parameters $\theta$ are learned by maximizing the loglikehood of the observed variable $\rvx_0$: $\log p_{\theta} (\rvx_0) = \log \int p_{\theta}(\rvx_{0:T}) d\rvx_{1:T} $. However the marginalization is intractable, and instead the following variational lower bound (VLB) is used.
\begin{equation}
\scalebox{0.94}{$
\begin{aligned}
     &\log p_{\theta}(\rvx_0) =\log \int q(\rvx_{1:T}|\rvx_0) \frac{p_{\theta}(\rvx_{0:T})}{q(\rvx_{1:T}|\rvx_0)}d\rvx_{1:T} \label{eq:vlb_0} 
     \geq \E_{q(\rvx_{1:T}|\rvx_0)} \big[\log p_{\theta}(\rvx_{0:T}) - \log q(\rvx_{1:T} | \rvx_0)] .
\end{aligned}
$}
\end{equation}
\noindent The above inequality holds for any conditional probability $q(\rvx_{1:T}|\rvx_0)$ and finding the best $q(\rvx_{1:T}|\rvx_0)$ to tighten the bound is  the inference problem in graphical models (i.e. E step in EM algorithm). Exact inference is intractable, thus $q(\rvx_{1:T}|\rvx_0)$ in diffusion models  is fixed or chosen specifically to simplify the learning objective. To simplify 
\eqref{eq:vlb_0}, it is important to assume $q(\rvx_{1:T}|\rvx_0)$ is \textit{decomposable}. The typical assumption is  $q(\rvx_{1:T}|\rvx_0) = \prod_{t=1}^T q(\rvx_t|\rvx_{t-1})$ \citep{DDPM}, which we also adopt. Others that have been explored include $q(\rvx_{1:T}|\rvx_0) = \prod_{t=1}^T q(\rvx_t|\rvx_{t-1}, \rvx_{0})$ \citep{DDIM}.

Assuming $q(\rvx_{1:T}|\rvx_0) = \prod_{t=1}^T q(\rvx_t|\rvx_{t-1})$,  \eqref{eq:vlb_0} can be simplified as the following
\begin{equation}
\scalebox{0.85}{$
\begin{aligned}
&\underbrace{\E_{q(\rvx_{1}|\rvx_0)}\big[ \log p_{\theta}(\rvx_{0}| \rvx_{1} )\big]}_{-\Ls_1(\theta)} 
      - \underbrace{\KL\big(q(\rvx_{T} |\rvx_0 ) || p_{\theta}(\rvx_{T}) \big)}_{\Ls_{\text{prior} }} 
  -\sum_{t=2}^T \underbrace{\E_{q(\rvx_{t}|\rvx_0)}\big[ \KL\big(q(\rvx_{t-1} | \rvx_{t}, \rvx_0 ) || p_{\theta}(\rvx_{t-1}| \rvx_{t} \big)  \big]}_{\Ls_t(\theta)} 
       \label{eq:vlb}
\end{aligned}
$}
\end{equation}
Where $\Ls_{\text{prior}}\approx 0$, since $p_{\theta}(\rvx_{T}) \approx q(\rvx_{T} |\rvx_0 )$ is designed as a  fixed noise distribution that is easy to sample from. (See Appx. \S\ref{ssec:vlbderive} for derivation.) To compute \eqref{eq:vlb}, we need to formalize distributions ($i$) 
 $q(\rvx_t| \rvx_0)$ and ($ii$)  $q(\rvx_{t-1} | \rvx_t, \rvx_0)$, as well as ($iii$) the parameterization of $p_{\theta}(\rvx_{t-1}| \rvx_{t} )$. We specify these respectively in \S\ref{ssec:frwd}, \S\ref{ssec:backw}, and \S\ref{ssec:param}. 

\subsection{Components of Discrete-time Discrete Diffusion}
After reviewing the forward process for discrete diffusion  (\S\ref{ssec:frwd}),  we  contribute a series of analytical simplifications for various components (\S\ref{ssec:backw}, \S\ref{ssec:param}) of the VLB, providing exact closed-form formulation in \S\ref{ssec:vlbloss}, as well as an approximated loss for easier optimization in \S\ref{ssec:loss_approx}. We present fast backward sampling in Appx.\S\ref{ssec:reparameter}, and give the extension to multi-element case in Appx. \S\ref{ssec:discrete_multi_element_app}. 

\subsubsection{The Forward  Diffusion Process}
\label{ssec:frwd}

We assume each discrete random variable $\rvx_t$ has a categorical distribution, i.e. $\rvx_t \sim \text{Cat}(\rvx_t ; \vp)$ with $\vp\in [0,1]^{K}$ and $\1^\top \vp = 1$ . One can verify that $p(\rvx_t= \vx_t) = \vx_t^\top\vp$, or simply $p(\rvx_t) = \vx_t^\top\vp$. As shown in \citep{hoogeboom2021argmax, D3PM}, the forward process with discrete variables $q(\rvx_t | \rvx_{t-1})$ at $t$ step can be represented as a transition matrix  $Q_t \in [0,1]^{K\times K}$ such that $[Q_{t}]_{ij} = q(\rvx_t = \ve_j | \rvx_{t-1} = \ve_i)$. Then, we can write the distribution explicitly as 
\begin{align}
    q(\rvx_t | \rvx_{t-1}) = \text{Cat}(\rvx_t; Q^\top_t\vx_{t-1}) \;.
\end{align}
Given transition matrices $Q_1,...,Q_T$,  the $t$-step  marginal distribution conditioning on $s, \forall t>s$ is  
\begin{align}
\hspace{-0.1in}
    q(\rvx_t | \rvx_s) = \text{Cat}(\rvx_t; \overline{Q}_{t|s}^\top\vx_s ), \text{ with } \overline{Q}_{t|s} = Q_{s+1}...Q_{t} \; .\label{eq:Qbar_ts}
\end{align} 
The $s$-step posterior distribution conditioning on $\rvx_0$ and $t$-step  can be derived as
\begin{align}
    q(\rvx_{s} | \rvx_t, \rvx_0) =\frac{q(\rvx_t | \rvx_{s}) q(\rvx_{s}| \rvx_0)}{q(\rvx_t | \rvx_0)} = \text{Cat}(\rvx_{s};\frac{\overline{Q}_{t|s} \vx_{t} \odot \overline{Q}_{s}^\top\vx_0}{\vx_t^\top\overline{Q}_t^\top\vx_0 }),  \ \forall t > s \; . \label{eq:t|s,0}
\end{align}
The above formulations are valid (derivation in Appx. \S\ref{ssec:t-1|t,0} ) for \textit{any} transition matrices $Q_1,...,Q_T$. However, to achieve uninformative noise $\rvx_T$,  $\{Q_i\}_{i=1}^T$ should be chosen such that every row of $\overline{Q}_t= \overline{Q}_{t|0}$ converge to the same known stationary distribution when $t$ becomes large enough (at $T$). Let the known stationary distribution be $\rvm_0 \sim \text{Cat}(\rvm_0;\vm)$. Then, the constraint can be stated as 
\begin{align}
    \lim_{t\rightarrow T} \overline{Q}_t  = \1 \vm^\top \;. \label{eq:stationary}
\end{align}
In addition, this paper focuses on \textit{nominal data} (see Appx. \S\ref{ssec:nominal}) where categories are unordered and only equality comparison is defined. Hence, no ordering prior \textit{except checking equality} should be used to define $Q_t$\footnote{Mathematically, this means $\forall k,j$, $q(\rvx_t | \rvx_{t-1} = \ve_j , \rvx_t\neq\{\ve_j, \ve_k\} ) = q(\rvx_t | \rvx_{t-1} =\ve_k,  \rvx_t\neq\{\ve_j, \ve_k\} ) .$}. To achieve the desired convergence on nominal data while keeping the flexibility of choosing any categorical stationary distribution $\rvm_0 \sim \text{Cat}(\rvm_0;\vm)$, we define $Q_t$ as 
\begin{align}
    Q_t = \alpha_t I + (1-\alpha_t)\1\vm^\top \;, \label{eq:Qt}
\end{align}
where $\alpha_t \in [0, 1]$. 
This results in the accumulated transition matrix $\overline{Q}_{t|s} $ being equal to 
\begin{align}
    \overline{Q}_{t|s} = \overline{\alpha}_{t|s} I + (1-\overline{\alpha}_{t|s} )\1 \vm^\top , \; 
   \forall t > s \;,\label{eq:Qbar_t_s}
\end{align}
where $\overline{\alpha}_{t|s} = \prod_{i=s+1}^t \alpha_i $. Note that $ \overline{\alpha}_{t} = \overline{\alpha}_{t|0} =\overline{\alpha}_{t|s} \overline{\alpha}_{s}$. We can achieve 
the uninformative noise with satisfying the constraint  
\eqref{eq:stationary} by picking $\alpha_t$ such that $\lim_{t\rightarrow T} \overline{\alpha}_t  = 0$.

\subsubsection{Analytical Form of $q(\rvx_{t-1} | \rvx_t, \rvx_0 )$}
\label{ssec:backw}

The formulation in \eqref{eq:Qt} can be used to simplify  $q(\rvx_{t-1} | \rvx_t, \rvx_0 )$. We provide a general formulation of $q(\rvx_{s} | \rvx_t, \rvx_0 )$ for any $s,t$ with $0 < s<t \le T$, which will be useful for unifying with continuous-time diffusion.
One can recover $q(\rvx_{t-1} | \rvx_t, \rvx_0 )$ by setting $s$ as $t$$-$$1$.

\begin{proposition}[\cite{RPM}]
For both discrete- and continuous-time discrete diffusion, we can write the conditional distribution as 
\begin{align}
\scalebox{0.95}{$
q(\rvx_s| \rvx_t, \rvx_0) =
\begin{cases}
    \text{Cat}\Big(\rvx_s; \ (1 - \lambda_{t|s}) \cdot \vx_t + \lambda_{t|s} \cdot \vm\Big)    &\text{when }\vx_t = \vx_0\\
     \text{Cat}\Big(\rvx_s; \ (1 - \mu_{t|s}) \cdot \vx_0 + \mu_{t|s}\overline{\alpha}_{t|s} \cdot \vx_t  + \mu_{t|s}(1-\overline{\alpha}_{t|s}) \cdot \vm \Big)    &\text{when }\vx_t \neq \vx_0 
\end{cases}
$}\label{eq:s|t,0 full}
\end{align}
where $\lambda_{t|s} := \frac{(1- \overline{\alpha}_s)(1-\overline{\alpha}_{t|s}) \langle \vm, \vx_t \rangle}{ \overline{\alpha}_t  + (1 - \overline{\alpha}_t) \langle \vm, \vx_t \rangle }, \  \mu_{t|s} := \frac{1- \overline{\alpha}_s}{1-\overline{\alpha}_t} \;. $
\vspace{-0.1in}
\end{proposition}
We remark that the above formulation is originally presented in \cite{RPM} with $s=t-1$. Appx. \S\ref{ssec:t-1|t,0} gives the detailed derivation of the probability's formulation. 

\subsubsection{Contrib. 1: Analytical Form of $p_{\theta}(\rvx_{t-1} | \rvx_{t})$}
\label{ssec:param}

The literature has explored 
three different parameterizations of $p_{\theta}(\rvx_{t-1} | \rvx_{t})$:  
(1) parameterizing $p_{\theta}(\rvx_{t-1} | \rvx_{t})$ directly;
(2) parameterizing $p_{\theta}(\rvx_{0} | \rvx_t) $ with $f_t^\theta$ such that $p_{\theta}(\rvx_{0} | \rvx_t) = \text{Cat}(\rvx_0; f_t^\theta(\rvx_t)) $ and letting $p_{\theta}(\rvx_{t-1} | \rvx_{t}) =q(\rvx_{t-1} | \rvx_{t}, f_t^\theta(\rvx_t)) $; and 
(3) parameterizing $p_{\theta}(\rvx_{0} | \rvx_{t})$ with $f_t^\theta$
and then marginalizing $q(\rvx_{t-1}, \rvx_0 | \rvx_{t} )$ such that  $p_\theta(\rvx_{t-1}| \rvx_{t} ) =\sum_{\rvx_0}q(\rvx_{t-1} | \rvx_t, \rvx_0) p_\theta(\rvx_0 | \rvx_t)$. 

Method (1) does not reuse any distribution in forward process, and hence is less effective. 
Method (2) has been widely used for continuous diffusion models as in \cite{DDPM} and \cite{DDIM}, and some discrete diffusion models like in \cite{hoogeboom2021argmax} and \cite{RPM}. It avoids marginalization and works effectively for continuous diffusion. However for discrete diffusion, as shown in \eqref{eq:s|t,0 full}, sample $\rvx_0$ determines which categorical distribution should be used, which cannot be determined without the true $\rvx_0$. Some heuristics have been proposed in \cite{RPM}, however those can have a large gap to the true $q(\rvx_{t-1}|\rvx_{t})$, leading to an inaccurate sampling process.

Method (3) has been used in \cite{D3PM} by directly marginalizing out $\rvx_0$, which introduces additional computational cost in both loss function computation and sampling process as the formulation of $p_\theta(\rvx_{t-1} | \rvx_t)$ has not been simplified to a closed-form distribution.  
In this paper,  we show that method (3) parameterization can be simplified to a clean formulation of categorical distribution.

\begin{proposition}
The parameterization of $p_\theta(\rvx_{s} | \rvx_t)$ can be simplified for any $0 < s<t \le T$ as~
{\small{
\begin{align}
p_\theta(\rvx_{s} | \rvx_t) & = \text{Cat}\Big(\rvx_s; 
        (1-\mu_{t|s})\cdot f_t^\theta(\vx_t)   + (\mu_{t|s}\overline{\alpha}_{t|s} + \gamma^\theta_{{t|s}})  \cdot \vx_t +
        (\mu_{t|s}(1-\overline{\alpha}_{t|s}) - \gamma^\theta_{{t|s}})\cdot \vm
    \Big) \label{eq:s|t}
\end{align}
}}
 where  $\gamma_{t|s}^\theta$ is affected by $f_t^\theta(\rvx_t)$, 
    $\gamma_{t|s}^\theta := (\mu_{t|s} - \lambda_{t|s} - \mu_{t|s}\overline{\alpha}_{t|s})
    \langle f_t^\theta(\vx_t) , \vx_t \rangle  \;.\label{eq:gamma}
$
\vspace{-0.1in}
\end{proposition}
The detailed proof is  in Appx. \S\ref{ssec:p_x_s_x_t}.
Notice that $f_t^\theta(\vx_t)$ is a parameterized neural network with $\1^T f_t^\theta(\vx_t) = 1$. As we show next, the above formulation simplifies the negative VLB loss computation greatly (\S\ref{ssec:vlbloss}),  further motivates an approximated loss that is much easier to optimize (\S\ref{ssec:loss_approx}),  and accelerates the sampling process through reparameterization (\S\ref{ssec:reparameter}). 

\subsubsection{Loss Function}
\label{ssec:vlbloss}

With $q(\rvx_{t-1} | \rvx_{t}, \rvx_{0})$ in \eqref{eq:s|t,0 full} and $p_\theta(\rvx_{t-1} | \rvx_t) $ in \eqref{eq:s|t}, the $\Ls_{t}(\theta)$ in \eqref{eq:vlb} can be written as
\begin{equation}
\scalebox{0.83}{$
\begin{aligned}
    \E_{q(\rvx_t| \rvx_0)}\Big[  \delta_{\rvx_t, \rvx_0}
    \KL\big(q(\rvx_{t-1}| \rvx_t = \rvx_0)\Vert p_\theta(\rvx_{t-1} | \rvx_t)  \big) + 
    (1-\delta_{\rvx_t, \rvx_0})
    \KL\big(q(\rvx_{t-1}| \rvx_t \neq\rvx_0)\Vert p_\theta(\rvx_{t-1} | \rvx_t) \big) \Big] \;, 
    \label{eq:vlb-discrete}
\end{aligned}
$}
\end{equation}
where $\delta_{\rvx_t, \rvx_0}$ denotes the Kronecker delta of $\rvx_t$ and $\rvx_0$. $q(\rvx_{t-1} | \rvx_t = \rvx_0)$ and $q(\rvx_{t-1} | \rvx_t \neq \rvx_0 )$ represent the first and second categorical distribution in \eqref{eq:s|t,0 full}, respectively. 

Apart from negative VLB, another commonly employed auxiliary loss is the cross-entropy (CE) loss between $q(\rvx_t | \rvx_0)$ and $p_\theta(\rvx_0|\rvx_t)$, which measures the  reconstruction quality. 
\begin{align}
    \Ls^{CE}_t(\theta) : = \E_{q(\rvx_t | \rvx_0)} [ - \log p_\theta(\rvx_0 | \rvx_t )] \label{eq:lt_approx}
\end{align}
\eqref{eq:lt_approx} and \eqref{eq:vlb-discrete} share the same global minima with $p_\theta(\rvx_0|\rvx_t)$ being the true posterior $q(\rvx_0 | \rvx_t)$. However they have different optimization landscape;  and thus under limited data and network capacity, which loss is be easier to minimize is unknown \citep{tewari2007consistency, demirkaya2020exploring}. 

\subsubsection{Contrib. 2: Further Loss Simplification  for Easier Optimization}
\label{ssec:loss_approx}

While \eqref{eq:vlb-discrete} is the \textit{exact} negative VLB loss, in practice we find it harder to minimize than $\Ls^{CE}_t$. In this section, we first derive a much simpler, approximated loss by observing a relation between $q(\rvx_s|\rvx_t,\rvx_0)$ and $p_\theta(\rvx_t|\rvx_0)$. 
We then show that  coefficient simplification and  $\Ls^{CE}_t$ are \textit{both} valuable for optimizing a general negative VLB where only partial time steps are observed. We combine all designs to derive the final  approximated loss, denoted as $\tilde{\Ls}_t$.

\begin{proposition}\label{prop:3}
For any $0<s<t\leq T$, with $\rvx_0$ known, $\triangle \vp_{\theta}(\rvx_s| \rvx_t, \rvx_0)$ is defined as 
\begin{align}
  p_\theta(\rvx_s | \rvx_t) - q(\rvx_s | \rvx_t, \rvx_0) =  (1-\mu_{t|s})[
    f^\theta_t(\vx_t)-\vx_0 + \phi_{t|s}  \langle f_t^\theta(\vx_t) -\vx_0 , \vx_t \rangle
    (\vx_t - \vm) 
    ] \;, \label{eq:prop3}
\end{align}
where 
$    \phi_{t|s} := \frac{(1-\bar{\alpha}_s)\bar{\alpha}_{t|s}}{\bar{\alpha}_t + (1-\bar{\alpha}_t)  \langle \rvx_t, \vm  \rangle } \;.$    
\end{proposition}
Proposition \ref{prop:3} shows that the distribution difference between $p_\theta(\rvx_s|\rvx_t)$ and $q(\rvx_s|\rvx_t,\rvx_0)$ has a closed-form formulation. (See Appx. \S\ref{ssec:proof_delta_pq} for the proof.) With it, we can apply the Taylor expansion (up to second order) to approximate the KL divergence directly (See Appx. \S\ref{ssec:kl_approx}.) 
\begin{equation}
\scalebox{1}{$
\begin{aligned}
    \KL\big(q(\rvx_s|\rvx_t, \rvx_0) \Vert p_\theta(\rvx_s |\rvx_t) \big) \approx \sum_{\rvx_s }\frac{ |\triangle \vp_{\theta} (\rvx_s| \rvx_t, \rvx_0)|^2 }{q(\rvx_s| \rvx_t, \rvx_0)}\label{eq:kl_approx}
\end{aligned}
$}
\end{equation}
The above formulation along with prop.\ref{prop:3} are valid for \textit{any} $0<s<t\leq T$. We next show that minimizing divergence between $q$ and $p_\theta$ at any $s$ and $t$ is also valid, as it is inside a general negative VLB with partial time steps. The initial version of the VLB is derived under the assumption that observations are made at every time step. Its backward denoising process is designed to advance by a single time step during each generation step  for best generation quality. Let us consider a more general case where only partial time steps are observed in the forward process, then, minimizing its negative VLB can help improve generation quality with fewer steps. 
Assuming only $\rvx_s$ and $\rvx_t$ are observed, where $0 < s < t \leq T$, a derivation analogous to that of \eqref{eq:vlb} can show that
\begin{equation}
\scalebox{0.88}{$
\begin{aligned}
    &\log p_\theta(\rvx_0) \ge  - \KL[q(\rvx_t| \rvx_0) \Vert p_\theta(\rvx_t)] + \E_{q(\rvx_s|\rvx_0)}[\log p_\theta(\rvx_0| \rvx_s)] - \KL[q(\rvx_s|\rvx_t,\rvx_0) \Vert p_\theta(\rvx_s|\rvx_t)] ,\label{eq:vlb_partial}
\end{aligned}
$}
\end{equation}
where the first divergence term between the prior and posterior quantifies the quality of the backward denoising process from $T$ to $t$. The second term represents the 
CE loss, $\Ls_s^{\text{CE}}$, which influences the generation quality from time $s$ to time $0$. The final term is given by \eqref{eq:kl_approx}, and contributes to the generation process from $t$ to $s$. (See  Appx. \S\ref{ssec:vlb_partial} for proof.) 

This generalized formulation of the VLB highlights the significance of the CE loss and the alignment between $q(\rvx_s|\rvx_s,\rvx_0)$ and $p_\theta(\rvx_s|\rvx_t)$ at any observed times $s$ and $t$. While CE loss  does not have a coefficient that depends on noise schedules and time, changing $s$ and $t$ or noise schedule (which determines $\bar{\alpha}_t $) during training will greatly impact the scale of the term in \eqref{eq:kl_approx}.
By rendering the loss scaling term independent of time and the noise schedule, the minimization of this adjusted loss concurrently leads to the minimization of the original loss in \eqref{eq:kl_approx} for any given $s$ and $t$. Hence, by removing the sensitive scale $\frac{(1-\mu_{t|s})^2}{q(\rvx_s|\rvx_t, \rvx_0)}$ in \eqref{eq:kl_approx}, we
reformulate the loss as
\begin{align}
   \Ls_t^2 :=  \Vert f^\theta_t(\vx_t)-\vx_0 + \phi_{t|s}  \langle f_t^\theta(\vx_t) -\vx_0 , \vx_t \rangle
    (\vx_t - \vm) \Vert^2_2,
\end{align}
where we can further clip $\phi_{t|s}$ to $\min(1, \phi_{t|s})$ for minimal scaling influence. It is important to note that while we have modified the coefficient to be invariant to the noise schedule and time $s$ to effectively minimize \eqref{eq:kl_approx} at any $t$ and $s$, a similar approach to coefficient revision has been previously explored in \cite{DDPM, karras2022elucidating}, primarily  to facilitate an easier optimization by achieving a balance of terms in the loss function.

Overall, we have the final approximated loss as 
$    \tilde{\Ls}_t(\theta) = \Ls_t^2(\theta) + \Ls_t^{CE}(\theta) \;.$We find that in practice this loss is much easier to optimize than the original exact negative VLB for more complicated distributions, leading to faster convergence and improved generation quality.

\section{Continuous-time Discrete Diffusion}
\label{sec:cd}
Despite being simple,  discrete-time diffusion limits the generation process as we can only ``jump back'' through  fixed time points. Recent works generalize continuous-state diffusion models to continuous-time \citep{song2021scorebased}.
This generalization enables great flexibility in backward generation  as one can ``jump back'' through any time in $[0,T]$  with improved sample quality.

Nevertheless, generalizing discrete-state diffusion model from discrete-time to continuous-time is nontrivial, as the score-matching based technique \citep{song2021scorebased} in continuous-state models requires the score function $\nabla_{\rvx}\log p_t(\rvx)$ to be available.  This function, however, is evidently non-existent for discrete distributions.
Recently, \citet{campbell2022continuous} presented the first continuous-time diffusion model for discrete data. 
It formulates the forward process through a Continuous Time Markov Chain (CTMC) 
and aims at learning a reverse CTMC that matches the marginal distribution with the forward CTMC at any time $t$. 
While being theoretically solid, the formulation in \cite{campbell2022continuous} has two problems: 
\textbf{(1)} the negative VLB loss of matching the forward and backward CTMCs is analytically complicated and hard to implement;  
and \textbf{(2)} the exact sampling through the learned backward CTMC is unrealistic.
\citet{campbell2022continuous} propose approximate solutions, which however, trade off computational tractability with unknown errors and sample quality.  
SDDM \cite{sun2023scorebased} takes a different approach with ratio matching \cite{ratiomatch1, ratiomatch2}, which is the generalization of score matching to discrete data. However, it needs a specific model architecture. \cite{lou2024discrete} advances ratio matching, yielding a loss that requires only a single model pass. Interestingly, this loss is similar to ours in formulation, despite originating from a different angle.

In this paper, we build on the CTMC formulation in \citet{campbell2022continuous}, 
and show that the loss can be analytically simplified for nominal data. This simplification also inspires an improved MCMC corrector with closed-form formulation.
For problem \textbf{(2)}, we argue that the difficulty arises from directly using the learned backward CTMC's transition rate. Instead, realizing that the reverse transition rate is computed based on the learned $p_\theta(\rvx_0 | \rvx_t)$, we propose to compute $p_\theta(\rvx_s | \rvx_t)$ through $ p_\theta(\rvx_0 | \rvx_t)$ without using the transition rate matrix. This avoids the approximation error of sampling directly from the backward CTMC and greatly simplifies the generation process.

\textit{Remarkably}, with the new formulation of generation, we show that the continuous-time and discrete-time diffusion models can be unified together, with exactly the same forward diffusion and now also backward generation process. {Moreover, we demonstrate that this unification offers mutual benefits: the continuous-time diffusion can leverage the swift and precise sampling formulation derived from the discrete-time case (as detailed in \S\ref{ssec:reparameter}), while the discrete-time diffusion can utilize the MCMC corrector from the continuous-time scenario (see  Appx. \S\ref{ssec:cont_other} and \S\ref{ssec:mcmc}).}

\subsection{Background: Continuous-Time Markov Chain}
\label{ssec:ctmcbackground}

CTMC generalizes Markov chain from discrete- to continuous-time via the Markov property: $\rvx_{t_1} \ci \rvx_{t_3} \ |\  \rvx_{t_2}$, $\forall t_1 < t_2 < t_3$. \citet{book_CTMC} provides an introduction to time-homogeneous CTMC. 
It can be derived from discrete-time Markov chain by  increasing the number of time stamps $N$ to infinite while keeping the total time $T$ fixed. Specifically, 
we can define $\triangle t = \frac{T}{N}$, $t_i = i \triangle t$, and a discrete-time Markov chain characterized by transition probability $q(\rvx_{t_{i}} | \rvx_{t_{i-1}})$ and transition matrix $Q_{t_i}$ with $[Q_{t_i}]_{jk} = q(\rvx_{t_{i}} =\ve_k | \rvx_{t_{i-1}}=\ve_j )$. 
By setting $N$ to infinite, the transition probability $q(\rvx_{t_{i}} | \rvx_{t_{i-1}})$ converges to 0, hence is not suitable for describing CTMC. Instead, CTMC is fully characterized by its \textit{transition rate} $r_t(\vy| \vx)$, s.t. 
\begin{align}
    r_t(\vy|\vx) =\lim_{N\rightarrow \infty} \frac{q(\rvx_{t_{i}} = \vy| \rvx_{t_{i-1}}=\vx)- \delta_{\vx, \vy}}{t_{i} - t_{i-1}}
    = \lim_{\triangle t \rightarrow 0} \frac{q_{t|t-\triangle t}(\vy | \vx) - \delta_{\vx, \vy}}{\triangle t} \;. \label{eq:r_t limit}
\end{align}
As the name suggests, $r_t(\vy | 
 \vx)$ measures the change rate of the transition probability of moving from state $\vx$ to state $\vy$ at time $t$ in the direction of the process. The corresponding \textit{transition rate matrix} $R_t$ with  $[R_t]_{ij} = r_t(\ve_j| \ve_i)$ fully determines the underlying stochastic process. A CTMC's transition probabilities satisfy the Kolmogorov equations \citep{kolmogorov_equation} (see Appx. \S\ref{ssec:ctmc-intro}), which have unique solution. As \citet{inhomo_CTMC} stated, when $R_{t_1}$ and $R_{t_2}$ commute (i.e. 
 $R_{t_1}R_{t_2} = R_{t_2}R_{t_1} $) for any $t_1,t_2$, the transition probability matrix can be written as 
$
    \overline{Q}_{t|s} = \exp\Big(\int_s^t R_a da \Big) \;, 
$
where $\exp(M):=\sum_{k=0}^{\infty}  \frac{M^k}{k!}$. The commutative property of $R_t$ can be achieved by choosing $R_t = \beta(t) R_b$ where $R_b \in \R^{K\times K}$ is a time-independent base rate matrix.

\subsection{Contrib. 3: Simplified Forward \& Backward CTMCs that Connect to Discrete-Time Case}
\label{ssec:fb_ctmc}
\textbf{Forward CTMC.} Two properties are needed for modeling the forward process: P1) the process can converge to an easy-to-sample stationary distribution at final time $T$; and P2) the conditional marginal distribution $q(\rvx_t | \rvx_0)$ can be obtained analytically for efficient training. As given above, P2) can be achieved by choosing commutative transition rate matrices  with $R_t = \beta(t)R_b$. 

We next show that property P1), i.e. $\lim_{t\rightarrow T}\overline{Q}_{t|0} = \overline{Q}_{T|0} = \1 \vm^\top$ for some stationary distribution $\vm$,  can be achieved by choosing $R_b = \1\vm^\top - I$, which is a valid transition rate matrix with the property  $(-R_b)^2 = (-R_b)$ (see derivation in Appx. \S\ref{ssec:derivation-qbar}). Then we have 
\begin{align}
    \overline{Q}_{t|s} = \exp(\overline{\beta}_{t|s}R_b) = e^{-\overline{\beta}_{t|s}} I + (1-e^{-\overline{\beta}_{t|s}}) \1 \vm^\top. 
    \label{eq:Qbar_ts_continuous}
\end{align}
{\LARGE{$\star$}} \textbf{Unified forward process.} \eqref{eq:Qbar_ts_continuous} has the same formulation as the transition matrix of the discrete-time case in \eqref{eq:Qbar_t_s}, if we set $\overline{\alpha}_{t|s} = \exp(-\overline{\beta}_{t|s}) = \exp(-\int_{s}^t \beta(a)da)$. Thus, \textit{this formulation unifies the forward processes of adding noise for both discrete- and continuous-time discrete diffusion}. With $\lim_{t\rightarrow T} \overline{\alpha}_{t|0} = 0$, or equivalently $\lim_{t\rightarrow T } \int_0^T\beta(a) da = \infty$, we achieve the goal $\overline{Q}_{T|0} = \1\vm^\top$. We use $\overline{\alpha}_{t|s}$ directly in the following sections, i.e. \eqref{eq:Qbar_t_s}. To summarize, for the forward CTMC
\begin{align}
R_t &= \beta(t)(\1\vm^\top - I)\;, \text{and} \quad
\overline{Q}_{t|s} = 
\overline{\alpha}_{t|s} I + (1 - \overline{\alpha}_{t|s}) \1 \vm^\top \;. \label{eq:forward_CTMC}
\end{align}
We can further get vector-form forward rate 
\begin{align}
 r_t(\vx|\cdot) &= R_t \vx = \beta(t) \big(\langle \vx, \vm \rangle \1 - \vx \big), \quad
r_t(\cdot|\vx) = R_t^\top \vx =\beta(t)
\big( \vm - \vx \big) \;. 
\end{align}

\textbf{Backward CTMC.} To generate samples from the target distribution, we have to reverse the process of the forward CTMC. 
Let $\widehat{r}_t$ be the transition rate of the backward CTMC with corresponding matrix $\widehat{R}_t$. When the forward and backward CTMCs are matched exactly, theoretically 
the forward and backward CTMCs have the following relationship (see \citet{campbell2022continuous}'s Proposition 1):
\begin{align}
    \widehat{r}_t(\vx| \vy) = r_t(\vy | \vx) \frac{q_t(\vx) }{q_t(\vy)}, \quad 
 \forall \vx \neq \vy  \;. \label{eq:forward_backward_rate}
\end{align}
However, the marginal distributions $q_t(\vx)$ and $q_t(\vy)$ are intractable analytically, hence we cannot derive the backward CTMC directly from \eqref{eq:forward_backward_rate}. Instead, \citet{campbell2022continuous} parameterize the transition rate $\widehat{r}^\theta_t$,  by observing that $\frac{q_t(\vx) }{q_t(\vy)} = \sum_{\vx_0} \frac{q_{t|0}(\vx | \vx_0)}{q_{t|0}(\vy | \vx_0)} q_{0|t} (\vx_0 | \vy)$\footnote{ As
$q_t(\vx) = \sum_{\vx_0}q_{t|0}(\vx|\vx_0)q_{\text{data}}(\vx_0)$ 
and $q_t(\vx) = \frac{q_{t|0}(\vx| \tilde{\vx}_0 ) q_{\text{data}}(\tilde{\vx}_0) }{q_{0|t}( \tilde{\vx}_0 | \vx) } $.}, as follows
\begin{align}
    \widehat{r}^\theta_t(\vx| \vy) = r_t(\vy | \vx) \sum_{\vx_0} \frac{q_{t|0}(\vx | \vx_0)}{q_{t|0}(\vy | \vx_0)} p^\theta_{0|t} (\vx_0 | \vy). \label{eq:rate_parameterization_original}
\end{align}
Then, $\widehat{r}^\theta_t$ is learned with $\theta$ to minimize the continuous-time negative VLB introduced next. 

\textbf{Negative VLB.} Similar to the discrete-time case, the backward CTMC can be learned by maximizing the VLB for data log-likelihood. Computing VLB for CTMC is nontrivial, and fortunately \citet{campbell2022continuous} has derived (see their Proposition 2) that the negative VLB can be formulated as 
\begin{align}
    T \ \E_{
    \substack{t\sim\text{Uni}(0,T) \\ \vx \sim q(\rvx_t | \rvx_0)  }} 
    \Big[ \sum_{\vz \neq \vx} \widehat{r}_t^\theta(\vz | \vx) - \sum_{\vz \neq \vx}r_t(\vz|\vx)\log \widehat{r}_t^\theta(\vx | \vz) \Big] \;.
    \label{eq:CTMC_VLB_0}
\end{align}
However, \citet{campbell2022continuous}'s original design did not simplify the negative VLB with the parameterization of $\widehat{r}^\theta_t$ in \eqref{eq:rate_parameterization_original}, making the implementation nontrivial and inefficient. In this section, we show that their formulation can be greatly simplified to a closed-form evaluation.

To simplify \eqref{eq:CTMC_VLB_0}, we introduce $g^\theta_t(\vx | \vy)$ such that $\widehat{r}^\theta_t(\vx| \vy) = r_t(\vy | \vx)  g^\theta_t(\vx | \vy)$, with 
\begin{align}
    g^\theta_t(\vx | \vy) := \sum_{\vx_0} \frac{q_{t|0}(\vx | \vx_0)}{q_{t|0}(\vy | \vx_0)} p^\theta_{0|t} (\vx_0 | \vy) \approx  \frac{q_t(\vx) }{q_t(\vy)} \;,
\end{align}
which is the estimator of the marginal probability ratio. 
\begin{proposition}
The vector form parameterization of $g^\theta_t(\vx | \vy)$ can be simplified analytically as:
\begin{align}
    g_t^\theta(\cdot| \vy) = \Big[ 
        \big(1 - \frac{\overline{\alpha}_{t|0} \langle f_t^\theta(\vy), \vy \rangle }{\overline{\alpha}_{t|0} + (1-\overline{\alpha}_{t|0})\langle \vy, \vm \rangle }\big) \vm +  
        \frac{\overline{\alpha}_{t|0}}{1-\overline{\alpha}_{t|0}} f_t^\theta(\vy)
    \Big] \odot \frac{\1 - \vy}{\langle \vy, \vm \rangle}   + \vy 
\label{eq:vector_g}
\end{align}
\vspace{-0.2in}
\end{proposition}
The proof is in Appx. \S\ref{ssec:derivation-gt}, with its  extension to multi-element  $g_t^{\theta,d}(\cdot|\vy^{1:D})$  in \eqref{eq:vector_g_multidim}.

\subsection{Contrib. 4: Simplification of Continuous-time Negative VLB}
\label{ssec:cont_loss_approx}

As the derivation is much harder in multi-element case, we work on it directly and single-element can be induced as a special case. Given  $\vx^{1:D}$, let $\vx^{\backslash d}$ represent $\vx^{1:D\backslash d}$, i.e. the object without $d$-th element.
As we assume the forward processes are independent for different elements, $r_t^d$ represents the transition rate of the forward CTMC process at the $d$-th element. 
\begin{proposition}\label{prop:ctmc_vlb_final}
The negative VLB in
\eqref{eq:CTMC_VLB_0} in multi-element case can be simplified as 
\begin{equation}
\scalebox{0.84}{$
\begin{aligned}
T\ \E_{\substack{t\sim\text{Uni}(0,T) \\ \vx^{1:D} \sim q_{t|0}(\cdot|\vx_0^{1:D}) 
\\\vz^{1:D} \sim S_t(\cdot | \vx^{1:D})}} 
\Big[ 
\sum_{d=1}^D r^d_t(\vx^d|\cdot)^\top  g_t^{\theta,d}(\cdot| \vx^{1:D}) -
\frac{1}{\gM_{S_t}(\vz^{1:D}|\vx^{1:D}_0)}
\sum_{d=1}^D \frac{\1^\top
[q_{t|0}(\cdot |\vx_0^d) \odot 
r_t^d(\vz^d | \cdot) \odot
\log g_t^{\theta,d}(\cdot | \vz^{1:D})]
}{q_{t|0}(\vz^d|\vx_0^d)} 
\Big] \label{eq:ctmc_vlb_final}
\end{aligned}
$}
\end{equation}
where $S_t(\vz^{1:D}|\vx^{1:D})$ is any unnormalized distribution (see \eqref{eq:St} in Appx.) of sampling the auxiliary  variable $\vz^{1:D}$ from $\vx^{1:D}$. The auxiliary variable is introduced to avoid multiple passes of the model for computing the second term of \eqref{eq:CTMC_VLB_0}. $M_{S_t}$ is a normalization scalar (see \eqref{eq:M_St} in Appx.) that only depends on $S_t, \vz^{1:D},\text{and } \vx_0^{1:D}$.  
\end{proposition} 

The detailed proof is in Appx. \S\ref{ssec:cont-vlb-multi-dim}. This formulation simplifies and generalizes the result in \citet{campbell2022continuous}, such that any $S_t$ can be used for introducing the auxiliary variable $\vz_{1:D}$. Importantly, \eqref{eq:ctmc_vlb_final} shows that changing $S_t$ only affects a scalar weight. Computing the loss requires two passes of the model for $\vz^{1:D}$ and $\vx^{1:D}$ separately. Now we show that it can be further simplified to only a single pass of the model.

\begin{proposition}\label{prop:ctmc_vlb_perfect}
The loss in \eqref{eq:ctmc_vlb_final} can be further simplified as 
\begin{equation}
\scalebox{1}{$
\begin{aligned}
    T\ \E_{\substack{t\sim\text{Uni}(0,T) \\ \vx^{1:D} 
    \sim q_{t|0}(\cdot|\vx_0^{1:D}) }} 
\Big[ 
\sum_{d=1}^D r^d_t(\vx^d|\cdot)^\top 
\Big( g_t^{\theta,d}(\cdot| \vx^{1:D}) 
- 
\frac{
q_{t|0}(\cdot |\vx_0^d) \odot 
\log g_t^{\theta,d}(\cdot | \vx^{1:D})
}{q_{t|0}(\vx^d|\vx_0^d)} \Big) 
\Big] 
\end{aligned}
$}
\end{equation}
\end{proposition}

The detailed proof can be found in Appendix \S\ref{apdx:proof_sedd}. Note that this loss requires only a single forward pass of the model, making it extremely efficient. Additionally, this loss aligns with the DWDSE loss (while presented in single-element case) in \cite{lou2024discrete}, which is derived from the different concrete score matching direction. This indicates that the VLB is essentially the same as the ratio matching loss, as also demonstrated in the continuous-state setting, such that the score matching \cite{song2021scorebased} is shown to be equivalent to VLB computation in \cite{DDPM}.


\subsection{Contrib. 5: Fast Backward Sampling \& Unification of Discrete Diffusion} \label{ssec:CTMC_reverse_sampling}

\citet{campbell2022continuous} proposes to use the learned transition rate $\widehat{r}^\theta_t$ of the backward CTMC  to sample reversely for the target distribution $p_{\text{data}}(\rvx_0)$. 
However, different from the forward CTMC where all elements are sampled independently, the backward CTMC processes for all elements are coupled together and do not have the closed-form transition probability derived from the learned transition rate. Direct and exact sampling from a CTMC uses the algorithm by \citet{gillespie1977exact},  which is intractable for multi-element objects.  \citet{campbell2022continuous} proposes to use tau-leaping \citep{gillespie2001approximate} that approximately samples all transitions from time $t$ to $t-\tau$,  assuming $R_t^{1:D}$ fixed during the time period. While being tractable, it introduces unknown errors that needs correction process.


 We propose to avoid the costly sampling from using the transition rate of backward CTMC. We first observe that the estimated transition rate $\widehat{r}^\theta_t$ is essentially derived from the estimated $p^\theta_{0 |t} (\rvx_0| \rvx_t)$. Hence using $\widehat{r}^\theta_t$ is equivalent to using $p^\theta_{0 |t} (\rvx_0| \rvx_t)$ in an indirect way.
 We realize that the transition probability $q_{s|t}(\rvx_s| \rvx_t)$ can be computed easily from $p^\theta_{0 |t} (\rvx_0| \rvx_t)$ directly, as shown in \eqref{eq:s|t}. With the help of \eqref{eq:s|t}, we  can sample $\rvx_{0}$ or equivalently $\rvx_{0|T}$ through sampling $\rvx_{t_{n-1} | t_{n}},...,\rvx_{t_1|t_2}, \rvx_{t_0|t_1}$ sequentially,  with any $\{t_i\}_{i=0}^n$ that satisfies $0$=$t_0<t_1<...<t_n$=$T$.  Although \eqref{eq:s|t} is derived for discrete-time case, it applies directly in continuous-time case, thanks to the unification of the forward process for discrete- and continuous-time diffusion (\S\ref{ssec:fb_ctmc}). Hence, discrete-time and continuous-time diffusion also have the same \textbf{unified backward generation process}.

 {\LARGE{$\star$}} \textbf{Unified Discrete Diffusion.}
All in all, through a series of mathematical simplifications, we have shown that both discrete\&continuous-time discrete diffusion share  \textbf{(1)}  the same forward diffusion process; \textbf{(2)}  the same parameterization for learning $p_{0|t}^\theta$; and  \textbf{(3)} the same backward denoising 
process. In light of our reformulations, we propose \method, a novel \underline{U}nified and \underline{S}implified \underline{D}iscrete \underline{D}enoising \underline{D}iffusion model.  
Notably, \method utilizes the same source code for both discrete- and continuous-time,  up to a single alteration in the loss function during training, and a shared sample generation process (resp. Algo. \ref{alg:overall} and \ref{alg:generation} in Appx. \S\ref{ssec:alg-unify}).


\section{Experiments}
\label{sec:exp}
Discrete diffusion has limited
, no established guidelines on 
model training.
Many prior work optimize a combined VLB and cross-entropy (CE) loss with a fixed weight, without deeper investigation. 
In this work we not only improve VLB loss mathematically, but also empirically explore an extensive testbed of training regimes for discrete diffusion---including various loss combinations, both discrete- and continuous-time training, as well as  varying model sizes---{toward a deeper understanding of which yield tractable optimization and higher generation quality.} 
 \method is a hybrid of our  simplified exact VLB and CE, and \methodsim refers to its approximation with further simplifications (\S\ref{ssec:loss_approx} and \S\ref{ssec:cont_loss_approx}).  \methodce and \methodvlb are variants, resp. with CE or VLB loss \textit{only}.

 \textbf{Baselines.}
We compare \method and variants to  three latest SOTA discrete diffusion models in the literature: \dpm in \cite{D3PM}, \rdm in \cite{zhang2023text} (discrete-time), \tauldr-0 in \cite{campbell2022continuous} , \sddm in \cite{sun2023scorebased} and \sedd in \cite{lou2024discrete} (continuous time).

 \textbf{Training Details.}
Thanks to unification, we can evaluate both discrete- and continuous-time models with both cosine and constant noise schedulers at ease. In \method, we combine VLB and CE with weight  $0.001$ for the latter, following prior literature. For \methodsim, we combine VLB with CE weight $1.0$. As architecture, we parameterize $f_t^{\theta}(\mathbf{x}_t)$ with a sequence transformer model. By varying its depths and widths, we study the impact of model size. The details are described in Appx.  \S\ref{ssec.training_details}.

\subsection{Music Generation}

\textbf{Lakh Pianoroll Datasets.~}
 We evaluate monophonic music generation on \ptwo, the cleaned Lakh pianoroll dataset [\cite{raffel2016learning,dong2017musegan}], containing $6,000$ training and $973$ evaluation (or test) sequences of $256$ notes each. Here we perform conditional generation; given the first $32$ notes, the models are required to generate the rest $224$ notes.  Interestingly, we find that some (124)  evaluation sequences share the \textit{same} first $32$ notes as one or more training sequences. Such repetition gives us a unique opportunity to investigate what-we-call ``parroting''---the phenomenon that  models are simply re-playing the exact training data during generation.
To that end we create \pone, a subset consisting only of these 124 evaluation sequences paired with their first-32-note-matching training sequences. Details are in Appx. \S\ref{sssec.piano}.

 \textbf{Metrics.~}
We use a number of new metrics for \ptwo: 
2- and 3-gram 
 (1) {Hellinger Distance} (\textdownarrow) and (2) {Proportion of Outliers} (\textdownarrow), in addition to previously used 1-gram counterparts based only on \textit{single} note appearances  in \cite{campbell2022continuous}; 
 (3) {Diverse Edit Distance} (\textuparrow), which accounts for the creativity/novelty across multiple generated samples. On \pone, we also report  (4) {Train-to-Test Ratios} (\textuparrow) for $\{$1,2,3$\}$-gram Hellinger as well as Proportion of Outliers, 
which compare the weighted distance of a generated sample to its matching training vs. test sequence that share the same first $32$ notes. Such ratios quantify the extent of ``parroting''; a smaller ratio 
indicates a higher (lower) degree of memorization (generalization).  
Calculation details of all metrics are given in  Appx. \S\ref{sssec.music_metrics}.

\textbf{Results.~} 
Table \ref{tab:piano_ngram} shows that \method and its variants perform better than the baseline methods across metrics. An exception is \dpm's Diverse Edit Distance, which  trades-off high novelty with low generation quality. \methodce, while performing well w.r.t. 1-gram metrics, does not compete with \methodvlb and \method w.r.t. $\{$2,3$\}$-gram metrics, which indicates that \textit{CE only} loss captures the least sequential information. Models with combined losses, \method and \methodsim, achieve higher Train-to-Test ratio than pure CE or VLB based ones, suggesting that the combination of two losses can alleviate overfitting, i.e. ``parroting'' the training data. 
We find continuous-time \method to perform better than discrete-time counterparts 
models because its ability of denoising through any timestep.
\begin{table}[ht]
\vspace{-0.2in}
    \centering
    \begin{minipage}{0.68\textwidth}
        \caption{Conditional music gen. quality (3 samples avg.) on \ptwo w.r.t. n-gram Hellinger, n-gram Prop. of Out., and Diverse Edit Dist. Also, Train-to-Test Ratio for on \pone quantifies ``parroting''. \first{First} \& \second{Second} shown in color (Detailed version is in Table \ref{tab:piano_ngram_full}).}
\scalebox{0.70}{
\begin{tabular}{ll|ccc|ccc|c|c}
\toprule
 &  & \multicolumn{3}{c|}{n-gram Hellinger(\textdownarrow)}                   & \multicolumn{3}{c|}{n-gram Prop. of Out.(\textdownarrow)} & Div. Edit & 3g-Prop.  \\ 
& Method & 
1gram & 2gram  & 3gram & 1gram  & 2gram & 3gram &  Dist. (\textuparrow) & 
Ratio(\textuparrow) \\
\midrule
\multirow{4}{*}{\rotatebox{90}{Discrete-time}} &  \dpm  & 0.398  & 0.530 & 0.591 & 0.120 & 0.253 & 0.379 & \first{0.295} & 2.221 \\
& \rdm  & 0.412  & 0.596 & 0.620 & 0.140  & 0.246 & 0.386 & 0.289 & 2.521 \\
\cmidrule{2-10}
& \scalebox{0.78}{\methodce} & 0.375 & 0.483 & 0.574 & 0.107 & 0.209 & 0.303 & 0.047 & 2.888 \\
& \scalebox{0.78}{\methodvlb} & 0.379 &  0.464 & \first{0.542}  & 0.117  & 0.184  & 0.273  & \second{0.082} & 2.863 \\
& \method & 0.377 & 0.469 & 0.552 &    \second{0.107} & 0.186 & 0.286 & 0.064 & \first{3.083} \\
& \methodsim & 0.375  & 0.470  & 0.555 &  0.110   & 0.191  & 0.283  & 0.066 & 2.959 \\
\midrule
\multirow{4}{*}{\rotatebox{90}{Continuous-time}} & \sddm & 0.375& 0.485& 0.577 & 0.110 & 0.205 & 0.340 & 0.060 & 2.901\\
& \tauldr-0 & 0.379 & 0.481 & 0.571 & 0.114 & 0.207 & 0.320  & 0.050 & 2.965 \\
& \sedd & 0.381 & 0.471 & 0.547 & \first{0.105} & \first{0.182} & 0.271 & 0.061 &  3.042 \\
\cmidrule{2-10}
& \scalebox{0.78}{\methodce} & \first{0.373} & 0.483 & 0.577 & 0.116 & 0.222 & 0.346 & 0.043 & 2.637\\
& \scalebox{0.78}{\methodvlb} & 0.376 & \second{0.462} & 0.549 & 0.120 & 0.186  & 0.295 & 0.061 & 2.904 \\
& \method & \second{0.374} & \first{0.462} & \second{0.543} & 0.112 & \second{0.184} & \first{0.270} & 0.066 & \second{3.083}  \\
& \methodsim & 0.374 & 0.462 & 0.549 & 0.114 & 0.184 & \second{0.271} & 0.052 & 2.923 \\
\bottomrule
\end{tabular}
\label{tab:piano_ngram}
}
    \end{minipage}\hfill
    \begin{minipage}{0.29\textwidth}
        \caption{
Image generation quality w.r.t. IS and FID over 50,000 unconditional generated CIFAR10  samples. \first{First} \& \second{Second} shown in color.}
\scalebox{0.75}{
\begin{tabular}{llcc}
\toprule
& Method  & IS (\textuparrow) & FID (\textdownarrow)  \\
\midrule
\multirow{4}{*}{\scalebox{0.90}{\rotatebox{90}{Discrete-time}}} &
  \dpm  & 8.13 & 18.08   \\
  &
  \rdm & 7.16 & 24.23   \\
\cline{2-4}
& \scalebox{0.8}{\methodce} & 9.02 & 12.64 \\
& \methodsim & \first{9.27} & \second{12.07}  \\
& \method & 8.85 & 13.25 \\
\midrule
\multirow{4}{*}{\scalebox{0.90}{\rotatebox{90}{Cont.-time}}} & \sddm & 8.72 & 14.17 \\
& \tauldr  & 8.37 & 17.61 \\
& \sedd & 8.51 & 17.01 \\
\cline{2-4}
& \scalebox{0.8}{\methodce} & \second{9.23} & \first{11.97} \\
& \methodsim & 8.83 & 14.03 \\
& \method & 8.97 & 13.09 \\
\midrule
\multicolumn{2}{l}{VQGAN Recons. } &10.42 &7.68\\
\bottomrule
\end{tabular}
}
\label{tab:vqcifar10}
    \end{minipage}
\vspace{-0.1in}
\end{table}

 \subsection{Image Generation}

\textbf{VQCIFAR10 Dataset.~}
CIFAR10  
  contains $50,000$ training images in continuous values, which we 
 convert to vectors from $64 \times 512$-dimensional quantization hash-code space with a pre-trained VQGAN \cite{esser2020taming}. This conversion allows us to (1) evaluate our methods on \textit{nominal} data by breaking the neighboring orders in the image space, and (2) select a closed-form stationary distribution $\vm$. 
 Details on \cifar dataset are given in Appx.  \S\ref{sssec.vqgan}.

\vspace{-0.05in}
\textbf{Metrics.~} 
After feeding the generated samples through the VQGAN decoder to obtain representations from the discretized space, we measure the {Inception Score (IS)}  (\textuparrow)  and {Frechet Inception Distance (FID)} (\textdownarrow) against the original CIFAR10 dataset. Note that our training set, which are discretized images from  VQGAN, achieves $\text{IS}$$=$$10.42$ and $\text{FID}$$=$$7.68$, which are optimistic limits for generation.

\textbf{Results.~}
Table \ref{tab:vqcifar10} shows that our proposed approaches outperform existing baselines. Discrete-time \dpm and \rdm fall short for the harder image generation task, whereas our simplified discrete-time loss boosts quality significantly. In continuous-time, 
\tauldr with a similar loss to \method falls short due to its complicated generation process, whereas SDDM is most competitive among the baselines, although requires substantial compute resources while being limited to a specialized model architecture.  \methodce achieves competitive IS and FID scores, which validates our finding (recall \eqref{eq:vlb_partial}) that CE loss is essential for diffusion loss minimization. 

\vspace{-0.05in}
{\bf Additional results.} 
We provide memory and runtimes of models in Table \ref{tab:size_time_memory} in Appx. \S \ref{ssec:memory_and_runningtime},
generated images in Fig.\ref{fig:vq_recons} in Appx. \S\ref{ssec:imagesamples}, and ablations of MCMC sampling in Appx. \S\ref{ssec:mcmc}.

\section{Conclusion}
\label{sec:conclusion}
We introduce two fundamental contributions 
for both discrete-time  and continuous-time diffusion for categorical data. First, 
we presented extensive \textit{mathematical simplifications for the loss functions}, including exact closed-form derivations as well as novel easy-to-optimize approximations. Second, we established a \textit{mathematical unification } of discrete-time  and continuous-time discrete diffusion, enabling mutual benefits.
Our proposed approach {\small{\method}} for discrete diffusion achieved significant improvements on established datasets  across a suite of generation quality metrics.

\bibliographystyle{plainnat}
\bibliography{reference}

\begin{thebibliography}{36}
\providecommand{\natexlab}[1]{#1}
\providecommand{\url}[1]{\texttt{#1}}
\expandafter\ifx\csname urlstyle\endcsname\relax
  \providecommand{\doi}[1]{doi: #1}\else
  \providecommand{\doi}{doi: \begingroup \urlstyle{rm}\Url}\fi

\bibitem[Anderson(2012)]{book_CTMC}
William~J Anderson.
\newblock \emph{Continuous-time Markov chains: An applications-oriented approach}.
\newblock Springer Science \& Business Media, 2012.

\bibitem[Austin et~al.(2021)Austin, Johnson, Ho, Tarlow, and van~den Berg]{D3PM}
Jacob Austin, Daniel~D Johnson, Jonathan Ho, Daniel Tarlow, and Rianne van~den Berg.
\newblock Structured denoising diffusion models in discrete state-spaces.
\newblock \emph{Advances in Neural Information Processing Systems}, 34:\penalty0 17981--17993, 2021.

\bibitem[Brown et~al.(2020)Brown, Mann, Ryder, Subbiah, Kaplan, Dhariwal, Neelakantan, Shyam, Sastry, Askell, et~al.]{brown2020language}
Tom Brown, Benjamin Mann, Nick Ryder, Melanie Subbiah, Jared~D Kaplan, Prafulla Dhariwal, Arvind Neelakantan, Pranav Shyam, Girish Sastry, Amanda Askell, et~al.
\newblock Language models are few-shot learners.
\newblock \emph{Advances in neural information processing systems}, 33:\penalty0 1877--1901, 2020.

\bibitem[Campbell et~al.(2022)Campbell, Benton, De~Bortoli, Rainforth, Deligiannidis, and Doucet]{campbell2022continuous}
Andrew Campbell, Joe Benton, Valentin De~Bortoli, Thomas Rainforth, George Deligiannidis, and Arnaud Doucet.
\newblock A continuous time framework for discrete denoising models.
\newblock \emph{Advances in Neural Information Processing Systems}, 35:\penalty0 28266--28279, 2022.

\bibitem[Chang et~al.(2022)Chang, Zhang, Jiang, Liu, and Freeman]{chang2022maskgit}
Huiwen Chang, Han Zhang, Lu~Jiang, Ce~Liu, and William~T. Freeman.
\newblock Maskgit: Masked generative image transformer, 2022.

\bibitem[Chen(2023)]{chen2023importance}
Ting Chen.
\newblock On the importance of noise scheduling for diffusion models.
\newblock \emph{arXiv preprint arXiv:2301.10972}, 2023.

\bibitem[Demirkaya et~al.(2020)Demirkaya, Chen, and Oymak]{demirkaya2020exploring}
Ahmet Demirkaya, Jiasi Chen, and Samet Oymak.
\newblock Exploring the role of loss functions in multiclass classification.
\newblock In \emph{2020 54th annual conference on information sciences and systems (ciss)}, pages 1--5. IEEE, 2020.

\bibitem[Deng et~al.(2009)Deng, Dong, Socher, Li, Li, and Fei-Fei]{imagenet}
Jia Deng, Wei Dong, Richard Socher, Li-Jia Li, Kai Li, and Li~Fei-Fei.
\newblock Imagenet: A large-scale hierarchical image database.
\newblock In \emph{2009 IEEE Conference on Computer Vision and Pattern Recognition}, pages 248--255, 2009.
\newblock \doi{10.1109/CVPR.2009.5206848}.

\bibitem[Dhariwal and Nichol(2021)]{dhariwal2021diffusion}
Prafulla Dhariwal and Alexander Nichol.
\newblock Diffusion models beat gans on image synthesis.
\newblock \emph{Advances in neural information processing systems}, 34:\penalty0 8780--8794, 2021.

\bibitem[Dong et~al.(2017)Dong, Hsiao, Yang, and Yang]{dong2017musegan}
Hao-Wen Dong, Wen-Yi Hsiao, Li-Chia Yang, and Yi-Hsuan Yang.
\newblock Musegan: Multi-track sequential generative adversarial networks for symbolic music generation and accompaniment, 2017.

\bibitem[Esser et~al.(2020)Esser, Rombach, and Ommer]{esser2020taming}
Patrick Esser, Robin Rombach, and Björn Ommer.
\newblock Taming transformers for high-resolution image synthesis, 2020.

\bibitem[Gillespie(1977)]{gillespie1977exact}
Daniel~T Gillespie.
\newblock Exact stochastic simulation of coupled chemical reactions.
\newblock \emph{The journal of physical chemistry}, 81\penalty0 (25):\penalty0 2340--2361, 1977.

\bibitem[Gillespie(2001)]{gillespie2001approximate}
Daniel~T Gillespie.
\newblock Approximate accelerated stochastic simulation of chemically reacting systems.
\newblock \emph{The Journal of chemical physics}, 115\penalty0 (4):\penalty0 1716--1733, 2001.

\bibitem[Ho et~al.(2020{\natexlab{a}})Ho, Jain, and Abbeel]{DDPM}
Jonathan Ho, Ajay Jain, and Pieter Abbeel.
\newblock Denoising diffusion probabilistic models.
\newblock \emph{Advances in Neural Information Processing Systems}, 33:\penalty0 6840--6851, 2020{\natexlab{a}}.

\bibitem[Ho et~al.(2020{\natexlab{b}})Ho, Jain, and Abbeel]{ho2020denoising}
Jonathan Ho, Ajay Jain, and Pieter Abbeel.
\newblock Denoising diffusion probabilistic models.
\newblock \emph{Advances in neural information processing systems}, 33:\penalty0 6840--6851, 2020{\natexlab{b}}.

\bibitem[Ho et~al.(2022)Ho, Chan, Saharia, Whang, Gao, Gritsenko, Kingma, Poole, Norouzi, Fleet, et~al.]{ho2022imagen}
Jonathan Ho, William Chan, Chitwan Saharia, Jay Whang, Ruiqi Gao, Alexey Gritsenko, Diederik~P Kingma, Ben Poole, Mohammad Norouzi, David~J Fleet, et~al.
\newblock Imagen video: High definition video generation with diffusion models.
\newblock \emph{arXiv preprint arXiv:2210.02303}, 2022.

\bibitem[Hoogeboom et~al.(2021)Hoogeboom, Nielsen, Jaini, Forr{\'e}, and Welling]{hoogeboom2021argmax}
Emiel Hoogeboom, Didrik Nielsen, Priyank Jaini, Patrick Forr{\'e}, and Max Welling.
\newblock Argmax flows and multinomial diffusion: Learning categorical distributions.
\newblock \emph{Advances in Neural Information Processing Systems}, 34:\penalty0 12454--12465, 2021.

\bibitem[Hyv{\"a}rinen(2007)]{ratiomatch1}
Aapo Hyv{\"a}rinen.
\newblock Some extensions of score matching.
\newblock \emph{Computational statistics \& data analysis}, 51\penalty0 (5):\penalty0 2499--2512, 2007.

\bibitem[Kang and Cho(2018)]{kang2018conditional}
Seokho Kang and Kyunghyun Cho.
\newblock Conditional molecular design with deep generative models.
\newblock \emph{Journal of chemical information and modeling}, 59\penalty0 (1):\penalty0 43--52, 2018.

\bibitem[Karras et~al.(2022)Karras, Aittala, Aila, and Laine]{karras2022elucidating}
Tero Karras, Miika Aittala, Timo Aila, and Samuli Laine.
\newblock Elucidating the design space of diffusion-based generative models.
\newblock \emph{Advances in Neural Information Processing Systems}, 35:\penalty0 26565--26577, 2022.

\bibitem[Kolmogoroff(1931)]{kolmogorov_equation}
Andrei Kolmogoroff.
\newblock {\"U}ber die analytischen methoden in der wahrscheinlichkeitsrechnung.
\newblock \emph{Mathematische Annalen}, 104:\penalty0 415--458, 1931.

\bibitem[Li et~al.(2022)Li, Thickstun, Gulrajani, Liang, and Hashimoto]{li2022diffusion}
Xiang Li, John Thickstun, Ishaan Gulrajani, Percy~S Liang, and Tatsunori~B Hashimoto.
\newblock Diffusion-lm improves controllable text generation.
\newblock \emph{Advances in Neural Information Processing Systems}, 35:\penalty0 4328--4343, 2022.

\bibitem[Li et~al.(2021)Li, Pei, and Lai]{li2021structure}
Yibo Li, Jianfeng Pei, and Luhua Lai.
\newblock Structure-based de novo drug design using 3d deep generative models.
\newblock \emph{Chemical science}, 12\penalty0 (41):\penalty0 13664--13675, 2021.

\bibitem[Lou et~al.(2024)Lou, Meng, and Ermon]{lou2024discrete}
Aaron Lou, Chenlin Meng, and Stefano Ermon.
\newblock Discrete diffusion modeling by estimating the ratios of the data distribution.
\newblock \emph{arXiv preprint arXiv:2310.16834}, 2024.

\bibitem[Lyu(2009)]{ratiomatch2}
Siwei Lyu.
\newblock Interpretation and generalization of score matching.
\newblock In \emph{Proceedings of the Twenty-Fifth Conference on Uncertainty in Artificial Intelligence}, pages 359--366, 2009.

\bibitem[OpenAI(2023)]{OpenAI_GPT4_2023}
OpenAI.
\newblock Gpt-4 technical report.
\newblock \emph{ArXiv}, abs/2303.08774, 2023.
\newblock URL \url{https://arxiv.org/abs/2303.08774}.

\bibitem[Raffel(2016)]{raffel2016learning}
Colin Raffel.
\newblock \emph{Learning-Based Methods for Comparing Sequences, with Applications to Audio-to-MIDI Alignment and Matching}.
\newblock PhD thesis, Columbia University, 2016.

\bibitem[Ramesh et~al.(2022)Ramesh, Dhariwal, Nichol, Chu, and Chen]{ramesh2022hierarchical}
Aditya Ramesh, Prafulla Dhariwal, Alex Nichol, Casey Chu, and Mark Chen.
\newblock Hierarchical text-conditional image generation with clip latents.
\newblock \emph{arXiv preprint arXiv:2204.06125}, 1\penalty0 (2):\penalty0 3, 2022.

\bibitem[Rindos et~al.(1995)Rindos, Woolet, Viniotis, and Trivedi]{inhomo_CTMC}
Andy Rindos, Steven Woolet, Ioannis Viniotis, and Kishor Trivedi.
\newblock Exact methods for the transient analysis of nonhomogeneous continuous time markov chains.
\newblock In \emph{Computations with Markov Chains: Proceedings of the 2nd International Workshop on the Numerical Solution of Markov Chains}, pages 121--133. Springer, 1995.

\bibitem[Song et~al.(2021{\natexlab{a}})Song, Meng, and Ermon]{DDIM}
Jiaming Song, Chenlin Meng, and Stefano Ermon.
\newblock Denoising diffusion implicit models.
\newblock In \emph{International Conference on Learning Representations}, 2021{\natexlab{a}}.
\newblock URL \url{https://openreview.net/forum?id=St1giarCHLP}.

\bibitem[Song et~al.(2021{\natexlab{b}})Song, Sohl-Dickstein, Kingma, Kumar, Ermon, and Poole]{song2021scorebased}
Yang Song, Jascha Sohl-Dickstein, Diederik~P Kingma, Abhishek Kumar, Stefano Ermon, and Ben Poole.
\newblock Score-based generative modeling through stochastic differential equations.
\newblock In \emph{International Conference on Learning Representations}, 2021{\natexlab{b}}.
\newblock URL \url{https://openreview.net/forum?id=PxTIG12RRHS}.

\bibitem[Sun et~al.(2023)Sun, Yu, Dai, Schuurmans, and Dai]{sun2023scorebased}
Haoran Sun, Lijun Yu, Bo~Dai, Dale Schuurmans, and Hanjun Dai.
\newblock Score-based continuous-time discrete diffusion models.
\newblock In \emph{The Eleventh International Conference on Learning Representations}, 2023.
\newblock URL \url{https://openreview.net/forum?id=BYWWwSY2G5s}.

\bibitem[Tewari and Bartlett(2007)]{tewari2007consistency}
Ambuj Tewari and Peter~L Bartlett.
\newblock On the consistency of multiclass classification methods.
\newblock \emph{Journal of Machine Learning Research}, 8\penalty0 (5), 2007.

\bibitem[van~den Oord et~al.(2018)van~den Oord, Vinyals, and Kavukcuoglu]{oord2018neural}
Aaron van~den Oord, Oriol Vinyals, and Koray Kavukcuoglu.
\newblock Neural discrete representation learning, 2018.

\bibitem[Zhang et~al.(2023)Zhang, Zhang, Zhang, and Kweon]{zhang2023text}
Chenshuang Zhang, Chaoning Zhang, Mengchun Zhang, and In~So Kweon.
\newblock Text-to-image diffusion model in generative ai: A survey.
\newblock \emph{arXiv preprint arXiv:2303.07909}, 2023.

\bibitem[Zheng et~al.(2023)Zheng, Yuan, Yu, and Kong]{RPM}
Lin Zheng, Jianbo Yuan, Lei Yu, and Lingpeng Kong.
\newblock A reparameterized discrete diffusion model for text generation.
\newblock \emph{arXiv preprint arXiv:2302.05737}, 2023.

\end{thebibliography}

\clearpage
\appendix
\section{Appendix}

\subsection{Variational Lower Bound Derivation}
\label{ssec:vlbderive}

\begin{align}
    \log &\int p_{\theta}(\rvx_{0:T}) d\rvx_{1:T} \geq \E_{q(\rvx_{1:T}|\rvx_0)} \big[\log p_{\theta}(\rvx_{0:T}) - \log q(\rvx_{1:T}|\rvx_0)\big] \nonumber\\
    &= \E_{q(\rvx_{1:T}|\rvx_0)} \big[
        \log p_{\theta}(\rvx_T) + \sum_{t=1}^T \log \frac{p_{\theta}(\rvx_{t-1}| \rvx_{t} )}{q(\rvx_{t}| \rvx_{t-1} )} 
    \big] \nonumber\\
    & = \E_{q(\rvx_{1:T}|\rvx_0)} \big[
        \log p_{\theta}(\rvx_T) + \log \frac{p_{\theta}(\rvx_{0}| \rvx_{1} )}{q(\rvx_{1}| \rvx_{0} )} +
        \sum_{t=2}^T \log \frac{p_{\theta}(\rvx_{t-1}| \rvx_{t} )}{q(\rvx_{t}| \rvx_{t-1}, \rvx_0 )} 
    \big] \nonumber\\
    & = \E_{q(\rvx_{1:T}|\rvx_0)} \big[
        \log p_{\theta}(\rvx_T) + \log \frac{p_{\theta}(\rvx_{0}| \rvx_{1} )}{q(\rvx_{1}| \rvx_{0} )} +
        \sum_{t=2}^T \log 
        \Bigg (\frac{p_{\theta}(\rvx_{t-1}| \rvx_{t} )}{q(\rvx_{t-1}| \rvx_{t}, \rvx_0 )} \cdot 
        \frac{q(\rvx_{t-1}| \rvx_0 )}{q(\rvx_{t}| \rvx_0 ) }   \Bigg )
    \big] \nonumber\\
    & = \E_{q(\rvx_{1:T}|\rvx_0)} \big[ 
        \log p_{\theta}(\rvx_T) + \log \frac{p_{\theta}(\rvx_{0}| \rvx_{1} )}{q(\rvx_{1}| \rvx_{0} )} + \log \frac{q(\rvx_1 | \rvx_0)}{q(\rvx_T | \rvx_0)} +
        \sum_{t=2}^T \log \frac{p_{\theta}(\rvx_{t-1}| \rvx_{t} )}{q(\rvx_{t-1}| \rvx_{t}, \rvx_0 )}
    \big] \nonumber\\
    & = \E_{q(\rvx_{1:T}|\rvx_0)} \big[ \log p_{\theta}(\rvx_{0}| \rvx_{1} ) + \log \frac{p_{\theta}(\rvx_T)}{q(\rvx_T | \rvx_0)} -
        \sum_{t=2}^T \log \frac{q(\rvx_{t-1}| \rvx_{t}, \rvx_0 )}{p_{\theta}(\rvx_{t-1}| \rvx_{t} )}
    \big] \nonumber\\
    &= \underbrace{\E_{q(\rvx_{1}|\rvx_0)}\big[ \log p_{\theta}(\rvx_{0}| \rvx_{1} )\big]}_{-\Ls_1(\theta)}
      - \sum_{t=2}^T \underbrace{\E_{q(\rvx_{t}|\rvx_0)}\big[ \KL\big(q(\rvx_{t-1} | \rvx_{t}, \rvx_0 ) || p_{\theta}(\rvx_{t-1}| \rvx_{t} \big)  \big]}_{\Ls_t(\theta)}
      - \text{const.} \label{eq:vlb_1}
\end{align}

\begin{figure}[h]
    \centering
    \includegraphics[width=\textwidth]{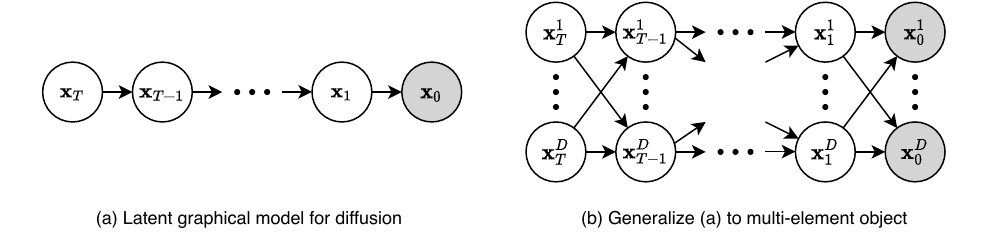}
    \caption{The graphical model diagram for both element-wise diffusion and multi-element diffusion. }
    \label{fig:pgm}
\end{figure}

\subsection{Definition of nominal data}
\label{ssec:nominal}
Nominal data, or categorical data, refers to data that can be categorized but not ranked or ordered. In mathematics, nominal data can be represented as a set of discrete categories or distinct labels without inherent ordering. As there is no ordering inside, operations of comparing categories is meaningless except checking equality. In the real world, most categories of data are nominal in nature, such as nationality, blood type, colors, among others.

\subsection{Derivation of $q(x_{t-1} | x_t, x_0)$}
\label{ssec:t-1|t,0}
First, let us define $\overline{Q}_{t|s} = Q_{s+1}...Q_{t}$. Note that $  \overline{Q}_{t|0} = \overline{Q}_t$ and $\overline{Q}_{t|t-1} = Q_t$. Accordingly, we  can derive the following two equalities.
\begin{align}
    q(\rvx_t | \rvx_s) = \text{Cat}(\rvx_t; \overline{Q}_{t|s}^\top\rvx_s ) \;
\end{align} 
\begin{align}
    q(\rvx_{s} | \rvx_t, \rvx_0) &= \frac{q(\rvx_t | \rvx_{s}) q(\rvx_{s}| \rvx_0)}{q(\rvx_t | \rvx_0)} = \frac{\text{Cat}(\rvx_t; \overline{Q}^\top_{t|s}\rvx_{s})\text{Cat}(\rvx_{s} ; \overline{Q}_{s}^\top\rvx_0)}{\text{Cat}(\rvx_t ; \overline{Q}_t^\top\rvx_0)}
    \nonumber \\
    &= \frac{\rvx_{s}^\top \overline{Q}_{t|s} \rvx_{t} \cdot \rvx_{s}^\top\overline{Q}_{s}^\top\rvx_0}{\rvx_t^\top\overline{Q}_t^\top\rvx_0 } 
    = \rvx_{s}^\top \frac{\overline{Q}_{t|s} \rvx_{t} \odot \overline{Q}_{s}^\top\rvx_0}{\rvx_t^\top\overline{Q}_t^\top\rvx_0 } = \text{Cat}(\rvx_{s};\frac{\overline{Q}_{t|s} \rvx_{t} \odot \overline{Q}_{s}^\top\rvx_0}{\rvx_t^\top\overline{Q}_t^\top\rvx_0 })
\end{align}

Using the formulation of $Q_t$ in \eqref{eq:Qt}, $\overline{Q}_{t|s}$ can be written as 
\begin{align}
    \overline{Q}_{t|s} = \overline{\alpha}_{t|s} I + (1-\overline{\alpha}_{t|s} )\1 \vm^\top  \;
\end{align}
where $\overline{\alpha}_{t|s} = \prod_{i=s+1}^t \alpha_i $. Note that $ \overline{\alpha}_{t} = \overline{\alpha}_{t|s} \overline{\alpha}_{s}$.

Next, we can simplify \eqref{eq:t|s,0} using the above formulations as 
\begin{align}
    &\frac{\overline{Q}_{t|s} \rvx_{t} \odot \overline{Q}_{s}^\top\rvx_0}{\rvx_t^\top\overline{Q}_t^\top\rvx_0 }  = \frac
        {
            \big(\overline{\alpha}_{t|s}\rvx_t + (1-\overline{\alpha}_{t|s}) \langle \vm, \rvx_t \rangle \1\big)\odot \big(\overline{\alpha}_s \rvx_0 + (1-\overline{\alpha}_s) \vm \big)
        }
        {
            \overline{\alpha}_t \langle \rvx_t, \rvx_0 \rangle + (1 - \overline{\alpha}_t) \langle \vm, \rvx_t \rangle 
        }  \nonumber\\ 
    &= \frac
        {
            \overline{\alpha}_t \rvx_t\odot\rvx_0 + (\overline{\alpha}_s - \overline{\alpha}_t )\langle \vm, \rvx_t \rangle  \rvx_0 
            + (\overline{\alpha}_{t|s} - \overline{\alpha}_t)\rvx_t \odot \vm
            + (1- \overline{\alpha}_s)(1-\overline{\alpha}_{t|s})
             \langle \vm, \rvx_t \rangle\vm
        }
        {
            \overline{\alpha}_t \langle \rvx_t, \rvx_0 \rangle + (1 - \overline{\alpha}_t) \langle \vm, \rvx_t \rangle 
        } \label{eq:t|s,0-simplify0}
\end{align}
We can simplify the above equation further by considering two cases: (1) $\rvx_t = \rvx_0$ and (2) $\rvx_t \neq \rvx_0 $. 

\textbf{(Case 1)} When $\rvx_t = \rvx_0$, using the fact that both $\rvx_t$ and $\rvx_0$ are one-hot encoded, we observe 
\begin{align}
    \text{\eqref{eq:t|s,0-simplify0}} &= \frac
        {
            \overline{\alpha}_t \rvx_t +  (\overline{\alpha}_s - \overline{\alpha}_t )\langle \vm, \rvx_t \rangle  \rvx_t + 
             (\overline{\alpha}_{t|s}- \overline{\alpha}_t )\langle \vm, \rvx_t \rangle  \rvx_t + 
             (1- \overline{\alpha}_s)(1-\overline{\alpha}_{t|s})
             \langle \vm, \rvx_t \rangle\vm
        }
        {
            \overline{\alpha}_t  + (1 - \overline{\alpha}_t) \langle \vm, \rvx_t \rangle 
        } \nonumber\\
    &=\frac
        {
             \overline{\alpha}_t + ( \overline{\alpha}_{t|s} +  \overline{\alpha}_s - 2  \overline{\alpha}_t) \langle \vm, \rvx_t \rangle 
        }        
        {
            \overline{\alpha}_t  + (1 - \overline{\alpha}_t) \langle \vm, \rvx_t \rangle 
        }\rvx_t
    + \frac
        {
            (1- \overline{\alpha}_s)(1-\overline{\alpha}_{t|s}) \langle \vm, \rvx_t \rangle
        }        
        {
            \overline{\alpha}_t  + (1 - \overline{\alpha}_t) \langle \vm, \rvx_t \rangle 
        }\vm \nonumber \\
    &= (1 - \lambda_{t|s}) \cdot \rvx_t + \lambda_{t|s} \cdot \vm \;, \label{eq:s|t=0}
\end{align}
where $\lambda_{t|s}$ is defined as 
\begin{align}
    \lambda_{t|s} = \frac{(1- \overline{\alpha}_s)(1-\overline{\alpha}_{t|s}) \langle \vm, \rvx_t \rangle}{ \overline{\alpha}_t  + (1 - \overline{\alpha}_t) \langle \vm, \rvx_t \rangle } \;. 
\end{align}

\textbf{(Case 2)} When $\rvx_t \neq \rvx_0$,  we can similarly derive 
\begin{align}
  \text{\eqref{eq:t|s,0-simplify0}} &=  \frac
        {
             (\overline{\alpha}_s - \overline{\alpha}_t )\langle \vm, \rvx_t \rangle  \rvx_0 + 
             (\overline{\alpha}_{t|s}- \overline{\alpha}_t )\langle \vm, \rvx_t \rangle  \rvx_t + 
             (1- \overline{\alpha}_s)(1-\overline{\alpha}_{t|s})
             \langle \vm, \rvx_t \rangle\vm
        }
        {(1-\overline{\alpha}_t ) \langle \vm, \rvx_t \rangle} \nonumber\\
        & = 
        \frac{\overline{\alpha}_s - \overline{\alpha}_t}{1-\overline{\alpha}_t}\rvx_0 + 
        \frac{1- \overline{\alpha}_s}{1-\overline{\alpha}_t}\overline{\alpha}_{t|s}\rvx_t + 
        \frac{1- \overline{\alpha}_s}{1-\overline{\alpha}_t}(1 - \overline{\alpha}_{t|s})\vm \nonumber \\ 
        & = (1 - \mu_{t|s}) \cdot \rvx_0 + \mu_{t|s}\overline{\alpha}_{t|s} \cdot \rvx_t + 
        \mu_{t|s}(1-\overline{\alpha}_{t|s}) \cdot \vm \;, \label{eq:s|t!=0}
\end{align}
where $\mu_{t|s}$ is defined as 
\begin{align}
    \mu_{t|s} = \frac{1- \overline{\alpha}_s}{1-\overline{\alpha}_t} \;
\end{align}
Combining the results from both cases together, we can write $q(\rvx_s| \rvx_t, \rvx_0)$ in the following form.
\begin{align}
    q(\rvx_s| \rvx_t, \rvx_0) = 
    \begin{cases}
        \text{Cat}\Big(\rvx_s; \ (1 - \lambda_{t|s}) \cdot \rvx_t + \lambda_{t|s} \cdot \vm\Big)   & \text{when }\rvx_t = \rvx_0\\
         \text{Cat}\Big(\rvx_s; \ (1 - \mu_{t|s}) \cdot \rvx_0 + \mu_{t|s}\overline{\alpha}_{t|s} \cdot \rvx_t + 
        \mu_{t|s}(1-\overline{\alpha}_{t|s}) \cdot \vm \Big)    & \text{when }\rvx_t \neq \rvx_0 
    \end{cases}
\end{align}

\subsection{Parameterization of $p_{\theta} (x_s |x_t) $}
\label{ssec:p_x_s_x_t}
We first provide the formulation as follows.
\begin{align}
    p_\theta(\rvx_{s} | \rvx_t) = \sum_{\rvx_0} q(\rvx_s | \rvx_t , \rvx_0) p_\theta(\rvx_0 | \rvx_t)
\end{align}
Using $q(\rvx_s | \rvx_t, \rvx_0)$ in \eqref{eq:s|t,0 full} and $p_\theta(\rvx_0 | \rvx_t) = \text{Cat}(\rvx_0; f_t^\theta(\rvx_t))$, we can expand it as

\begin{align}
    p_\theta(\rvx_s | \rvx_t )& = q(\rvx_s|\rvx_t,\rvx_t)p_\theta(\rvx_t|\rvx_t) + \sum_{\rvx \neq \rvx_t} q(\rvx_s|\rvx_t,\rvx) p_\theta(\rvx|\rvx_t) \nonumber\\
    =   & \rvx_s^\top \Big( (1-\lambda_{t|s})\rvx_t + \lambda_{t|s}\vm \Big)\rvx_t^\top 
           f_t^\theta(\rvx_t)  \nonumber\\
        & + \sum_{\rvx \neq \rvx_t} \rvx_s^T\Big((1-\mu_{t|s})\rvx + 
        \mu_{t|s}\overline{\alpha}_{t|s}\rvx_t + \mu_{t|s}(1-\overline{\alpha}_{t|s})\vm  \Big) \rvx^\top f_t^\theta(\rvx_t) \nonumber\\
    =& \rvx_s^\top \Big[  
            \big( (1-\lambda_{t|s})\rvx_t + \lambda_{t|s}\vm \big)\rvx_t^\top f_t^\theta(\rvx_t) 
            + (1-\mu_{t|s})(\sum_{\rvx \neq \rvx_t}\rvx \rvx^\top)f_t^\theta(\rvx_t) \nonumber\\
            &+ (\mu_{t|s}\overline{\alpha}_{t|s}\rvx_t + \mu_{t|s}(1-\overline{\alpha}_{t|s})\vm)
            (\sum_{\rvx \neq \rvx_t} \rvx)^\top f_t^\theta(\rvx_t)
        \Big] \nonumber\\
    =& \rvx_s^\top \Big[
            \big( (1-\lambda_{t|s})\rvx_t + \lambda_{t|s}\vm \big)\rvx_t^\top f_t^\theta(\rvx_t) 
            + (1-\mu_{t|s})(I - \rvx_t\rvx_t^\top)f_t^\theta(\rvx_t) \nonumber\\
            &+(\mu_{t|s}\overline{\alpha}_{t|s}\rvx_t + \mu_{t|s}(1-\overline{\alpha}_{t|s})\vm)
            (\1 - \rvx_t)^\top f_t^\theta(\rvx_t) 
    \Big] \nonumber \\
    = & \rvx_s^\top \Big[
            (1-\lambda_{t|s})\rvx_t^\top f_t^\theta(\rvx_t) \rvx_t + \lambda_{t|s}\rvx_t^\top f_t^\theta(\rvx_t) \vm      
            + (1-\mu_{t|s})(f_t^\theta(\rvx_t) - \rvx_t^\top f_t^\theta(\rvx_t)\rvx_t) \nonumber\\
            &+(\mu_{t|s}\overline{\alpha}_{t|s}\rvx_t + \mu_{t|s}(1-\overline{\alpha}_{t|s})\vm)(1 - \rvx_t^T f_t^\theta(\rvx_t))
    \Big] \nonumber \\
    = & \rvx_s^\top \Big[ (1-\mu_{t|s})\cdot f_t^\theta(\rvx_t)
     + \Big( (\mu_{t|s} - \lambda_{t|s} - \mu_{t|s}\overline{\alpha}_{t|s})\rvx_t^\top f_t^\theta(\rvx_t) + \mu_{t|s}\overline{\alpha}_{t|s} \Big)\cdot\rvx_t \nonumber\\
     & + \Big( -(\mu_{t|s} - \lambda_{t|s} - \mu_{t|s}\overline{\alpha}_{t|s})\rvx_t^\top f_t^\theta(\rvx_t) + \mu_{t|s}(1-\overline{\alpha}_{t|s}) \Big)\cdot\vm 
    \Big] \nonumber\\
    =&\quad \text{Cat}\Big(\rvx_s; \ 
        (1-\mu_{t|s})\cdot f_t^\theta(\rvx_t) + 
        (\mu_{t|s}\overline{\alpha}_{t|s} + \gamma^\theta_{{t|s}})\cdot \rvx_t  + 
        (\mu_{t|s}(1-\overline{\alpha}_{t|s}) - \gamma^\theta_{{t|s}})\cdot \vm
    \Big)
\end{align}

\subsection{Proof of Proposition 3}\label{ssec:proof_delta_pq}
As $q(\rvx_s|\rvx_t, \rvx_0)$ has two different categorical distributions based on whether $\rvx_t$ is equal to $\rvx_0$, we prove the Proposition 3 based on two cases.

\textbf{Case 1: $\rvx_t = \rvx_0$}. 
In this case, let $h^\theta_t := f_t^\theta(\vx_t) - \vx_0$, now replace $f_t^\theta $ with $g_t^\theta$ in \eqref{eq:gamma} and \eqref{eq:s|t}, we get 
\begin{align}
    \gamma_{t|s}^\theta = (\mu_{t|s} - \lambda_{t|s} - \mu_{t|s}\overline{\alpha}_{t|s})
    \langle h^\theta_t + \vx_0 , \vx_t \rangle  = (\mu_{t|s} - \lambda_{t|s} - \mu_{t|s}\overline{\alpha}_{t|s})
    \langle h^\theta_t , \vx_t \rangle + (\mu_{t|s} - \lambda_{t|s} - \mu_{t|s}\overline{\alpha}_{t|s})
\end{align}
Taking the formulation into \eqref{eq:s|t} and cancel out terms, we can simplify $p_\theta(\rvx_s|\rvx_t)$ as 
\begin{align}
p_\theta(\rvx_s|\rvx_t) &= (1-\mu_{t|s})h^\theta_t + (1-\lambda_{t|s})\vx_t + \lambda_{t|s} \vm + (\mu_{t|s} - \lambda_{t|s} - \mu_{t|s}\overline{\alpha}_{t|s}) \langle h^\theta_t, \vx_t \rangle(\vx_t - \vm) \nonumber\\
&= q(\rvx_s|\rvx_t, \rvx_0) + (1-\mu_{t|s})\Big[ h^\theta_t + \frac{\mu_{t|s} - \lambda_{t|s} - \mu_{t|s}\overline{\alpha}_{t|s}}{1-\mu_{t|s}} \langle h^\theta_t, \vx_t \rangle(\vx_t - \vm)  \Big] \nonumber\\
&=  q(\rvx_s|\rvx_t, \rvx_0) + (1-\mu_{t|s})\Big[ h^\theta_t + \phi_{t|s}\langle h^\theta_t, \vx_t \rangle(\vx_t - \vm)  \Big]
\end{align}

\textbf{Case 2: $\rvx_t \neq \rvx_0$.} Based on \eqref{eq:s|t,0 full} and \eqref{eq:s|t}, it's easy to find that 
\begin{align}
    p_\theta(\rvx_s| \rvx_t) &= (1-\mu_{t|s})(f^\theta_t(\vx_t)-\vx_0) + \gamma_{t|s}^\theta (\vx_t - \vm) + (1-\mu_{t|s})\vx_0 + \mu_{t|s}\bar{\alpha}_{t|s}\vx_t + \mu_{t|s}(1-\bar{\alpha}_{t|s} ) \vm \nonumber\\
    &= (1-\mu_{t|s})(f^\theta_t(\vx_t)-\vx_0) + \gamma_{t|s}^\theta (\vx_t - \vm) + q(\rvx_s|\rvx_t, \rvx_0) \nonumber\\
    &= q(\rvx_s|\rvx_t, \rvx_0) + (1-\mu_{t|s})\Big[f^\theta_t(\vx_t)-\vx_0 + \frac{\gamma_{t|s}^\theta}{(1-\mu_{t|s})} (\vx_t - \vm)\Big] 
\end{align}
Taking the definition of $\gamma_{t|s}^\theta$ in \eqref{eq:gamma} and using the fact $<\vx_0, \vx_t> = 0$ we get the formulation in proposition 3. 

For both cases, we proved that the formulation in proposition 3 is correct. 

\subsection{Proof of KL divergence approximation}\label{ssec:kl_approx}
Assume that $p(\vx ) = q(\vx)+ \triangle p(\vx)$, and both $p(\vx )$  and $q(\vx )$ are valid categorical distributions. Then 
\begin{align}
 \KL\big( q(\vx) \Vert p(\vx ) \big) & = -\sum_\vx q(\vx) \log \frac{p(\vx)}{q(\vx)} 
   = -\sum_\vx q(\vx) \log (1 + \frac{\triangle p(\vx)}{q(\vx)})
\end{align}
By Taylor expansion, $\log (1+x) \approx x - \frac{x^2}{2}$, we ignore larger than 2 order terms. Then apply this we get 
\begin{align}
    \KL\big( q(\vx) \Vert p(\vx ) \big) & = -\sum_\vx q(\vx) [\frac{\triangle p(\vx)}{q(\vx)} - (\frac{\triangle p(\vx)}{q(\vx)})^2] \nonumber\\
    & = -\sum_\vx \triangle p(\vx)  + \sum_\vx \frac{\triangle p(\vx)^2 }{q(\vx)} \nonumber\\
    & = \sum_\vx \frac{\triangle p(\vx)^2 }{q(\vx)}
\end{align}

Change probabilities $q$ and $p$ to $q_(\rvx_s|\rvx_t, \rvx_0)$ and $p_\theta(\rvx_s| \rvx_t)$ we get the targeted formulation.

\subsection{VLB with partial time steps}\label{ssec:vlb_partial}
Assume that we only observe $\rvx_t$, $\rvx_s$, and $\rvx_0$, and assume that we have learned the prior $p_\theta(\rvx_t)$. 
\begin{align}
    \log p_\theta(\rvx_0) &= \log \E_{q(\rvx_t, \rvx_s | \rvx_0)}\frac{p_\theta(\rvx_t,\rvx_s,\rvx_0)}{q(\rvx_t,\rvx_s|\rvx_0)} \ge \E_{q(\rvx_t,\rvx_s|\rvx_0)}\log \frac{p_\theta(\rvx_t,\rvx_s,\rvx_0)}{q(\rvx_t,\rvx_s|\rvx_0)} \nonumber\\
    &= \E_{q(\rvx_t, \rvx_s| \rvx_0)} [\log \frac{p_\theta(\rvx_t)p_\theta(\rvx_s|\rvx_t) p_\theta(\rvx_0|\rvx_s)}{q(\rvx_t|\rvx_0) q(\rvx_s|\rvx_t, \rvx_0)}] \nonumber\\
    &=\E_{q(\rvx_t, \rvx_s| \rvx_0)} \Big[ \log \frac{p_\theta(\rvx_t)}{q(\rvx_t|\rvx_0)}
    + \log p_\theta(\rvx_t) + \log \frac{p_\theta(\rvx_s|\rvx_t)}{q(\rvx_s|\rvx_t,\rvx_0)} \Big] \nonumber\\
    &= -\KL[q(\rvx_t|\rvx_0) \Vert p_\theta(\rvx_t) ]
    + \E_{q(\rvx_s|\rvx_0)}\log p_\theta(\rvx_0|\rvx_s) - \KL[q(\rvx_s|\rvx_t,\rvx_0) \Vert p_\theta(\rvx_s|\rvx_t)]
\end{align}

\subsection{Reparameterization Form for Sampling}
\label{ssec:reparameter}

In practice, we need to sample $\rvx_{t|0} \sim q(\rvx_t| \rvx_0)$ for training and $\rvx_{s|t} \sim p_\theta(\rvx_s | \rvx_t)$ for  generation (backward denoising). In what follows, we provide the reparameterization form for these to facilitate fast sampling. Given
$q(\rvx_t | \rvx_0) = \text{Cat} (\rvx_t ;  \overline{\alpha}_{t}\rvx_0 + (1-\overline{\alpha}_t )\vm )$ and $p_\theta(\rvx_s|\rvx_t)$ as in \eqref{eq:s|t}, 
we can rewrite the corresponding variables as  
\begin{align}
    \rvx_{t|0} & = \delta_{1, \rvb_t} \rvx_0 + (1-  \delta_{1, \rvb_t})\rvm_0,  
    \text{ where } \rvb_t \sim \text{Bernoulli}(\overline{\alpha}_t)
    \label{eq:reparam_t|0} \\
    \rvx_{s|t}  & = \delta_{1, \rvb_{s|t}} \tilde{\rvx}_{0|t} \;+\; \delta_{2, \rvb_{s|t}} \rvx_{t} \;+\; \delta_{3, \rvb_{s|t}} \rvm_0, \label{eq:reparam_s|t} 
\end{align}
where $\tilde{\rvx}_{0|t}\sim \text{Cat}(f_t^\theta(\rvx_t)) $
 and  $\rvb_{s|t} \sim $
 $ \text{Cat}(\cdot;[1-\mu_{t|s}, \mu_{t|s}\overline{\alpha}_{t|s} + \gamma^\theta_{{t|s}}, \ \mu_{t|s}(1-\overline{\alpha}_{t|s}) - \gamma^\theta_{{t|s}}])$. 

\eqref{eq:reparam_t|0} and \eqref{eq:reparam_s|t} essentially show that the sampling process can be divided into two steps: first, sample the branch indicator $\rvb_t$ (or $\rvb_{s|t}$), and then sample from the categorical distribution of that branch, i.e. $\rvx_t$, $\rvm_0$,  or $\tilde{\rvx}_{0|t}$. Moreover, the three terms in \eqref{eq:reparam_s|t} highlight that the denoising step of generating $\rvx_s $ from $\rvx_t$ essentially draws samples via three levers: (1) use the predicted sample $\rvx_0$ from the trained network $f_t^\theta(\rvx_t)$ directly, (2) keep it unchanged as $\rvx_t$, or (3) roll it back to noise $\rvm_0$, offering an intuitive understanding.

\subsection{Discrete-Time Multi-element Object Extension}\label{ssec:discrete_multi_element_app}

We first extend $q(\rvx_t|\rvx_0)$, $q(\rvx_s | \rvx_t, \rvx_0)$, and $p_\theta(\rvx_s | \rvx_t)$ to multi-element object, and then present the negative VLB loss extension. First, for the forward process conditional on $\rvx_0$, each element $\rvx^d_t$ of the multi-element object $\rvx^{1:D}_{t}$ has its own diffusion process without interactions with others. Formally,
\begin{align}
    q(\rvx_t^{1:D}|\rvx_0^{1:D}) = \prod_{d=1}^D  q(\rvx_t^{d}|\rvx_0^{d}) \quad, \text{and} \quad q(\rvx^{1:D}_s | \rvx^{1:D}_t, \rvx^{1:D}_0) = \prod_{d=1}^D q(\rvx^d_s | \rvx^d_t, \rvx^d_0) \;.
\end{align}
The corresponding forward reparameterization form for generating $\rvx^{1:D}_{t|0}$ can be updated as 
\begin{align}
    \label{eq:forward_sampling}
    \rvx^{1:D}_{t|0} & = \delta^{1:D}_{1, \rvb^{1:D}_t} \odot \rvx^{1:D}_0 + (1-  \delta^{1:D}_{1, \rvb^{1:D}_t})\odot\rvm^{1:D}_0 \quad ,
\end{align}
where $\delta^{1:D}_{1, \rvb^{1:D}_t} $ is a $D$-dimensional vector with $d$-th element being $\delta_{1, \rvb^d_t}$ and $\rvb^d_t \sim \text{Bernoulli}(\overline{\alpha}_t)$.

For the backward process, all elements' transition processes are coupled together, as shown in the graphical model in Appx. \figref{fig:pgm}(b). We start with the parameterization $p_\theta(\rvx^{1:D}_0 | \rvx^{1:D}_t)$, 
with the form
\begin{align}
    p_\theta(\rvx^{1:D}_0 | \rvx^{1:D}_t) = \prod_{d=1}^D p_\theta(\rvx^{d}_0 | \rvx^{1:D}_t) \quad \text{with} \quad p_\theta(\rvx^{d}_0 | \rvx^{1:D}_t) = \text{Cat}\big(\rvx_0^d | f_t^{\theta, d}( \rvx_t^{1:D})\big)
\end{align}
Then, $p_\theta(\rvx^{1:D}_s|\rvx^{1:D}_t) = \prod_{d=1}^D p_\theta(\rvx^{d}_s|\rvx^{1:D}_t)$, with the component $p_\theta(\rvx^{d}_s|\rvx^{1:D}_t)$ has the form
\begin{gather}
     \text{Cat}\Big(\rvx^d_s; \ 
        (1-\mu_{t|s})f_t^{\theta, d}( \rvx_t^{1:D}) + 
        (\mu_{t|s}\overline{\alpha}_{t|s} + \gamma^{\theta, d}_{{t|s}}) \rvx_t^d  + 
        (\mu_{t|s}(1-\overline{\alpha}_{t|s}) - \gamma^{\theta, d}_{{t|s}}) \vm^d
    \Big) \;, \\
    \text{where } \gamma^{\theta, d}_{{t|s}} =  
    (\mu_{t|s} - \lambda^d_{t|s} - \mu_{t|s}\overline{\alpha}_{t|s})
    \langle f_t^{\theta, d}( \rvx_t^{1:D}), \rvx^d_t \rangle, 
    \ 
        \lambda^d_{t|s} = \frac{(1- \overline{\alpha}_s)(1-\overline{\alpha}_{t|s}) \langle \vm^d, \rvx^d_t \rangle}{ \overline{\alpha}_t  + (1 - \overline{\alpha}_t) \langle \vm^d, \rvx^d_t \rangle } \;.
    \nonumber
\end{gather}
Similarly, the reparameterization form can be expressed as 
\begin{align}
    \rvx_{s|t}^{1:D} = \delta^{1:D}_{1, \rvb^{1:D}_{s|t}} \odot \tilde{\rvx}^{1:D}_{0|t} \;+\; \delta^{1:D}_{2, \rvb^{1:D}_{s|t}} \odot \rvx_{t}^{1:D} \;+\; \delta^{1:D}_{3, \rvb^{1:D}_{s|t}} \odot \rvm_0^{1:D} \;,
\end{align}
where $\tilde{\rvx}^d_{0|t}\sim \text{Cat}(f_t^{\theta, d}( \rvx_t^{1:D}))$ and $\rvb^d_{s|t} \sim \text{Cat}(1-\mu_{t|s}, \ \mu_{t|s}\overline{\alpha}_{t|s} + \gamma^{\theta,d}_{{t|s}}, \ \mu_{t|s}(1-\overline{\alpha}_{t|s}) - \gamma^{\theta,d}_{{t|s}})$.

Finally, the following reformulate the loss functions $\Ls_t(\theta)$  for multi-element objects, where we drop the constant term for simplicity.
\begin{align}
    \Ls^{1:D}_t(\theta) =& \E_{q(\rvx^{1:D}_t| \rvx^{1:D}_0)} \sum_{d=1}^{D} \Big[  \delta_{\rvx^d_t, \rvx^d_0}\E_{q(\rvx^d_{t-1}| \rvx^d_t = \rvx^d_0)}[-\log p_\theta(\rvx^d_{t-1} | \rvx^{1:D}_t)] + \nonumber\\
    & \qquad \qquad \qquad(1-\delta_{\rvx^d_t, \rvx^d_0})\E_{q(\rvx^d_{t-1}| \rvx^d_t\neq \rvx^d_0)}[-\log p_\theta(\rvx^d_{t-1} | \rvx^{1:D}_t)] \Big] 
\end{align}

\subsection{CTMC introduction}
\label{ssec:ctmc-intro}
 The CTMC process can go with either increasing $t$ or decreasing $t$, and we use increasing $t$ as the default direction to introduce CTMC. 
 For any CTMC, its \textit{transition rate matrix} $R_t$ with  $[R_t]_{ij} = r_t(\ve_j| \ve_i)$ fully determines the underlying stochastic process. CTMC is categorized into time-homogeneous and time-\textit{in}homogeneous based on whether $R_t$ is static with respect to $t$. In this paper we work with time-inhomogeneous CTMC. 

 Based on the definition in \eqref{eq:r_t limit}, we can derive some properties as follows
\begin{gather}
    \forall \vx,  r_t(\vx|\vx) \leq 0;\qquad \forall \vy \neq \vx,  r_t(\vy| \vx) \geq 0; \qquad \forall \vx, \sum_{\vy} r_t(\vy | \vx ) = 0  \label{eq:transition_rate_properties} \\
    q_{t| t - \triangle t} (\vy | \vx) = \delta_{\vx, \vy} + r_t(\vy| \vx)\triangle t + o(\triangle t)  \label{eq:local_transition_prob_change}
\end{gather}
\eqref{eq:transition_rate_properties} shows the properties of transition rate, which imply $r_t(\vx|\vx) = - \sum_{\vy \neq \vx} r_t(\vy|\vx)$. \eqref{eq:local_transition_prob_change} characterizes the infinitesimal change of transition probability, and can be used to derive the relationship between transition matrix $\overline{Q}_{t|s}$ and transition rate matrix $R_t$. 
Formally, a CTMC's transition probabilities satisfies the Kolmogorov forward and backward equations \citep{kolmogorov_equation} :
\begin{align}
    \frac{d}{dt}q_{t|s} (\vy| \vx) = \sum_{\rvz}q_{t|s}(\vz | \vx) r_t(\vy | \vz)  \quad &\text{or} \quad
    \frac{d}{dt}\overline{Q}_{t|s}  = \overline{Q}_{t|s} R_t   &\text{Kolmogorov Forward} \label{eq:kol_forward}\\ 
     \frac{d}{ds}q_{t|s} (\vy| \vx) = -\sum_{\vz}r_s(\vz | \vx) q_{t|s}(\vy | \vz)  \quad &\text{or} \quad
     \frac{d}{ds}\overline{Q}_{t|s}  = -R_s\overline{Q}_{t|s} &\text{Kolmogorov Backward} \label{eq:kol_backward}
\end{align}

The above equations are Ordinary Different Equations (ODEs) and have unique solutions. \cite{inhomo_CTMC} mentioned that when $R_{t_1}$ and $R_{t_2}$ commute (i.e. 
 $R_{t_1}R_{t_2} = R_{t_2}R_{t_1} $) for any $t_1,t_2$, the transition probability matrix can be written as 
\begin{align}
    \overline{Q}_{t|s} = \exp\Big(\int_s^t R_a da \Big) \; \label{eq:rate_to_transit_0}
\end{align}
where $\exp(M):=\sum_{k=0}^{\infty}  \frac{M^k}{k!}$ is the matrix exponential operation. The commutative property of $R_t$ can be achieved by choosing $R_t = \beta(t) R_b$ where $R_b \in \R^{K\times K}$ is a time-independent base transition rate matrix satisfying properties in \eqref{eq:transition_rate_properties} and $\beta(t)$ is a positive scalar dependent on time. Now assume $R_b$ is diagonalizable and $R_b = U\Sigma U^{-1}$ by eigen-decomposition, then we can simplify \eqref{eq:rate_to_transit_0} as 
\begin{align}
  \overline{Q}_{t|s} = \exp\Big(\int_s^t \beta(t) R_b da \Big)   = \exp(\overline{\beta}_{t|s} R_b) = \exp(\overline{\beta}_{t|s}U\Sigma U^{-1}) = U\exp(\overline{\beta}_{t|s}\Sigma) U^{-1}
\end{align}
where $\overline{\beta}_{t|s}:= \int_s^t \beta(a) da$. 

\subsection{Derivation of \eqref{eq:Qbar_ts_continuous}}
\label{ssec:derivation-qbar}
\begin{align}
    \overline{Q}_{t|s} = \exp(\overline{\beta}_{t|s}R_b) &= I + \sum_{k=1}^\infty \frac{(-\overline{\beta}_{t|s})^k (-R_b)^k}{k!} = I - (\sum_{k=1}^\infty \frac{(-\overline{\beta}_{t|s})^k}{k!}) R_b \nonumber\\
    &= I-(e^{-\overline{\beta}_{t|s}}-1)R_b 
     = e^{-\overline{\beta}_{t|s}} I + (1-e^{-\overline{\beta}_{t|s}}) \1 \vm^\top 
\end{align}

\subsection{Derivation of \eqref{eq:vector_g}}
\label{ssec:derivation-gt}

With the help of \eqref{eq:forward_CTMC}, we first show that $g_t^\theta$ can be simplified as follows.
\begin{align}
    g^\theta_t(\vz | \vx) & = \frac{1}{\langle \vx, \vm \rangle}\Big[ 
        \langle \vz, \vm \rangle + 
        \frac{\overline{\alpha}_{t|0}}{1-\overline{\alpha}_{t|0}}p^\theta_{0|t}(\vz | \vx) -
        \frac{\overline{\alpha}_{t|0} \langle \vz, \vm \rangle}{\overline{\alpha}_{t|0} + (1- \overline{\alpha}_{t|0} )\langle \vx , \vm \rangle} p^\theta_{0|t}(\vx | \vx)
    \Big] \nonumber\\
    & = \frac{1}{\langle \vx, \vm \rangle}\Big[ 
    \big(1-\frac{\overline{\alpha}_{t|0} p^\theta_{0|t}(\vx | \vx)}{\overline{\alpha}_{t|0} + (1- \overline{\alpha}_{t|0} )\langle \vx , \vm \rangle}\big) \langle \vz, \vm \rangle + \frac{\overline{\alpha}_{t|0}}{1-\overline{\alpha}_{t|0}}p^\theta_{0|t}(\vz | \vx)
   \Big]  \quad \forall \vz \neq \vx \label{eq:g_formulation}
\end{align}

\begin{proof}
\begin{align}
    \frac{q_{t|0}(\vz | \vx_0) }{q_{t|0}(\vx | \vx_0)} = 
    \frac{\vz^\top \overline{Q}_{t|0}^\top \vx_0 }{\vx^\top \overline{Q}_{t|0}^\top \vx_0} =
    \frac{\overline{\alpha}_{t|0} \vz^\top\vx_0 + (1-\overline{\alpha}_{t|0})\langle \vz, \vm \rangle}{\overline{\alpha}_{t|0} \vx^\top\vx_0 + (1-\overline{\alpha}_{t|0})\langle \vx, \vm \rangle}
\end{align}
Then when $\vz \neq \vx$, we can write it as 
\begin{align}
    \frac{q_{t|0}(\vz | \vx_0) }{q_{t|0}(\vx | \vx_0)} = 
    \begin{cases}
         \frac{\langle \vz, \vm \rangle}{\langle \vx, \vm \rangle} 
            &\quad\text{if } \vx_0 \neq \vx
             \text{ and }\vx_0 \neq \vz\\[5pt]
         \frac{\frac{\overline{\alpha}_{t|0}}{1-\overline{\alpha}_{t|0}} +  \langle \vz, \vm \rangle}{\langle \vx, \vm \rangle} 
            &\quad\text{if } \vx_0 = \vz \\[5pt]
         \frac{\langle \vz, \vm \rangle}{\frac{\overline{\alpha}_{t|0}}{1-\overline{\alpha}_{t|0}} +  \langle \vx, \vm \rangle}
            &\quad\text{if } \vx_0 = \vx 
    \end{cases}
\end{align}
Hence,
\begin{align}
    &g^{\theta}_t(\vz| \vx) = \sum_{\vx_0} \frac{q_{t|0}(\vz | \vx_0)}{q_{t|0}(\vx | \vx_0) } p_{0|t}^\theta(\vx_0|\vx) \nonumber \\
    &= \frac{q_{t|0}(\vz | \vx)}{q_{t|0}(\vx | \vx) } p_{0|t}^\theta(\vx|\vx) \nonumber + \frac{q_{t|0}(\vz | \vz)}{q_{t|0}(\vx | \vz) } p_{0|t}^\theta(\vz|\vx)
    + \sum_{\vx_0\neq \vz, \vx} \frac{q_{t|0}(\vz | \vx_0)}{q_{t|0}(\vx | \vx_0) } p_{0|t}^\theta(\vx_0|\vx) \nonumber\\
    &= \frac{\langle \vz, \vm \rangle}{\frac{\overline{\alpha}_{t|0}}{1-\overline{\alpha}_{t|0}} +  \langle \vx, \vm \rangle} p^\theta_{0|t}(\vx | \vx) + 
    \frac{\frac{\overline{\alpha}_{t|0}}{1-\overline{\alpha}_{t|0}} +  \langle \vz, \vm \rangle}{\langle \vx, \vm \rangle} p^\theta_{0|t}(\vz | \vx) + 
     \frac{\langle \vz, \vm \rangle}{\langle \vx, \vm \rangle} (1-p^\theta_{0|t}(\vz | \vx)-p^\theta_{0|t}(\vx | \vx) ) \nonumber \\
     & = \frac{1}{\langle \vx, \vm \rangle}\Big[ 
    \big(1-\frac{\overline{\alpha}_{t|0} p^\theta_{0|t}(\vx | \vx)}{\overline{\alpha}_{t|0} + (1- \overline{\alpha}_{t|0} )\langle \vx , \vm \rangle}\big) \langle \vz, \vm \rangle + \frac{\overline{\alpha}_{t|0}}{1-\overline{\alpha}_{t|0}}p^\theta_{0|t}(\vz | \vx)
   \Big]  \quad \forall \vz \neq \vx 
\end{align}

Additionally, when $\vz = \vx$, $g_t^\theta(\vz | \vx) = \sum_{\vx_0} p^\theta_{0|t}(\vx_0 | \vx) = 1$.
\end{proof}

We can also compute the vectorization $g_t^\theta(\cdot| \vx)$ directly as 
\begin{align}
    g_t^\theta(\cdot| \vx) = \frac{1}{\langle \vx, \vm \rangle} \Big[ 
        \big(1 - \frac{\overline{\alpha}_{t|0} \langle f_t^\theta(\vx), \vx \rangle }{\overline{\alpha}_{t|0} + (1-\overline{\alpha}_{t|0})\langle \vx, \vm \rangle }\big) \vm +  
        \frac{\overline{\alpha}_{t|0}}{1-\overline{\alpha}_{t|0}} f_t^\theta(\vx)
    \Big] \odot(\1 - \vx) + \vx 
\end{align}
where $f_t^\theta(\vx)$ is the parameterized neural network for $p^\theta_{0|t}(\cdot | \vx)$ that outputs the distribution of $\rvx_{0|t}$. \eqref{eq:vector_g} combines \eqref{eq:g_formulation} and $g_t^\theta(\vx| \vx)=1, \forall \vx$.

\subsection{Derivation of Forward and Backward Transition Rate in Multi-element Case}
\label{ssec:transition-rate-multi-dim}

In this section, we show how to extend  transition rates, and and the ratio $g_t^\theta$, into multi-element case. We let $\vx^{\backslash d}$ represent $\vx^{1:D\backslash d}$, i.e. the object without $d$-th element, for simplicity.

{\bf Forward Transition Rates:} First, the transition rates for forward sampling has a specific decomposition formulation in multi-element case as proven by \cite{campbell2022continuous}, thus, we summarize the result as follows. The key assumption for CTMC is that at a single time, only one dimension can change. 
\begin{align}
    r^{1:D}_t(\vy^{1:D} | \vx^{1:D}) = \sum_{d=1}^D r^d_t(\vy^d | \vx^d) \delta_{\vx^{\backslash d}, \vy^{\backslash d}}
\end{align}

where $\delta_{\vx^{\backslash d}, \vy^{\backslash d}}$ is the Kronecker delta and it is 1 if and only if $\vx^{\backslash d} = \vy^{\backslash d}$. 
As we also assume that all dimension processes are indepedent, $r^d_t$ denotes the transition rate of the CTMC process at d-th element/dimension. 

{\bf Backward Transition Rates:} 
Now let us work on $\hat{r}_t^{1:D}(\vy^{1:D} | \vx^{1:D} )$. Notice that as the backward process is also a CTMC, it also satisfies that only one dimension can change at a time.
We summarize two equivalent formulations as follows.

{
\setlength{\abovedisplayskip}{4pt}
\setlength{\belowdisplayskip}{4pt}
\begin{align}
    \hat{r}^{1:D}_t(\vy^{1:D} | \vx^{1:D}) &= \sum_{d=1}^D  r^d_t(\vx^d | \vy^d) \delta_{\vx^{\backslash d}, \vy^{\backslash d}}
       \sum_{\vx_0^d} \frac{q_{t|0}(\vy^d | {\vx_0^d})}{q_{t|0}(\vx^d | {\vx_0^d})} q_{0|t}( \vx_0^d| \vx^{1:D}) \label{eq:rate_multi_element_1}
\end{align}
}
\begin{align}
    \hat{r}^{1:D}_t(\vy^{1:D} | \vx^{1:D}) &= 
    \sum_{d=1}^D \frac{  r^d_t(\vx^d | \vy^d) \delta_{\vx^{\backslash d}, \vy^{\backslash d}}}{
        \sum_{\vx_0^d} \frac{q_{t|0}(\vx^d | {\vx_0^d})}{q_{t|0}(\vy^d | {\vx_0^d})} q_{0|t}( \vx_0^d| \vy^{1:D}) 
    }\label{eq:rate_multi_element_2}
\end{align}

Notice that these two formulations should be equivalent. In practice, we use the first formulation to parameterize the reverse transition rate in learning. 

\begin{proof}
\begin{align}
    \frac{q_t(\vy^{1:D})}{q_t(\vx^{1:D})} 
= \sum_{\vx_0^{1:D}} \frac{q_{t|0}(\vy^{1:D}| \vx_0^{1:D}) }{q_{t|0}(\vx^{1:D}| \vx_0^{1:D})} q_{0|t}(\vx_0^{1:D} | \vx^{1:D})= \sum_{\vx_0^{1:D}} q_{0|t}(\vx_0^{1:D} | \vx^{1:D}) \prod_{d=1}^D \frac{q_{t|0}(\vy^d | \vx_0^d)}{q_{t|0}(\vx^d | \vx_0^d)}
\end{align}
\begin{align}
    \sum_{\vx_0^{\backslash d}} q_{0|t}({ \vx_0^{1:D}} | \vx^{1:D}) =  \sum_{\vx_0^{\backslash d}} q_{0|t}({ \vx_0^{d}} | \vx^{1:D}) q_{0|t}({ \vx_0^{\backslash d}} | \vx^{1:D}, \vx_0^{d}) = q_{0|t}({ \vx_0^{d}} | \vx^{1:D})
\end{align}
\textbf{(Case 1)}
\begin{align}
    \hat{r}^{1:D}_t(\vy^{1:D} | \vx^{1:D}) &= r^{1:D}_t(\vx^{1:D} | \vy^{1:D}) \frac{q_t(\vy^{1:D})}{q_t(\vx^{1:D})} \nonumber\\
    &= \Big(\sum_{d=1}^D r^d_t(\vx^d | \vy^d) \delta_{\vx^{\backslash d}, \vy^{\backslash d}}\Big)\Big(\sum_{\vx_0^{1:D}} q_{0|t}(\vx_0^{1:D} | \vx^{1:D}) \prod_{d=1}^D \frac{q_{t|0}(\vy^d | \vx_0^d)}{q_{t|0}(\vx^d | \vx_0^d)}\Big) \nonumber \\
    &= \sum_{d=1}^D \sum_{\color{orange}\vx_0^{1:D}} r^d_t(\vx^d | \vy^d) \delta_{\vx^{\backslash d}, \vy^{\backslash d}}  q_{0|t}({\color{orange} \vx_0^{1:D}} | \vx^{1:D}) \frac{q_{t|0}(\vy^d | {\color{orange}\vx_0^d})}{q_{t|0}(\vx^d | {\color{orange}\vx_0^d})}  \nonumber  \\
    &= \sum_{d=1}^D  r^d_t(\vx^d | \vy^d) \delta_{\vx^{\backslash d}, \vy^{\backslash d}}
       \sum_{\color{orange}\vx_0^d} \frac{q_{t|0}(\vy^d | {\color{orange}\vx_0^d})}{q_{t|0}(\vx^d | {\color{orange}\vx_0^d})} \sum_{\vx_0^{\backslash d}} q_{0|t}({ \vx_0^{1:D}} | \vx^{1:D}) \nonumber  \\
    &= \sum_{d=1}^D  r^d_t(\vx^d | \vy^d) \delta_{\vx^{\backslash d}, \vy^{\backslash d}}
       \sum_{\vx_0^d} \frac{q_{t|0}(\vy^d | {\vx_0^d})}{q_{t|0}(\vx^d | {\vx_0^d})} q_{0|t}( \vx_0^d| \vx^{1:D})
\end{align}
\textbf{(Case 2)}
\begin{align}
    \hat{r}^{1:D}_t(\vy^{1:D} | \vx^{1:D}) &= r^{1:D}_t(\vx^{1:D} | \vy^{1:D}) / \frac{q_t(\vx^{1:D})}{q_t(\vy^{1:D})} \nonumber\\
    &= \sum_{d=1}^D \frac{  r^d_t(\vx^d | \vy^d)  \delta_{\vx^{\backslash d}, \vy^{\backslash d}}}
    {\sum_{\vx_0^{1:D}} q_{0|t}(\vx_0^{1:D} | \vy^{1:D}) \prod_{i=1}^D \frac{q_{t|0}(\vx^i | \vx_0^i)}{q_{t|0}(\vy^i | \vx_0^i)}} \nonumber\\
    & = \sum_{d=1}^D \frac{  r^d_t(\vx^d | \vy^d) \delta_{\vx^{\backslash d}, \vy^{\backslash d}}}{\sum_{\vx_0^{1:D}} q_{0|t}(\vx_0^{1:D} | \vy^{1:D})  \frac{q_{t|0}(\vx^d | \vx_0^d)}{q_{t|0}(\vy^d | \vx_0^d)}} \nonumber \\
    & = \sum_{d=1}^D \frac{  r^d_t(\vx^d | \vy^d) \delta_{\vx^{\backslash d}, \vy^{\backslash d}}}{
        \sum_{\vx_0^d} \frac{q_{t|0}(\vx^d | {\vx_0^d})}{q_{t|0}(\vy^d | {\vx_0^d})} q_{0|t}( \vx_0^d| \vy^{1:D})
    }
\end{align}
\end{proof}

{\bf Ratio:} We now define an estimator  $g_t^{\theta,d}(\vx^d | \vy^{1:D})$ as follows. 
\begin{align}
    g_t^{\theta,d}(\vx^d | \vy^{1:D}): &= \sum_{\vx_0^d} \frac{q_{t|0}(\vx^{d}| \vx_0^{d}) }{q_{t|0}(\vy^{d}| \vx_0^{d})} p_{0|t}^\theta(\vx_0^{d} | \vy^{1:D}) \approx \sum_{\vx_0^d} \frac{q_{t|0}(\vx^{d}| \vx_0^{d}) }{q_{t|0}(\vy^{d}| \vx_0^{d})} q_{0|t}(\vx_0^{d} | \vy^{1:D}) =\frac{q_t(\vx^{d} | \vy^{\backslash d})}{q_t( \vy^d | \vy^{\backslash d})} \label{eq:gtd}\\
    g_t^{\theta}(\vx^{1:D} | \vy^{1:D}): &= \prod_{d=1}^D g^{\theta,d}_t(\vx^d|\vy^{1:D}) = \sum_{\vx_0^{1:D}} \frac{q_{t|0}(\vx^{1:D}| \vx_0^{1:D}) }{q_{t|0}(\vy^{1:D}| \vx_0^{1:D})} p_{0|t}^\theta(\vx_0^{1:D} | \vy^{1:D}) \approx \frac{q_t(\vx^{1:D})}{q_t(\vy^{1:D})}
\end{align}

We can extend the vector formulation \eqref{eq:vector_g} in Proposition 4 to $ g_t^{\theta,d}(\vx^d | \vy^{1:D})$:
\begin{align}
        g_t^{\theta,d}(\cdot| \vx^{1:D}) = \frac{1}{\langle \vx^d, \vm^d \rangle} \Big[ 
        \big(1 - \frac{\overline{\alpha}_{t|0} \langle f_t^{\theta,d}(\vx^{1:D}), \vx^d \rangle }{\overline{\alpha}_{t|0} + (1-\overline{\alpha}_{t|0})\langle \vx^d, \vm^d \rangle }\big) \vm^d +  
        \frac{\overline{\alpha}_{t|0}}{1-\overline{\alpha}_{t|0}} f_t^{\theta,d}(\vx^{1:D})
    \Big] \odot(\1 - \vx^d) + \vx^d \label{eq:vector_g_multidim}
\end{align}

Then, we can derive two approximators for transition rate $\hat{r}^{1:D}_t(\vy^{1:D} | \vx^{1:D})$ as follows.
\begin{align}
    [\hat{r}^{\theta,1:D}_t]^{1}(\vy^{1:D} | \vx^{1:D}) &:= \sum_{d=1}^D    r^d_t(\vx^d | \vy^d) g_t^{\theta,d}(\vy^d | \vx^{1:D})  \cdot \delta_{\vx^{\backslash d}, \vy^{\backslash d}} \approx \text{\eqref{eq:rate_multi_element_1}}\\
    [\hat{r}^{\theta,1:D}_t]^{2}(\vy^{1:D} | \vx^{1:D}) &:= \sum_{d=1}^D   \frac{ r^d_t(\vx^d | \vy^d)}{ g_t^{\theta,d}(\vx^d | \vy^{1:D})}  \cdot \delta_{\vx^{\backslash d}, \vy^{\backslash d}}  \approx \text{\eqref{eq:rate_multi_element_2}}
\end{align}

\subsection{Derivation of Negative VLB Loss in Multi-element Case}
\label{ssec:cont-vlb-multi-dim}

As forward process is defined in \eqref{eq:forward_CTMC}, in multi-element case we can easily get 
\begin{align}
    r^d_t(\vx^d|\cdot) &= R^d_t \vx^d = \beta(t) \big(\langle \vx^d, \vm^d \rangle \1 - \vx^d \big)\\
    r^d_t(\cdot|\vx^d) &= (R^d_t)^\top \vx^d =\beta(t)\big( \vm^d - \vx^d \big)
\end{align}
The $r^d_t(\vx | \cdot)$ and $r^d_t(\cdot| \vx)$ are essentially the $\vx$-th column and row of the transition rate matrix $R^d_t$.

In Mult-element case, the negative VLB loss in \eqref{eq:CTMC_VLB_0} can be written as 
\begin{align}
T \ \E_{
    \substack{t\sim\text{Uni}(0,T) \\ \vx^{1:D} \sim q_{t|0} }} 
    \Big[ \sum_{\vz^{1:D} \neq \vx^{1:D}} \hat{r}_t^{\theta,1:D}(\vz^{1:D} | \vx^{1:D}) - \sum_{\vz^{1:D} \neq \vx^{1:D}}r^{1:D}_t(\vz^{1:D}|\vx^{1:D})\log \hat{r}_t^{\theta,1:D}(\vx^{1:D} | \vz^{1:D}) \Big]  
\end{align}
As there are two terms, let's work on each term separately. 

\subsubsection{Term 1}
Based on the formulation of $r_t^{1:D}$ and $\hat{r}_t^{\theta, 1:D}$, we can rewrite the first term as 
\begin{align}
\text{Term1} &= T\ \E_{t,\vx^{1:D}} \sum_{\vz^{1:D} \neq \vx^{1:D}} \hat{r}_t^{\theta,1:D}(\vz^{1:D} | \vx^{1:D}) \nonumber\\
&= T\ \E_{t,\vx^{1:D}} \sum_{\vz^{1:D} \neq \vx^{1:D}} \sum_{d=1}^D r^d_t(\vx^d | \vz^d) g_t^{\theta,d}(\vz^d | \vx^{1:D})  \cdot \delta_{\vx^{\backslash d}, \vz^{\backslash d}} \nonumber\\
&= T\ \E_{t,\vx^{1:D}}\Big[
        \sum_{d=1}^D \sum_{\vz^d \neq \vx^d } r^d_t(\vx^d | \vz^d) g_t^{\theta,d}(\vz^d | \vx^{1:D}) \Big] \nonumber\\
&= T\ \E_{t,\vx^{1:D}}\Big[ \sum_{d=1}^D r^d_t(\vx^d|\cdot)^\top  g_t^{\theta,d}(\cdot| \vx^{1:D}) \Big] + \text{const.}  \nonumber\\
&= T \ \E_{\substack{t \sim \text{Uni.}(0,T) \\ \vx^{1:D} \sim q_{t|0}}} \ \beta(t) \sum_{d=1}^D  \langle \vx^d, \vm^d \rangle  \Big[
        \1^\top g_t^{\theta,d}(\cdot| \vx^{1:D}) \Big] + \text{const.}
\end{align}

\subsubsection{Term 2}
As the evaluation of $\hat{r}^\theta_t(\cdot | \vx)$ for any $\vx$ requires a single forward pass of the parameterized network $p_{0|t}^\theta(\cdot|\vx)$, the second term within the expectation  of \eqref{eq:CTMC_VLB_0} requires multiple passes of the model. This complexity is even greatly amplified in cases with multi-element objects.  \citet{campbell2022continuous} avoids the multiple passes by changing the expectation variable through importance sampling. We take a similar approach to simplify the second term. Differently, \citet{campbell2022continuous} uses a specific sampling distribution (same as the forward transition rate) to introduce the auxiliary variable for changing the expectation variable, we generalize it to use a general sampling process $S_t$ defined below. 

 Let $\vy^{1:D}$ be the new variable upon which the exchanged expectation is based, and assume that $\vy^{1:D} $ is sampled from an unnormalized joint distribution $S_t(\vy^{1:D} | \vx^{1:D})$. We restrict $S_t(\vy^{1:D} | \vx^{1:D})$ to be a unnormalized probability that is nonzero if and only if $\vy^{1:D} $ and $ \vx^{1:D}$ are different at a single element. 
Formally, we can write the unnormalized distribution as 
\begin{align}
    S_t(\vy^{1:D} | \vx^{1:D}) &= (1- \delta_{\vy^{1:D},\vx^{1:D}}) \sum_{d=1}^D S^d_t(\vy^d | \vx^d) \delta_{\vy^{\backslash d}, \vx^{\backslash d} } \nonumber \\
 \text{with normalizer } \gS_t(\vx^{1:D}) &= \sum_{\vy^{1:D}} S_t(\vy^{1:D} | \vx^{1:D}) = \sum_{d=1}^D \sum_{\vy^d \neq \vx^d} S_t^d(\vy^d | \vx^d) \label{eq:St}
\end{align}
where $S_t^d (\vy^d | \vx^d)$ is any unnormalized probability at dimension $d$.  

Now for the second term, we have 
\begin{align}
    \text{Term2} &= T \ \E_{t, \vx^{1:D}} \sum_{\vz^{1:D} \neq \vx^{1:D} } r_t(\vz^{1:D} | \vx^{1:D}) \log \hat{r}_t^\theta (\vx^{1:D} | \vz^{1:D} ) \nonumber\\
    &= T \ \E_{t, \vx^{1:D}} \sum_{\vz^{1:D} \neq \vx^{1:D} } 
    \frac{S_t(\vz^{1:D}|\vx^{1:D})}{\gS(\vx^{1:D})} \cdot \frac{ \gS_t(\vx^{1:D})}{S_t(\vz^{1:D}|\vx^{1:D})} \cdot r_t(\vz^{1:D} | \vx^{1:D}) \log \hat{r}_t^\theta (\vx^{1:D} | \vz^{1:D} )  \nonumber\\
    &= T \ \E_{\substack{t, \vx^{1:D}\\\vz^{1:D}\sim S_t }} \Big [ 
    \frac{ \gS_t(\vx^{1:D})}{S_t(\vz^{1:D}|\vx^{1:D})} \cdot r_t(\vz^{1:D} | \vx^{1:D}) \log \hat{r}_t^\theta (\vx^{1:D} | \vz^{1:D} ) \Big] \qquad \text{ (As $S_t$ samples $\vz^{1:D} \neq \vx^{1:D}$ )} \nonumber\\
    &= T \ \E_{\substack{t, \vz^{1:D}\sim S_t }} 
    \sum_{\vx^{1:D}}
    \Big[
        q_{S_t}(\vx^{1:D} | \vx_0^{1:D}, \vz^{1:D} ) \cdot
    \frac{ \gS_t(\vx^{1:D})}{S_t(\vz^{1:D}|\vx^{1:D})} \cdot r_t(\vz^{1:D} | \vx^{1:D}) \log \hat{r}_t^\theta (\vx^{1:D} | \vz^{1:D} ) 
    \Big]
\end{align}
where $ q_{S_t}(\vx^{1:D} | \vx_0^{1:D}, \vz^{1:D} )$ is the conditional posterior distribution such that (for clearity we replace $\vx^{1:D}, \vz^{1:D}$ with $\vx^{1:D}_t, \vz^{1:D}_t$ respectively, as they are variables at time $t$) 
\begin{align}
  &q_{S_t}(\vx_t^{1:D} | \vx_0^{1:D}, \vz_t^{1:D} ) \nonumber\\
  & = 
  \frac{q_{S_t}(\vx_t^{1:D}, \vz_t^{1:D} | \vx_0^{1:D} )}{\sum_{\vy_t^{1:D}} q_{S_t}(\vy_t^{1:D}, \vz_t^{1:D} | \vx_0^{1:D} ) } 
  = \frac{q_{t|0}(\vx^{1:D}_t | \vx^{1:D}_0)  \cdot S_t(\vz^{1:D}_t | \vx^{1:D}_t) / \gS_t(\vx^{1:D}_t) }{
  \sum_{\vy^{1:D}_y} q_{t|0}(\vx^{1:D} | \vx^{1:D}_0)  \cdot S_t(\vz^{1:D}_t | \vy^{1:D}_t) / \gS_t(\vy^{1:D}_t) 
  } \nonumber\\
  &=\frac{
    (1-\delta_{\vz_t, \vx_t}) \sum_{d=1}^D\delta_{\vz_t^{\backslash d}, \vx_t^{\backslash d}} S_t^d(\vz_t^d|\vx_t^d ) / \gS_t(\vx_t^{1:D})
    \cdot q_{t|0}(\vx^d_t \circ \vz_t^{\backslash d} |\vx_0^{1:D})
  }
  {
    \sum_{\vy^{1:D}_t} \big[ 
    (1-\delta_{\vz_t, \vy_t}) \sum_{d=1}^D\delta_{\vz_t^{\backslash d}, \vy_t^{\backslash d}} S_t^d(\vz_t^d|\vy_t^d ) / \gS_t(\vy_t^{1:D} )
    \cdot q_{t|0}(\vy_t^d \circ \vz_t^{\backslash d} |\vx_0^{1:D})
    \big]  
  } \nonumber\\
  &=\frac{
    (1-\delta_{\vz_t, \vx_t}) \sum_{d=1}^D\delta_{\vz_t^{\backslash d}, \vx_t^{\backslash d}} 
    \frac{S_t^d(\vz_t^d|\vx_t^d )} {\gS_t(\vx_t^{1:D})}
    \cdot \frac{q_{t|0}(\vx^d_t |\vx_0^d)}{q_{t|0}(\vz^d_t |\vx_0^d)}
  }
  {
    \sum_{\vy_t^{1:D}} \big[ 
    (1-\delta_{\vz_t, \vy_t}) \sum_{d=1}^D\delta_{\vz_t^{\backslash d}, \vy_t^{\backslash d}} 
    \frac{S_t^d(\vz_t^d|\vy_t^d )} {\gS_t(\vy_t^{1:D})}
    \cdot \frac{q_{t|0}(\vy_t^d |\vx_0^d)}{q_{t|0}(\vz^d_t |\vx_0^d)}
    \big]  
  }\nonumber\\
  &=\frac{
    (1-\delta_{\vz_t, \vx_t}) \sum_{d=1}^D\delta_{\vz_t^{\backslash d}, \vx_t^{\backslash d}} 
    \frac{S_t^d(\vz_t^d|\vx_t^d )} {\gS_t(\vx_t^{1:D})}
    \cdot \frac{q_{t|0}(\vx^d_t |\vx_0^d)}{q_{t|0}(\vz^d_t |\vx_0^d)}
  }{
  \sum_{d=1}^D \sum_{\vy_t^d \neq \vz^d}  \frac{S_t^d(\vz_t^d|\vy_t^d )} {\gS_t(\vy_t^d \circ \vz_t^{\backslash d} )}
    \cdot \frac{q_{t|0}(\vy_t^d |\vx_0^d)}{q_{t|0}(\vz^d_t |\vx_0^d)}
  }
\end{align}

Now taking the formulation of $q_{S_t}$ back into Term 2, we further simplify Term 2 as 
\begin{align}
    &\text{Term2} \nonumber\\
    &= T \ \E_{\substack{t, \vz^{1:D}\sim S_t }} 
    \frac{\sum_{\vx^{1:D}}\Big[
  (1-\delta_{\vz, \vx}) \sum_{d=1}^D\delta_{\vz^{\backslash d}, \vx^{\backslash d}} 
    \frac{S_t^d(\vz^d|\vx^d )} {S_t(\vz^{1:D}|\vx^{1:D})}
    \cdot \frac{q_{t|0}(\vx^d |\vx_0^d)}{q_{t|0}(\vz^d |\vx_0^d)}
     \cdot r_t(\vz^{1:D} | \vx^{1:D}) \log \hat{r}_t^\theta (\vx^{1:D} | \vz^{1:D} ) 
        \Big]
    }{
      \sum_{d=1}^D \sum_{\vy^d \neq \vz^d}  \frac{S_t^d(\vz^d|\vy^d )} {\gS_t(\vy^d \circ \vz^{\backslash d})}
        \cdot \frac{q_{t|0}(\vy^d |\vx_0^d)}{q_{t|0}(\vz^d|\vx_0^d)}
    }\nonumber\\
&=\E_{\substack{t, \vz^{1:D}\sim S_t }} \frac{
\sum_{d=1}^D \sum_{\vx^d \neq \vz^d} \frac{S_t^d(\vz^d|\vx^d )} {S_t(\vz^{1:D}|\vx^d \circ \vz^{\backslash d})}
    \cdot \frac{q_{t|0}(\vx^d |\vx_0^d)}{q_{t|0}(\vz^d |\vx_0^d)}
    \cdot r_t(\vz^{1:D} | \vx^d \circ \vz^{\backslash d}) \log \hat{r}_t^\theta (\vx^d \circ \vz^{\backslash d} | \vz^{1:D} ) 
}
{
      \sum_{d=1}^D \sum_{\vy^d \neq \vz^d}  \frac{S_t^d(\vz^d|\vy^d )} {\gS_t(\vy^d \circ \vz^{\backslash d})}
        \cdot \frac{q_{t|0}(\vy^d |\vx_0^d)}{q_{t|0}(\vz^d|\vx_0^d)}
    }\nonumber\\
&=\E_{\substack{t, \vz^{1:D}\sim S_t }} 
\frac{ 
\sum_{d=1}^D \sum_{\vx^d \neq \vz^d}  
 \frac{q_{t|0}(\vx^d |\vx_0^d)}{q_{t|0}(\vz^d|\vx_0^d)} \cdot 
 r_t^d(\vz^d | \vx^d) \cdot
 \log  r^d_t(\vz^d | \vx^d) g_t^{\theta,d}(\vx^d | \vz^{1:D})
 }
{
 \sum_{d=1}^D \sum_{\vy^d \neq \vz^d}  
 \frac{q_{t|0}(\vy^d |\vx_0^d)}{q_{t|0}(\vz^d|\vx_0^d)} \cdot \frac{S_t^d(\vz^d|\vy^d )} {\gS_t(\vy^d \circ \vz^{\backslash d})}
} \nonumber\\
&=\E_{\substack{t, \vz^{1:D}\sim S_t }} 
\frac{ 
\sum_{d=1}^D \sum_{\vx^d \neq \vz^d}  
 \frac{q_{t|0}(\vx^d |\vx_0^d)}{q_{t|0}(\vz^d|\vx_0^d)} \cdot 
 r_t^d(\vz^d | \vx^d) \cdot
 \log g_t^{\theta,d}(\vx^d | \vz^{1:D})
 }
{
 \sum_{d=1}^D \sum_{\vy^d \neq \vz^d}  
 \frac{q_{t|0}(\vy^d |\vx_0^d)}{q_{t|0}(\vz^d|\vx_0^d)} \cdot \frac{S_t^d(\vz^d|\vy^d )} {\gS_t(\vy^d \circ \vz^{\backslash d})}
} + \text{const.}
\end{align}
The above equation further shows that the sampling distribution $S_t$ for adding exchanging variable $\vz^{1:D}$  only affect a weighting term of the loss computation. Let us define the scalar weighting term as 
\begin{align}
    \gM_{S_t}(\vz^{1:D}|\vx^{1:D}_0):= \sum_{d=1}^D \sum_{\vy^d \neq \vz^d}  
 \frac{q_{t|0}(\vy^d |\vx_0^d)}{q_{t|0}(\vz^d|\vx_0^d)} \cdot \frac{S_t^d(\vz^d|\vy^d )} {\gS_t(\vy^d \circ \vz^{\backslash d})}
 \label{eq:M_St}
\end{align}

With this definition, the Term 2 can be further rewrited as 
\begin{align}
\text{Term2} &=  
\E_{\substack{t, \vz^{1:D}\sim S_t }} \Big[
\frac{1}{\gM_{S_t}(\vz^{1:D}|\vx^{1:D}_0)}
\sum_{d=1}^D \frac{1}{q_{t|0}(\vz^d|\vx_0^d)} \sum_{\vx^d \neq \vz^d} 
q_{t|0}(\vx^d |\vx_0^d) \cdot 
 r_t^d(\vz^d | \vx^d) \cdot
 \log g_t^{\theta,d}(\vx^d | \vz^{1:D}) 
 \Big]\nonumber\\
 &=
 \E_{\substack{t, \vz^{1:D}\sim S_t }} \Big[
 \frac{1}{\gM_{S_t}(\vz^{1:D}|\vx^{1:D}_0)}
\sum_{d=1}^D \frac{\1^\top
[q_{t|0}(\cdot |\vx_0^d) \odot 
 r_t^d(\vz^d | \cdot) \odot
 \log g_t^{\theta,d}(\cdot | \vz^{1:D})]
}{q_{t|0}(\vz^d|\vx_0^d)} 
 \Big]+ \text{const.}\nonumber\\
&= 
\E_{\substack{t, \vz^{1:D}\sim S_t }} \Big[
\frac{\beta(t)}{\gM_{S_t}(\vz^{1:D}|\vx^{1:D}_0)}
\sum_{d=1}^D \frac{\langle \vz^d, \vm^d \rangle 
}{q_{t|0}(\vz^d|\vx_0^d)}
\1^\top
[q_{t|0}(\cdot |\vx_0^d) \odot 
 \log g_t^{\theta,d}(\cdot | \vz^{1:D})] + \text{const.} 
 \Big]
\end{align}

\subsubsection{All Terms}
Combine term 1 and term 2 together, we can write the negative VLB loss as 
\begin{align}
&T \ \E_{\substack{t\sim\text{Uni}(0,T) \\ \vx^{1:D} \sim q_{t|0} }} 
    \Big[ \sum_{\vz^{1:D} \neq \vx^{1:D}} \hat{r}_t^{\theta,1:D}(\vz^{1:D} | \vx^{1:D}) - \sum_{\vz^{1:D} \neq \vx^{1:D}}r^{1:D}_t(\vz^{1:D}|\vx^{1:D})\log \hat{r}_t^{\theta,1:D}(\vx^{1:D} | \vz^{1:D}) \Big]   \nonumber\\
&= T\ \E_{\substack{t\sim\text{Uni}(0,T) 
\nonumber
\\ \vx^{1:D} \sim q_{t|0}(\vx_0^{1:D}) 
\nonumber\\\vz^{1:D} \sim S_t(\vx^{1:D})}} 
\Big[ \sum_{d=1}^D r^d_t(\vx^d|\cdot)^\top  g_t^{\theta,d}(\cdot| \vx^{1:D}) 
\nonumber\\
&-
 \frac{1}{\gM_{S_t}(\vz^{1:D}|\vx^{1:D}_0)}
\sum_{d=1}^D \frac{\1^\top
[q_{t|0}(\cdot |\vx_0^d) \odot 
 r_t^d(\vz^d | \cdot) \odot
 \log g_t^{\theta,d}(\cdot | \vz^{1:D})]
}{q_{t|0}(\vz^d|\vx_0^d)} 
    \Big] + \text{const.} 
\end{align}

\subsection{Further Simplification of Continuous VLB} \label{apdx:proof_sedd}

\begin{proposition}
The loss in \eqref{eq:ctmc_vlb_final} can be further simplified as 
\begin{equation}
\scalebox{1}{$
\begin{aligned}
    T\ \E_{\substack{t\sim\text{Uni}(0,T) \\ \vx^{1:D} 
    \sim q_{t|0}(\cdot|\vx_0^{1:D}) }} 
\Big[ 
\sum_{d=1}^D r^d_t(\vx^d|\cdot)^\top 
\Big( g_t^{\theta,d}(\cdot| \vx^{1:D}) 
- 
\frac{
q_{t|0}(\cdot |\vx_0^d) \odot 
\log g_t^{\theta,d}(\cdot | \vx^{1:D})
}{q_{t|0}(\vx^d|\vx_0^d)} \Big) 
\Big] 
\end{aligned}
$}
\end{equation}
where the loss only requires a single forward pass of the model and no auxiliary variable is needed. 
\end{proposition}

\begin{proof}
We use $q_{t|0}(\vx^{1:D} | \vx_0^{1:D})$ to denote $P(\rvx^{1:D}_t = \vx^{1:D}| \vx_0^{1:D})$. Similarly, let $\tilde{q}_{t|0}(\vx^{1:D} | \vx_0^{1:D})$ denotes $P(\rvz^{1:D}_t = \vx^{1:D}| \vx_0^{1:D})$. Based on the definition of $S_t$ in \eqref{eq:St} and the procedure of sampling $\rvz^{1:D}_t$ conditional on $\rvx^{1:D}_t$ described in Proposition \ref{prop:ctmc_vlb_final}, we can compute the conditional probability $\tilde{q}_{t|0}(\vx^{1:D} | \vx_0^{1:D})$ as follows:
\begin{align}
    \tilde{q}_{t|0}(\vx^{1:D} | \vx_0^{1:D})&= \sum_{\vx_t^{1:D}} P(\rvz^{1:D}_t = \vx^{1:D}, \vx_t^{1:D}| \vx_0^{1:D})= \sum_{\vx_t^{1:D}} \frac{S_t( \vx^{1:D}|\vx_t^{1:D} )} {S_t(\vx_t^{1:D})} q_{t|0}(\vx_t^{1:D}|\vx_0^{1:D}) \nonumber \\ 
    &= \sum_{\vx_t^{1:D}}(1-\delta_{\vx^{1:D},\vx_t^{1:D}}) 
    \Big( \sum_{d=1}^D \delta_{\vx^{\backslash d}, \vx_t^{\backslash d}} \cdot  \frac{S_t^d(\vx^d|\vx_t^d)}{S_t(\vx^{\backslash d} \circ \vx_t^d  )}q_{t|0}(\vx^{\backslash d} \circ \vx_t^d  |\vx_0^{1:D}) \Big) \nonumber \\
    &=q_{t|0}(\vx^{1:D} | \vx_0^{1:D}) \sum_{\vx_t^{1:D}}(1-\delta_{\vx^{1:D},\vx_t^{1:D}}) 
    \Big( \sum_{d=1}^D \delta_{\vx^{\backslash d}, \vx_t^{\backslash d}} \cdot  \frac{S_t^d(\vx^d|\vx_t^d)}{S_t(\vx^{\backslash d} \circ \vx_t^d  )}  \frac{q_{t|0}(\vx_t^d | \vx_0^d)}{q_{t|0}(\vx^d | \vx_0^d)}  \Big ) \nonumber \\
    & = q_{t|0}(\vx^{1:D} | \vx_0^{1:D}) \ \gM_{S_t}(\vx^{1:D}|\vx^{1:D}_0) \text{ (Based on $\gM_{S_t}$ in \eqref{eq:M_St}).} \label{eq:specific_mt_for_single_pass}
\end{align}

The original loss in \eqref{eq:ctmc_vlb_final} can be written as
\begin{align}
    \eqref{eq:ctmc_vlb_final}&= T\ \E_{\substack{t\sim\text{Uni}(0,T) \\ \vx^{1:D} \sim q_{t|0}(\cdot|\vx_0^{1:D}) }} 
\Big[ 
\sum_{d=1}^D r^d_t(\vx^d|\cdot)^\top  g_t^{\theta,d}(\cdot| \vx^{1:D}) \Big]- \nonumber\\
&T\ \E_{\substack{t\sim\text{Uni}(0,T) \\ \vx^{1:D} \sim q_{t|0}(\cdot|\vx_0^{1:D}) 
\\\vz^{1:D} \sim S_t(\cdot | \vx^{1:D})}} \Big[\frac{1}{\gM_{S_t}(\vz^{1:D}|\vx^{1:D}_0)}
\sum_{d=1}^D \frac{\1^\top
[q_{t|0}(\cdot |\vx_0^d) \odot 
r_t^d(\vz^d | \cdot) \odot
\log g_t^{\theta,d}(\cdot | \vz^{1:D})]
}{q_{t|0}(\vz^d|\vx_0^d)} 
\Big]  \nonumber 
\end{align}

Notice that the second part of the loss can rewritten to 

\begin{align}
&\E_{\substack{t\sim\text{Uni}(0,T) \\ \vz^{1:D} \sim \tilde{q}_{t|0}(\cdot|\vx_0^{1:D}) }} \Big[\frac{1}{\gM_{S_t}(\vz^{1:D}|\vx^{1:D}_0)}
\sum_{d=1}^D \frac{\1^\top
[q_{t|0}(\cdot |\vx_0^d) \odot 
r_t^d(\vz^d | \cdot) \odot
\log g_t^{\theta,d}(\cdot | \vz^{1:D})]
}{q_{t|0}(\vz^d|\vx_0^d)} 
\Big]  \nonumber \\
=& \E_t \sum_{\vz^{1:D}}
 \Big[\frac{\tilde{q}_{t|0}(\vz^{1:D}|\vx_0^{1:D})}{\gM_{S_t}(\vz^{1:D}|\vx^{1:D}_0)}
\sum_{d=1}^D \frac{\1^\top
[q_{t|0}(\cdot |\vx_0^d) \odot 
r_t^d(\vz^d | \cdot) \odot
\log g_t^{\theta,d}(\cdot | \vz^{1:D})]
}{q_{t|0}(\vz^d|\vx_0^d)} 
\Big] \text{(now take \eqref{eq:specific_mt_for_single_pass} into)} \nonumber \\
=&\E_t \sum_{\vz^{1:D}}
 \Big[q_{t|0}(\vz^{1:D}|\vx_0^{1:D})
\sum_{d=1}^D \frac{\1^\top
[q_{t|0}(\cdot |\vx_0^d) \odot 
r_t^d(\vz^d | \cdot) \odot
\log g_t^{\theta,d}(\cdot | \vz^{1:D})]
}{q_{t|0}(\vz^d|\vx_0^d)} 
\Big] \nonumber \\
=& \E_{\substack{t\sim\text{Uni}(0,T) \\ \vz^{1:D} \sim q_{t|0}(\cdot|\vx_0^{1:D}) }} 
\Big[ \sum_{d=1}^D \frac{\1^\top
[q_{t|0}(\cdot |\vx_0^d) \odot 
r_t^d(\vz^d | \cdot) \odot
\log g_t^{\theta,d}(\cdot | \vz^{1:D})]
}{q_{t|0}(\vz^d|\vx_0^d)} 
\Big] \label{eq:rewrite_second_part_loss}
\end{align}

Taking \eqref{eq:rewrite_second_part_loss} back into \eqref{eq:ctmc_vlb_final} we can get the following further simplified loss that only have a single forward pass of the model:
\begin{align}
     &T\ \E_{\substack{t\sim\text{Uni}(0,T) \\ \vx^{1:D} \sim q_{t|0}(\cdot|\vx_0^{1:D}) }} 
\Big[ 
\sum_{d=1}^D r^d_t(\vx^d|\cdot)^\top  g_t^{\theta,d}(\cdot| \vx^{1:D}) - \sum_{d=1}^D \frac{\1^\top
[q_{t|0}(\cdot |\vx_0^d) \odot 
r_t^d(\vx^d | \cdot) \odot
\log g_t^{\theta,d}(\cdot | \vx^{1:D})]
}{q_{t|0}(\vx^d|\vx_0^d)} 
\Big] \nonumber\\
&=T\ \E_{\substack{t\sim\text{Uni}(0,T) \\ \vx^{1:D} 
    \sim q_{t|0}(\cdot|\vx_0^{1:D}) }} 
\Big[ 
\sum_{d=1}^D r^d_t(\vx^d|\cdot)^\top 
\Big( g_t^{\theta,d}(\cdot| \vx^{1:D}) 
- 
\frac{
q_{t|0}(\cdot |\vx_0^d) \odot 
\log g_t^{\theta,d}(\cdot | \vx^{1:D})
}{q_{t|0}(\vx^d|\vx_0^d)} \Big) 
\Big] 
\end{align}
\end{proof}


\subsection{Algorithm for Unified Training and Generation}\label{ssec:alg-unify}
Alg. \ref{alg:overall} and Alg. \ref{alg:generation} demonstrate our unified training and generation framework.

\begin{algorithm}
\caption{\method Unified Training: \textcolor{purple}{Red: discrete-time step}, and \textcolor{blue}{blue: continuous-time step}.}
\label{alg:overall}
\begin{algorithmic}[1]
    \STATE {{\bfseries Input:}}  A stationary distribution $\vm$, samples  $\sim p_\text{data}$, weight factor $\lambda$, \textcolor{purple}{max time $T$}
    \STATE \textbf{repeat}
        \STATE \quad Draw $\rvx_0 \sim p_\text{data}(\rvx_0)$
        \STATE \quad Draw \textcolor{purple}{$t \sim \textbf{Uniform}({0,...,T})$} or \textcolor{blue}{$t \sim \textbf{Uniform}(0,1)$}
        \STATE \quad Compute $\overline{\alpha}_{t}$ from Table \ref{tab:noise_schedule}
        \STATE \quad Draw $\rvm_0^{1:D} \sim \text{Cat}(\rvm_0^{1:D};\vm)$
        \STATE \quad 
        $\rvx^{1:D}_{t|0}  = \delta^{1:D}_{1, \rvb^{1:D}_t} \odot \rvx^{1:D}_0 + (1-  \delta^{1:D}_{1, \rvb^{1:D}_t})\odot\rvm^{1:D}_0 $
    $\text{ where } \rvb_t \sim \text{Bernoulli} (\overline{\alpha}_t)$, from \eqref{eq:forward_sampling}
        \STATE \quad Compute $f_t^\theta(\vx_t^{1:D})$ as parameterization of $p_{\theta}(\rvx_{0}^{1:D} | \rvx_{t}^{1:D})$ 
        \STATE \quad Take gradient descent step on
     \textcolor{purple}{ $\nabla_{\theta}\{\Ls_t(\theta)+ \lambda \Ls^{CE}_t(\theta)\}$ (from \eqref{eq:vlb-discrete} $ + $  \eqref{eq:lt_approx})} 
     \STATE \quad \quad \quad \quad \quad \quad \quad \quad \quad \quad or \textcolor{blue}{ $\nabla_{\theta}\{\mathcal{L}^{CTMC}_{t}(\theta)+ \lambda \Ls^{CE}_t(\theta)\}$ (from \eqref{eq:CTMC_VLB_0} $ + $  \eqref{eq:lt_approx})} 
\STATE \textbf{until convergence}
\end{algorithmic}
\end{algorithm}
\begin{algorithm}
\caption{\method Unified Sampling}
\label{alg:generation}
\begin{algorithmic}[1]
    \STATE {{\bfseries Input:}}  A stationary distribution $\vm$, $\{t_i\}_{i=0}^n$ s.t. 
    $0=t_0 < t_1 < ... < t_n=T$, learned $f^{\theta}_{t_i}$. 
    \STATE {\bfseries MCMC Input:} use$\_$MCMC, step size $\triangle n$, total steps $N$.
    \STATE Set $\rvx_{n}^{1:D} \sim
    \text{Cat} (\rvx_{n}^{1:D}; \vm)$
    \STATE \textbf{for} $i \in \{n,...,0\}$ \textbf{do}
        \STATE \quad Compute $f_{t_i}^\theta(\vx_{i}^{1:D})$ and $p_\theta(\rvx_{i-1}^{1:D} | \rvx_{i}^{1:D})$ from \eqref{eq:s|t}
        \STATE \quad Draw $\rvx_{i-1}^{1:D} \sim \text{Cat}(\rvx_{i-1}^{1:D} ;p_\theta(\rvx_{i-1}^{1:D} | \rvx_{i}^{1:D})) $
        \STATE \quad \textbf{If} use$\_$MCMC: 
        \STATE \quad \quad $\rvx_{i-1}^{1:D} \leftarrow \text{MCMC$\_$Corrector}(t_i, \rvx_{i-1}^{1:D},  \triangle n, f_{t_i}^\theta, N)$ (Algo. \ref{alg:cap})
\STATE \textbf{return $\rvx_0$}
\end{algorithmic}
\end{algorithm}

\subsection{Continuous-Time MCMC Sampling Corrector \& Noise Scheduling}\label{ssec:cont_other}

\citet{campbell2022continuous} show that the MCMC for discrete data can be done by a predictor
step to simulate $\widehat{r}^\theta_t$
and a corrector step using $r_t + \widehat{r}^\theta_t$. We extend this result with improved derivation and show that,  thanks to the shared parameterized form, \textit{both} discrete- and continuous-time discrete diffusion can leverage the same transition probability calculation (see \eqref{eq:corrector_CTMC_prob}), leading to a shared MCMC scheme, with detailed derivation provided in  Appx. \S\ref{ssec:cont_other}.

\subsubsection{The MCMC Sampling Corrector}
\cite{song2021scorebased} introduced a predictor-corrector step to further improve the quality of generated data based on score-based Markov Chain Monte Carlo (MCMC) for continuous-time diffusion over continuous distribution. \cite{campbell2022continuous} showed that there is a similar MCMC based corrector that can be used for CTMC to improve reverse sampling at any time $t$. Although we use different reverse sampling than \cite{campbell2022continuous}, the similar corrector step can also be developed to improve the quality of reverse sampling introdued in \S \ref{ssec:CTMC_reverse_sampling}. In this section, we derive the corrector formally and simplify it based the multi-element formulations summarized in \eqref{eq:corrector_CTMC_prob}. 

Formally, at any time $t$, \cite{campbell2022continuous} proved that a \textit{time-homogeneous} CTMC with transition rate being $c_t:= r_t + \hat{r}_t$ has its stationary distribution being $q_t(\rvx^{1:D}_t)$. To avoid ambiguity, we use $n \in [0, +\infty)$ as the time variable for that CTMC with stationary distribution $q_t(\rvx^{1:D}_t)$. Then for any sample $\vz_t^{1:D}$ generated from reverse sampling process at time $t$, we can push it closer to the target marginal distribution $q_t(\rvx^{1:D}_t)$ by sampling from the corrector CTMC with initial value being $\vz_t^{1:D}$, named as $\vz_{t,n=0}^{1:D}$. Let $N$ be the maximum time allocated in the corrector CTMC, then after the corrector step $\vz_{t,n=N}^{1:D}$ is used to replace the original $\vz_t^{1:D}$. 

We now introduce how to sampling from the CTMC. Let $\triangle n$ be the time incremental for each sampling step of the corrector CTMC. Solving the Kolmogorov forward equation of this time-homogeneous CTMC can derive the transition probability at any time $n$ as 
\begin{align}
   \forall n \text{ and } \triangle n, \   p_{n+\triangle n|n}( \vy^{1:D} | \vx^{1:D}) =  \exp(\triangle n\cdot C^{1:D}_t)[\vx^{1:D}, \vy^{1:D}] 
\end{align}
Where $C_t$ is the transition rate matrix of the corrector CTMC, and 
$\exp(\triangle n\cdot C^{1:D}_t)$ is the transition probability matrix at time $n$. Notice that this matrix exponential does not have analytical formulation. Instead, we propose to control $\triangle n$ to be small enough such that 
\begin{align}
    \exp(\triangle n\cdot C^{1:D}_t) \approx I + \triangle n \cdot C^{1:D}_t 
\end{align}
Then taking it back we can obtain 
\begin{align}
     p_{n+\triangle n|n}( \vy^{1:D} &| \vx^{1:D}) \approx \delta_{\vx^{1:D}, \vy^{1:D}} + \triangle n \cdot c^{1:D}_t(\vy^{1:D}|\vx^{1:D})
     \nonumber\\
     &=  \delta_{\vx^{1:D}, \vy^{1:D}} + \triangle n \sum_{d=1}^D \Big[ r_t^d(\vy^d | \vx^d) + r_t^d(\vx^d| \vy^d)\cdot g_t^d(\vy^d|\vx^{1:D})
     \Big] \delta_{\vx^{\backslash d}, \vy^{\backslash d}} 
\end{align}

Instead of sampling all elements jointly, we propose to sample each element of the object independently from their individual marginal distribution, which can be analytically formulated as   
\begin{align}
    p_{n+\triangle n|n}( \vy^d | \vx^{1:D}) &= \sum_{\vy^{\backslash d}} p_{n+\triangle n|n}( \vy^{1:D} | \vx^{1:D}) \nonumber\\
    &= \triangle n \Big[ r_t^d(\vy^d | \vx^d) + r_t^d(\vx^d| \vy^d)\cdot g_t^d(\vy^d|\vx^{1:D})
     \Big] \text{ if } \vy^d \neq \vx^d \nonumber \\
    &=  \triangle n (\vy^d)^\top \Big[ r_t^d(\cdot | \vx^d) + r_t^d(\vx^d| \cdot )\odot g_t^d(\cdot |\vx^{1:D})
     \Big] \text{ if } \vy^d \neq \vx^d  \nonumber \\
    &=  \triangle n \beta(t) (\vy^d)^\top \Big[\vm^d - \vx^d + \big(\langle \vx^d, \vm^d  \rangle\1 - \vx^d\big)\odot g_t^d(\cdot| \vx^{1:D}) 
    \Big ] \text{ if } \vy^d \neq \vx^d \nonumber \\
    &= \triangle n \beta(t) (\vy^d)^\top \Big[\vm^d + \langle \vx^d, \vm^d  \rangle g_t^d(\cdot| \vx^{1:D}) \Big ] \text{ if } \vy^d \neq \vx^d \nonumber \\
    &\approx \triangle n \beta(t) (\vy^d)^\top \Big[\vm^d + \langle \vx^d, \vm^d  \rangle g_t^{\color{orange}\theta, d}(\cdot| \vx^{1:D}) \Big ] \text{ if } \vy^d \neq \vx^d
\end{align}

Now we define the notation $\overline{p_{\triangle n}( \cdot | \vx^{1:D})}$  to derive the distributional form of $p_{n+\triangle n|n}( \vy^d | \vx^{1:D})$. 
\begin{align}
    \overline{p^{\theta,d}_{\triangle n}( \cdot | \vx^{1:D})}:= \triangle n \beta(t)
     \Big[ 
        \big(2 - \frac{\overline{\alpha}_{t|0} \langle f_t^{\theta,d}(\vx^{1:D}), \vx^d \rangle }{\overline{\alpha}_{t|0} + (1-\overline{\alpha}_{t|0})\langle \vx^d, \vm^d \rangle }\big) \vm^d +  
        \frac{\overline{\alpha}_{t|0}}{1-\overline{\alpha}_{t|0}} f_t^{\theta,d}(\vx^{1:D})
    \Big] \odot(\1 - \vx^d) \label{eq:helper_of_corrector}
\end{align}
With the above notation, the sampling probability can be further simplified as 
\begin{align}
    p_{n+\triangle n|n}( \vy^d | \vx^{1:D}) = \text{Cat}\Big(\vy^d; 
    \overline{p^{\theta,d}_{\triangle n}( \cdot | \vx^{1:D})}
    + \big(1-\1^\top\overline{p^{\theta,d}_{\triangle n}( \cdot | \vx^{1:D})}\big)\vx^d
    \Big) \label{eq:corrector_CTMC_prob}
\end{align}
Notice that $\triangle n$ should be set small enough such that $\1^\top\overline{p^{\theta,d}_{\triangle n}( \cdot | \vx^{1:D})} \le 1$. This condition can be used to derive $\triangle n$ dynamically. In practice, we can also easily clip the scale of $\overline{p^{\theta,d}_{\triangle n}( \cdot | \vx^{1:D})}$ to 1 when  $\1^\top\overline{p^{\theta,d}_{\triangle n}( \cdot | \vx^{1:D})} > 1$ to prevent illness condition. Intuitively, $1 - \1^\top\overline{p^{\theta,d}_{\triangle n}( \cdot | \vx^{1:D})}$ defines the keeping rate of the $d$-th element during correction step, and it should be larger with increasing $t$ and $n$ during the reverse sampling and correction period. 

\begin{algorithm}
\caption{The MCMC correcting algorithm at time $t$}\label{alg:cap}
\begin{algorithmic}
    \STATE {\bfseries Input:}  The sample at time $t$ from reverse sampling, $\vz_t^{1:D}$; step size $\triangle n$; learned $f^{\theta}_t$; total steps $N$.
    \STATE \textbf{Initialize} $\vx^{1:D}_{t,0} \gets \vz_t^{1:D}$
    \FOR{$i$ from $1$ to  $N$}
        \STATE Compute $\overline{p^{\theta,d}_{\triangle n}( \cdot | \vx^{1:D}_{t,i-1})}$ from \eqref{eq:helper_of_corrector};
        \STATE $\forall d$, Draw $\vx^{d}_{t,i} \sim \text{Cat}\Big(\ \cdot\ ; 
    \overline{p^{\theta,d}_{\triangle n}( \cdot | \vx^{1:D}_{t,i-1})}
    + \big(1-\1^\top\overline{p^{\theta,d}_{\triangle n}( \cdot | \vx^{1:D}_{t,i-1})}\big)\vx^d_{t,i-1} 
    \Big)$ 
    \ENDFOR
\STATE {\bfseries Output:}  the improved (corrected) sample $\vx^{1:D}_{t,N}$
\end{algorithmic}
\end{algorithm}

\subsubsection{Noise Scheduling in Training} 
Notice that $\beta(t)=-\frac{1}{\overline{\alpha}_t}\cdot \frac{d\overline{\alpha}_t}{dt}$ as $\overline{\alpha}_{t|s} = \exp(- \int_s^t \beta(a) da)$. We now first present a general way to design the scheduling of $\overline{\alpha}_t$ based on \cite{chen2023importance}. For any continuous function $h(t)$, we define $\overline{\alpha}_t$ as the follows that satisfies $\overline{\alpha}_0 = 1 $ and $\overline{\alpha}_T = 0$:  
\begin{align}
    \overline{\alpha}_t = \frac{h(T) - h(t)}{ h(T)- h(0)}
\end{align}
We can easily derive that 
\begin{align}
    \beta(t) &= -\frac{1}{\overline{\alpha}_t}\cdot \frac{h'(t)}{h(0)-h(T)}
      = \frac{h'(t)}{h(T) - h(t)} \\
    \overline{\alpha}_{t|s} &= \exp(- \int_s^t \frac{d h(t)}{h(T)-h(t)})
    = \exp(\int_s^t \frac{d (h(T) - h(t))}{h(T)-h(t)}) = \frac{h(T)-h(t)}{h(T)-h(s)}
\end{align}

Based on the above general formulation, we now present some widely used noise schedules 

\begin{table}[H]
\fontsize{8.7}{9}\selectfont
 \setlength{\tabcolsep}{1pt}
    \caption{Widely used noise scheduling with maximum time T}\label{tab:noise_schedule}
    \centering
    \begin{tabular}{lcc|cccc}
    \toprule
    Type & Param. & Reference & $h(t)$ & $\overline{\alpha}_t$ & $\beta(t)$ & $\overline{\alpha}_{t|s}$  \\
    \midrule
    Cosine & $a=0.008 $
           & \citep{hoogeboom2021argmax} 
           & $\cos(\frac{t/T + a}{1+a}\frac{\pi}{2})$ 
           & $\frac{h(t)}{h(0)}$ 
           & $\frac{\pi \tan(\frac{t/T + a}{1+a}\frac{\pi}{2})}{2T(1+a)}$ 
           & $\frac{h(t)}{h(s)}$\\
    Linear & - 
           & \citep{DDPM} 
           & $t$ 
           & $1-\frac{t}{T} $ 
           & $\frac{1}{T-t}$ 
           & $\frac{T-t}{T-s}$ \\
    Exp. & $a, b$ 
               & \citep{campbell2022continuous} 
               & - 
               & $\exp(Ta(1-b^{\frac{t}{T}}))$ 
               & $ab^{\frac{t}{T}}\log b$ &$\exp(Ta(b^{\frac{s}{T}}-b^{\frac{t}{T}}))$\\
    \bottomrule
    \end{tabular}
\end{table}

\subsection{Experiment Details}
\label{ssec:exp}

\subsubsection{Lakh Piano Dataset Details}
\label{sssec.piano}
The Lakh pianoroll dataset contains $6,000$ training and $973$ evaluating piano sequences, with each music sequence spanning a length of $256$ in total. Each music note in the sequence can take on a value of the $128$ music notes plus $1$ additional class meaning an empty note. The music note orderings are scrambled into the same random order as described in \cite{campbell2022continuous} such that the ordinal structure of the music notes are destroyed.

For evaluation, the first $32$ notes of the $967$ evaluation sequences are given to the model, while the model is asked to generate the resting $224$ notes. Upon an analysis of the training and evaluation music sequences, we find that a total of $124$ evaluation samples can be found with at least one matching training samples that have the same $32$ dimensions. We separate out these samples and call the set \pone. Among \pone, 20 samples can be found to contain the same first $32$ notes with $2$ samples from the training sequences, three of them contain the same first $32$ notes with $4$ training samples each, one of them shares the same with $6$ training samples, and one shares the same with $8$ training samples. If the same $32$ notes appear both in the training samples and as quests for the model to provide the inferences, the model is likely to directly memorizing the remaining $224$ notes from the training set, and "parroting" music sequences according to the training samples. The rest $843$ evaluation samples that do not have matching training samples are constituent of \ptwo.

\subsubsection{Pre-training VQGAN}
\label{sssec.vqgan}
To generate images in a categorical discrete latent space, we follow the implementation of VQGAN \cite{esser2020taming} in MaskGIT \cite{chang2022maskgit}. Specifically, we use the same VQGAN setting as mentioned in \cite{sun2023scorebased}. VQGAN is a variant of Vector Quantized Variational Autoencoder (VQ-VAE) by \cite{oord2018neural}.
In our setup for CIFAR10, a VQGAN encodes an image of shape $H \times W \times 3$ to $(\frac{H}{4} \times \frac{W}{4})$ tokens with vocabulary size of $512$. For the encoder, we use three convolutional blocks, with filter sizes of $64$,$128$ and $256$, and an average pooling between each blocks. For each block, it consists of two residual blocks. After an image is encoded, the output is mapped to a token index with a codebook of $512 \times 256$. For VQGAN loss objective, the additional GAN loss and perceptual loss are added with weight $0.1$. To train a general VQGAN model that allows us to embed CIFAR10 images without overfitting, we apply data augmentation (random flipping and cropping) to the $64 \times 64$ version of ImageNet dataset [\cite{imagenet}], and train for $90$ epochs. The VQGAN is trained with Adam optimizer ($\beta_1 =0$, $\beta_2 = 0.99$), and the learning rate is linearly warmup to the peak of $1e^{-4}$ and then drops with cosine decay. 

After VQGAN is trained, we freeze the VQGAN and apply encoder to the CIFAR10 images to create $8 \times 8$ latent codes. The latent codes then flattened to a vector of size $64$ as the input of the diffusion model. Before we evaluate our diffusion methods, we will feed the generated latent codes back to the VQGAN decoder to reconstruct a sample in the image spaces. We test the effectivenss of VQGAN by trying to reconstruct the CIFAR10 dataset. The reconstruction gives a FID of $7.68$ and IS of $10.42$, using the Inception V3 Model \footnote{\url{https://github.com/openai/consistency_models/tree/main/evaluations}}.

\subsubsection{Music Generation Eval Metrics}
\label{sssec.music_metrics}
In evaluation of the conditional music generation task, we apply the following metrics to measure generation quality: 
\begin{itemize}[nosep, left=5pt,itemsep=4pt,leftmargin=1em,labelwidth=*,align=left]
\item $1$-gram Hellinger Distance (\textdownarrow) and $1$-gram Proportion of Outliers (\textdownarrow): for these two metrics, we following the same evaluation as described in \cite{campbell2022continuous} and \cite{chen2023importance}.
\item $\{2,3\}$-gram Hellinger Distance (\textdownarrow) and $\{2,3\}$-gram Proportion of Outliers (\textdownarrow): similar to n-gram models, we first convert the music sequences into tuples of neighboring nodes. Then for Hellinger Distance, we compute the distance of the empirical probabilities of conditional generated samples to the ground truth samples. The empirical probabilities are constructed based on the histograms of the neighboring tuples, with bins being all possible $\{2,3\}$-gram nodes. Similarly, for the Proportion of Outliers, we count the fractions of newly appeared tuples that are not seen in the training samples. With these metrics, we are able to capture the sequence information instead of just measuring the single node distributions.
\item Diverse Edit Distance (\textuparrow): which
accounts for the creativity/novelty of generated samples across multiple generation runs. For the  conditionally generated music samples given the same first $32$ notes, we calculate the edit distance, which is the minimum number of single-character edits (insertions, deletions, or substitutions) required to change one music sequence into the other, between each two of the generation samples. The mean and standard deviation are obtained for all edit distances between pairs. The higher the diverse edit distance, the further apart the music sequences are and more creativity are enforced in the generation process. However, there is also a trade-off between diverse edit distance and other accuracy measurements when the model processes large uncertainty about the underlying distribution.

\item Train-to-Test
Ratios (\textuparrow) for $\{1,2,3\}$-gram Hellinger as well as Proportion of Outliers, which compare the weighted distance
of a generated sample to its evaluation (test) vs. training
sequence that share the same first 32 notes. We evaluate the ratios only for \pone. Denote the evaluation ground truth set in \pone as \textit{tr} , the corresponding set of training samples as \textit{ts}, and conditionally generated samples as \textit{gs}. For the selected distance metrics \textit{dist()} (from n-gram Hellinger or n-gram Proportion of Outliers), we calculate the ratio as: $\frac{1}{\textit{dist(tr,gs)} + \textit{dist(ts,gs)}} * \frac{\textit{dist(tr,gs)}}{\textit{dist(ts,gs)}}$. Such ratio measures quantify the extent of “parroting” in \pone. The larger the ratio, the more equally distant the generated examples is from its training and evaluating set, and the less "parroting" occurs by simply memorizing all training sequences. We apply an additional coefficient to the ratio such that it will also penalize the models that provide unrealistic examples that do not conform both training and evaluation distributions.
\end{itemize}

\subsubsection{Training Details}
\label{ssec.training_details}
\textbf{\method Lakh Pianoroll Training Details.} 

For the backbone sequence transformer structure, we adopt a similar transformer model as utilized by in \tauldr-0 \cite{campbell2022continuous}. The sequence transformer is composed of several encoder blocks, where for each internal block, time is fed into the block through a FiLM layer and added to the input. The result is then fed into a multi-headed self-attention layer and fully connected layer. We use RELU for activation. At the output of self-attention layer and fully-connected layer, a dropout of 0.1 is applied. After obtaining the final embedding of each token, we feed that into
a 1-block ResNet to obtain the predicted logits. In comparison with other baseline metrics, the transformer contains 6 encoder blocks, each containing 16 attention heads, input dimension of 1024 and MLP dimension of 4096. In ablation study, we also test our methods on a smaller architecture that contains 6 encoder blocks, each containing 8 attention heads, with input dimension of 128 and MLP dimension of 1024. 

For training the pianoroll dataset, we use a batch size of $64$, a learning rate of $5e^{-4}$, with a warmup of first $25$ epochs. We adopt a constant learning rate decay scheduler, decaying the learning rate by half after every $500$ epochs. The final result is given over $3000$ epochs. We run our results with 2 A6000 GPUs. In discrete-time diffusion, we sample 1000 number of timesteps, fixing a cosine scheduler with $\alpha = 0.008$, In continuous-time diffusion, we sample time $t$ between $[0,1]$ and apply a constant scheduler with rate equals $0.007$. We maintain an exponential moving average of parameters with decay factor $0.9999$. We clip the gradient norm at a norm value of $1.0$.

\textbf{Baseline Training Details.}

For \textbf{\dpm} and \textbf{\tauldr-0}, for fair comparison, we use the same architecture, diffusion scheme and training scheme. For calculating the loss of both methods, we follow the previous literature and give $0.001$ to CE loss, and $1$ to the VLB loss. For \textbf{\rdm}, we use the official implementation of loss function in reparam multinomial diffusion\footnote{\url{https://github.com/HKUNLP/reparam-discrete-diffusion}}.We keep the same neural network architecture, and apply reweighting = 1.0, with no label smoothing, and let sampling strategy be cmlm-dep. Similar to \dpm and \method, we apply cosine scheduler with uniform noise. We remove the padding functions and sampling temperatures when adapting \rdm from language tasks to music and image generation.

For \textbf{\sedd}, we follow the official code implementation with uniform noise and loss functions\footnote{\url{https://github.com/louaaron/Score-Entropy-Discrete-Diffusion}}. While keeping the same neural architectures as \dpm and \method, we add an additional masking to the output (the same masking as in \sedd's diffusion transformer), since \sedd models the ratio of two probabilities instead of direct log probabilities. We use log-linear noise schedules with $1e^{-3}$. 

For \textbf{\sddm}, since it adopts a different architecture, directly utilizing \sddm with our experiment configurations is not feasible. Instead, we use the experiment configurations as given in the paper: the backbone structure is a hollow transformer \cite{sun2023scorebased}. Each transformer block has 6 layers with embedding size of 256, 8 attention heads and hidden dimension of 2048. The batch size is 64 and number of training steps is 2 million. The weight decay is set to $1e^{-6}$. The learning rate is at constant $1e^{-3}$. For diffusion, \sddm adopts the constant noise scheduler with a uniform rate constant of $0.04$.

\textbf{\method VQCIFAR10 Training Details.}

For image generation task, we parameterize $f_t^{\theta}(\mathbf{x}_t)$ with the same sequence transformer as mentioned before. The model is a $12$ layer transformer, where each layer has 16 attention heads, input embedding of $768$ and a hidden layer size of $3072$ for the MLPs. We use ReLU for activation. At the output of each internal block, a dropout of $0.1$ is applied. The time is input into the network through FiLM layers before the self attention block. We use a learning rate of $5e^{-4}$, with warmup of $50000$ steps, and a cosine learning rate decay scheduler, to train 2 million steps. In discrete-time diffusion, all the \method and its variants use the same cosine noise scheduler with $\alpha = 0.008$.  For continuous-time diffusion, \methodce and \methodsim apply the cosine noise scheduler with $\alpha = 0.008$. We find that \method is extremely hard to optimize due to the scale differences of coefficient $\beta(t)$, so we provide \methodsim which clips the 
$\beta(t) = \max(1, \beta(t))$. \method utilizes a constant noise scheduler with 0.007, to match the scheduler in \tauldr-0 and \sddm. We maintain an exponential moving average of parameters with decay factor $0.9999$. We clip the gradient norm at a norm value of $1.0$.

\textbf{Baseline Training Details.}

For \textbf{\tauldr-0} and \textbf{\dpm}, we again use the same transformer architecture and same training scheme. We apply a hybrid loss with cross entropy as a directly supervision, added with $0.001$ CE loss. The \tauldr-0 uses a constant rate noise scheduler of $0.007$, while \dpm uses cosine scheduler with $\alpha=0.008$. We train all models in parallel on $2$ A6000 GPUs. For \textbf{\dpm} and \textbf{\tauldr-0}, for fair comparison, we use the same architecture, diffusion scheme and training scheme. For calculating the loss of both methods, we follow the previous literature and give $0.001$ to CE loss, and $1$ to the VLB loss. For \textbf{\rdm}, we apply reweighting = 1.0, with no label smoothing, and let sampling strategy be cmlm-dep. We apply cosine scheduler with uniform noise. For \textbf{\sedd}, we use log-linear noise schedules with $1e^{-3}$.

For \textbf{\sddm}, the model uses a masked modeling, where backbone neural network is BERT-based, with 12
layers of transformers. Each layer has 12 attention heads, embedding size of 768 and hidden layer size of 3072 for MLPs. After obtaining the final embedding of each token, the output is fed that into
a 2-block ResNet to acquire the predicted logits. For diffusion, \sddm uses a constant uniform rate of $0.007$ as the noise scheduler in the forward process. In \cite{sun2023scorebased}, the number of training step is set to $700,000$, where the learning rate is warmed up to $1e^{-4}$ during the first $3\%$ steps, and then decays to 0 in a linear schedule. We extended the training to $2m$ steps, but did not observe improved performance.

\subsection{Additional Experiment Results}

\subsubsection{Full Experiment Results on Music Generation}
We present the full evaluation metrics and results, including mean and standard deviation, for Lakh Pianoroll music generation tasks, as shown in Table \ref{tab:piano_ngram_full}.

\label{ssec:full_experiment_result}
\begin{table*}[ht]
\centering
\caption{ Metrics comparing generated conditional samples and evaluating ground truths for Lakh Pianoroll. For each of n-gram Hellinger Distance (ng.-Hellinger) and Proportion of Outliers (ng.-Prop. Outlier) and Diverse Edit Distance metrics, we show mean $\pm$ std with respect to 3 generated samples. The top two are highlighted by \first{First}, \second{Second}.}
\scalebox{0.67}{
\begin{tabular}{llccccccc}
\toprule
& Method & 1g.-Hellinger(\textdownarrow) & 2g.-Hellinger(\textdownarrow)  & 3g.-Hellinger(\textdownarrow) & 1g.-Prop.Out.(\textdownarrow)  & 2g.-Prop.Out.(\textdownarrow) & 3g.-Prop.Out.(\textdownarrow) & Edit Distance(\textuparrow) \\
\midrule
\multirow{4}{*}{\scalebox{0.86}{\rotatebox{90}{Discrete-time}}} & \dpm  & 0.3982$\pm$0.0004  & 0.5303$\pm$0.0038 & 0.5918$\pm$0.0046 & 0.1209$\pm$0.0008 & 0.2532$\pm$ 0.0053 & 0.3790$\pm$0.0059 & \first{0.2950$\pm$0.0775}\\
& \rdm  & 0.4123$\pm$0.0012  & 0.5964$\pm$0.0021 & 0.6198$\pm$0.0039 & 0.1401$\pm$0.0011 & 0.2459$\pm$ 0.0023 & 0.3864$\pm$0.0030 & 0.2891$\pm$0.0690\\
\cmidrule{2-9}
& \scalebox{0.78}{\methodce} & 0.3754$\pm$0.0007 & 0.4835$\pm$0.0009 & 0.5741$\pm$0.0008 & 0.1079$\pm$0.0002 & 0.2099$\pm$0.0005 & 0.3034$\pm$0.0007 & 0.0472$\pm$0.0465  \\
& \scalebox{0.78}{\methodvlb} & 0.3790$\pm$0.0009 & 0.4640$\pm$0.0010 & \first{0.5427$\pm$0.0009} & 0.1174$\pm$0.0007 & 0.1845$\pm$0.0010 & 0.2734$\pm$0.0011 & \second{0.0828$\pm$0.0567}\\
& \method & 0.3770$\pm$0.0011 & 0.4693$\pm$0.0015 & 0.5525$\pm$0.0015 &    \second{0.1077$\pm$0.0006} & 0.1861$\pm$0.0011 & 0.2861$\pm$0.0013 & 0.0648$\pm$0.0459
\\
& \methodsim & 0.3753$\pm$0.0013 & 0.4704$\pm$0.0010 & 0.5550$\pm$0.0012 & 0.1107$\pm$0.0005  & 0.1911$\pm$0.0005 & 0.2832$\pm$0.0011 & 0.0664$\pm$0.0471 \\
\midrule
\multirow{4}{*}{\scalebox{0.86}{\rotatebox{90}{Continuous-time}}} & \sddm & 0.3759$\pm$0.0020 & 0.4856$\pm$0.0015 & 0.5773$\pm$0.0009 & 0.1101$\pm$0.0015 & 0.2059$\pm$0.0005 & 0.3409$\pm$0.0017 & 0.0606$\pm$0.0249\\
& \tauldr-0 & 0.3796$\pm$0.0009 & 0.4811$\pm$0.0008 & 0.5710$\pm$0.0007 & 0.1149$\pm$0.0008 & 0.2078$\pm$0.0014 & 0.3202$\pm$0.0014 & 0.0553$\pm$0.0637\\
& \sedd & 0.3814$\pm$0.0012 & 0.4712 $\pm$0.0020 & 0.5470 $\pm$0.0013 & \first{0.1052$\pm$0.0010} & \first{0.1823$\pm$0.0015} & 0.2709$\pm$0.0013 & 0.0632$\pm$0.0532\\
\cmidrule{2-9}
& \scalebox{0.78}{\methodce} & \second{0.3734$\pm$0.0002} & 0.4837$\pm$0.0003 & 0.5776$\pm$0.0013 & 0.1158$\pm$0.0004 & 0.2218$\pm$0.0003 & 0.3461$\pm$0.0005 & 0.0434$\pm$0.0357\\
& \scalebox{0.78}{\methodvlb} & 0.3764$\pm$0.0006 & \second{0.4620$\pm$0.0014} & 
0.5487$\pm$0.0018 & 0.1198$\pm$0.0009 & 0.1866$\pm$0.0011 & 0.2952$\pm$0.0011 & 0.0661$\pm$0.0478\\
& \method & \first{0.3735$\pm$0.0005} & \first{0.4617$\pm$0.0012} & \second{0.5432$\pm$0.0018} & 0.1123$\pm$0.0009 & \second{0.1840$\pm$0.0011} & \first{0.2704$\pm$0.0009} & 0.0661$\pm$0.0401 \\
& \methodsim & 0.3737$\pm$0.0004 & 0.4623$\pm$0.0009 & 0.5490$\pm$0.0017 & 0.1143$\pm$0.0008 & 0.1844$\pm$0.0008 & \second{0.2707$\pm$0.0011} & 0.0516$\pm$0.0365\\
\bottomrule
\end{tabular}
}
\label{tab:piano_ngram_full}
\vspace{-0.15in}
\end{table*}

\subsubsection{Study of ``Parroting'' with Ratio Metrics}
\label{ssec:study_parroting}
As shown in Table \ref{tab:piano_parroting}, we find that models with combined losses, \method and
\methodsim, achieve higher Train-to-Test ratio than pure \methodce or
\methodvlb, suggesting that the combination of two
losses can alleviate overfitting, i.e. “parroting” the training data. Overall, \tauldr-0 is a strong competitor in n-gram Hellinger Train-to-Test ratios, suggesting that its VLB focused loss optimization can also alleviate parroting and generate diverse samples.
\vspace{-0.2in}
\begin{table*}[ht!]
\centering
\caption{ Train-to-test Ratio Metrics (\textuparrow) for Lakh Pianoroll, using evaluation sequences fron \pone and the training sequences. The higher ratios indicate less "parroting" phenomenon and better generalization ability. Mean and std are calculated across 3 generated samples.The top two are highlighted by \first{First}, \second{Second}.}
\scalebox{0.67}{
\begin{tabular}{llccc|ccc}
\toprule
& Method & 1g.-Hellinger Ratio & 2g.-Hellinger Ratio& 3g.-Hellinger Ratio& 1g.-Prop.Out. Ratio& 2g.-Prop.Out. Ratio& 3g.-Prop.Out.r Ratio\\
\midrule
\multirow{5}{*}{\scalebox{0.9}{\rotatebox{90}{Discrete-time}}} & \dpm  & 1.7182$\pm$0.0500 & 1.3624$\pm$0.0411 & 1.1307$\pm$0.0304 & 5.6237$\pm$0.1463 & 3.7051$\pm$0.1162 & 2.2216$\pm$0.0616   \\
& \rdm & 1.7534$\pm$0.0928 & 1.4023$\pm$0.00701 & 1.2023$\pm$0.0392 & 5.9981$\pm$0.0497 & 3.6053$\pm$0.1532 & 2.5211$\pm$0.0256\\
\cmidrule{2-8}
& \scalebox{0.78}{\methodce} & 1.7922$\pm$0.0186 & 1.4028$\pm$0.0074 & 1.1906$\pm$0.0090 & 6.3964$\pm$0.4669 & 3.9206$\pm$0.1082 & 2.8887$\pm$0.0399  \\
& \scalebox{0.78}{\methodvlb} & 1.7758$\pm$0.0097 & 1.4123$\pm$0.0118 & 1.2011$\pm$0.0101 & 6.2024$\pm$0.2584 & 3.9093$\pm$ 0.0663 & 2.8635$\pm$0.0479 \\
& \method & 
1.8073$\pm$0.0880& 1.4749$\pm$0.0708 & 1.2078$\pm$0.0518 & \first{6.8742$\pm$0.7366} & \first{4.2448$\pm$ 0.4161} & \first{3.0830$\pm$0.2520}  \\
& \scalebox{0.78}{\methodsim} & 1.7705$\pm$0.0235 & \first{1.4959$\pm$0.0176} & 1.1928$\pm$0.0125 & \second{6.6452$\pm$0.2702} & 3.9867$\pm$0.1188 & 2.9591$\pm$0.0653 \\
\midrule
\multirow{5}{*}{\scalebox{0.9}{\rotatebox{90}{Continuous-time}}} & \sddm & 1.7872$\pm$0.0131 & 1.4516$\pm$0.0097 & 1.1812$\pm$0.0103 & 6.3019$\pm$0.2522 & 4.0652$\pm$0.0591 & 2.9014$\pm$0.0782 \\
& \tauldr-0 & \first{1.8648$\pm$0.0306} & 1.4726$\pm$0.0233 & \first{1.2497$\pm$0.0185} &  6.4586$\pm$0.3798 & 4.0863$\pm$0.2001 & 2.9651$\pm$0.1089 \\
& \sedd & 1.8230$\pm$0.0298& \second{1.4902$\pm$0.0090} & 1.2152$\pm$0.0150 & 6.4361$\pm$0.2001 & 4.0598$\pm$0.0621 &
3.0421$\pm$0.0922\\
\cmidrule{2-8}
& \scalebox{0.78}{\methodce} & 1.7951$\pm$0.0290 & 1.4023$\pm$0.0264 & 1.1474$\pm$0.0206 & 6.3712$\pm$0.4502 & 3.6232$\pm$0.1734 & 2.6379$\pm$0.0823 \\
& \scalebox{0.78}{\methodvlb} & 1.7379$\pm$0.0220 & 1.4117$\pm$0.0250 & 1.1810$\pm$0.0232 & 6.2752$\pm$0.1232 &
3.9321$\pm$0.0949 &
2.9042$\pm$0.0853 \\
& \method & \second{1.8623$\pm$0.0340}  & 1.4516$\pm$0.0077 & 1.2017$\pm$0.0096 & 6.4218$\pm$0.1739 & \second{4.0877$\pm$0.0553} & \second{3.0828$\pm$0.0502} \\
& \methodsim & 1.7970$\pm$0.0106 & 1.4898$\pm$0.0026 & \second{1.2377$\pm$0.0015} & 6.4361$\pm$0.2023 & 3.9340$\pm$0.0284 & 2.9234$\pm$0.0110\\
\bottomrule
\end{tabular}
}
\label{tab:piano_parroting}
\end{table*}

\subsubsection{Model Sizes and Loss Analysis on Piano Dataset}
\label{ssec:model_size}

Table \ref{tab:piano_modelsize} shows the evaluation metrics of different model sizes.This ablation study shows that CE loss is also preferred in combination with VLB, especially for smaller network structures, since VLB is harder to optimize alone.  

\begin{table*}[ht!]
\centering
\caption{ Metrics comparing different loss combinations and different model sizes for Lakh Pianoroll. For each of n-gram Hellinger Distance (ng.-Hellinger) and Proportion of Outliers (ng.-Prop. Outlier) metrics, we show mean $\pm$ std with respect to 3 generated samples. We use \methodvlb to denote an additional variant of our model that only uses the exact VLB loss in training. "Small" refers to the backbone transformer model that has 6 Layers, 8 Attention Heads, Input Dimension of 128 and MLP dimension of 1024. The top two are highlighted by \first{First}, \second{Second}.}
\scalebox{0.67}{
\begin{tabular}{llccccccc}
\toprule
& Method  & 1g.-Hellinger(\textdownarrow) & 2g.-Hellinger(\textdownarrow)  & 3g.-Hellinger(\textdownarrow) & 1g.-Prop.Out.(\textdownarrow)  & 2g.-Prop.Out.(\textdownarrow) & 3g.-Prop.Out.(\textdownarrow) & Edit Distance(\textuparrow) \\
\midrule
\multirow{6}{*}{\scalebox{0.95}{\rotatebox{90}{Discrete-time}}} & \scalebox{0.78}{\methodce-S}  & 0.3984$\pm$0.0006 & 0.4902 $\pm$0.0004 &0.5785$\pm$0.0004 & 0.1158$\pm$0.0002 & 0.1899$\pm$0.0006 & 0.3142$\pm$0.0005 & 0.1301$\pm$0.0613  \\
& \scalebox{0.78}{\method-S}  & 0.4011$\pm$0.0014 & 0.4902$\pm$0.0009 & 0.5707$\pm$0.0008 & 0.1215$\pm$ 0.0007 & 0.1866$\pm$0.0006 & 0.3006$\pm$0.0013 & 0.1292$\pm$0.0623  \\
& \scalebox{0.78}{\methodvlb-S}  & 0.4115$\pm$0.0001 & 0.4954$\pm$ 0.0001 & 0.5738$\pm$0.0005 & 0.1203$\pm$0.0006 & 0.1958$\pm$0.0009 & 0.3036$\pm$0.0010 & 0.2137$\pm$0.0843  \\
\cmidrule{2-9}
& \scalebox{0.78}{\methodce} & 0.3754$\pm$0.0007 & 0.4835$\pm$0.0009 & 0.5741$\pm$0.0008 &  \second{0.1079$\pm$0.0002} & 0.2099$\pm$0.0005 & 0.3034$\pm$0.0007 & 0.0472$\pm$0.0465  \\
& \scalebox{0.78}{\method} & 0.3770$\pm$0.0011 & 0.4693$\pm$0.0015 & 0.5525$\pm$0.0015 &    \first{0.1077$\pm$0.0006} & \second{0.1861$\pm$0.0011} & 0.2861$\pm$0.0013 & 0.0648$\pm$0.0459 \\
& \scalebox{0.78}{\methodvlb}  & 0.3790$\pm$0.0009& 0.4640$\pm$0.0010 & \first{0.5427$\pm$0.0009} & 0.1174$\pm$0.0007 & 0.1845$\pm$0.0010 & \second{0.2734$\pm$0.0011} & 0.0828$\pm$0.0567  \\
\midrule
\multirow{6}{*}{\scalebox{0.95}{\rotatebox{90}{Continuous-time}}} & \scalebox{0.78}{\methodce-S}  & 0.4208$\pm$0.0170 & 0.5196$\pm$0.0078 & 0.6072$\pm$0.0005& 0.1355$\pm$0.0147 & 0.2276$\pm$0.0007 & 0.3431$\pm$0.0181 & 0.2358$\pm$0.0619  \\
& \scalebox{0.78}{\method-S}  & 0.4239$\pm$0.0012& 0.5083$\pm$0.0011 & 0.5852$\pm$0.0011 & 0.1403$\pm$0.0007 & 0.2070$\pm$0.0008 & 0.2943$\pm$0.0016  & \second{0.2468$\pm$0.0907}  \\
& \scalebox{0.78}{\methodvlb-S}  & 0.4435$\pm$0.0012 & 0.5295$\pm$0.0011 & 0.6070$\pm$0.0009 & 0.1562$\pm$0.0014 & 0.2269$\pm$0.0013 & 0.3168$\pm$0.0013 & \first{0.2809$\pm$0.0866}  \\
\cmidrule{2-9}
& \scalebox{0.78}{\methodce} & \second{0.3734$\pm$0.0002} & 0.4837$\pm$0.0003 & 0.5776$\pm$0.0013 & 0.1158$\pm$0.0004 & 0.2218$\pm$0.0003 & 0.3461$\pm$0.0005 & 0.0434$\pm$0.0357\\
& \scalebox{0.78}{\method} & \first{0.3735$\pm$0.0005} & \first{0.4617$\pm$0.0012} & \second{0.5432$\pm$0.0018} & 0.1123$\pm$0.0009 & \first{0.1840$\pm$0.0011} & \first{0.2704$\pm$0.0009} & 0.0661$\pm$0.0401   \\
& \scalebox{0.78}{\methodvlb}  & 0.3764$\pm$0.0006 &\second{0.4620$\pm$0.0014} & 
0.5487$\pm$0.0018 & 0.1198$\pm$0.0009 & 0.1866$\pm$0.0011 & 0.2952$\pm$0.0011 & 0.0661$\pm$0.0478\\
\bottomrule
\end{tabular}
}
\label{tab:piano_modelsize}
\end{table*}

\subsubsection{Memory and Running-time Comparison}
\label{ssec:memory_and_runningtime}

\begin{table}[!t]
\centering
\caption{The GPU-memory, running time and number of network parameters in all methods.\method is easier to train and incurs the least GPU memory in both discrete- and continuous- time diffusions.}
\scalebox{0.78}{
\begin{tabular}{llccc|llccc}
\toprule
& Method  & Num. Parameters &Memory & Runtime  &   &Method & Num. Parameters  & Memory & Runtime \\
\midrule
\multirow{4}{*}{\scalebox{0.90}{\rotatebox{90}{Discrete-time}}} &
  \dpm  & $\sim$ 102,700,000 & 15669MiB & 93 hrs &  \multirow{4}{*}{\scalebox{0.86}{\rotatebox{90}{Continuous-time}}} 
 & \sddm & $\sim$ 12,350,000 & 85528MiB &  96 hrs \\
 & \rdm & $\sim$ 102,700,000  & 9570MiB & 86 hrs &  & \sedd & $\sim$ 102,700,000 & 4620MiB & 85 hrs \\
& \methodce & $\sim$ 102,700,000 & 9735MiB &  83 hrs&  & \tauldr-0  & $\sim$ 102,700,000 & 17669 MiB & 129 hrs \\
& \methodsim & $\sim$ 102,700,000  &  9735MiB  &  82 hrs &  & \methodce & $\sim$ 102,700,000 & 15703MiB &  93 hrs  \\
& \method & $\sim$ 102,700,000 & 9735MiB &  90 hrs & & \methodsim & $\sim$ 102,700,000 & 17620MiB & 91 hrs   \\
& & & & & &\method & $\sim$ 102,700,000  & 17620MiB & 101 hrs  \\
\bottomrule
\end{tabular}
}
\label{tab:size_time_memory}
\end{table}
Table \ref{tab:size_time_memory} provides the running time, number of parameters (in the transformer architecture) and GPU-memory. While \dpm, and \tauldr-0 use the same backbone transformer structure as \method and contain the same number of parameters in the transformer, they require a passing of transition matrices, which incur additional memory and slow down the training process. \rdm is similar to \methodce and \methodsim in memory calculation and runtime. \sddm requires a special architecture and considerably larger memory during training. \sedd requires significantly less memory than \method (continuous version), and training time is almost equivalent to the discrete-version \methodce and \methodsim. \sedd's loss calculation removes the redundant additional forward sampling, yet the loss calculation is less exact tham \method in continuous-time.

\subsubsection{Qualitative Examples for VQCIFAR10 Image Generation}

We provide 80 reconstructed images VQGAN decoder in Fig. \ref{fig:vq_recons}. Most examples are easy to recognize from one of the 10 classes in the original CIFAR10: airplanes, cars, birds, cats, deer, dogs, frogs, horses, ships, and trucks.

\label{ssec:imagesamples}

\begin{figure}[h]
    \centering   \includegraphics[width=0.98\columnwidth]{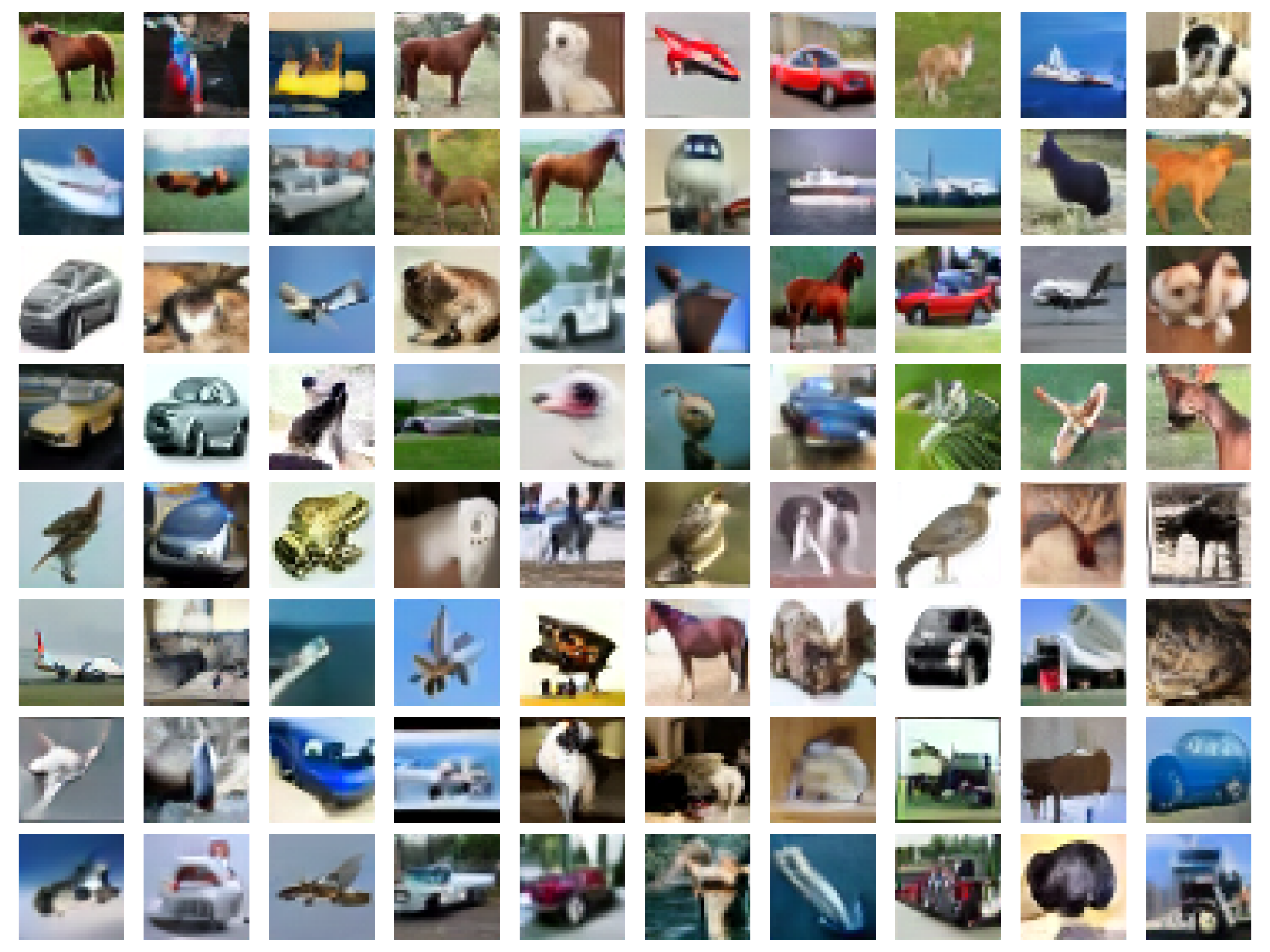}
    \caption{Example image samples generated by \methodsim as trained on \cifar. Most images are easy to recognize as being from one of the 10 classes in CIFAR10.}
    \label{fig:vq_recons}
\end{figure}

\subsubsection{MCMC for Discrete-time and Continuous-time Discrete Diffusion}
\label{ssec:mcmc}
To demonstrate the effectiveness of MCMC corrector steps, we take the top performing methods in discrete and continuous time diffusion models (for discrete time, \methodsim, and for continuous time \methodce) and show improved quality metrics over   generated images after MCMC corrector is applied. Due to extensive time required for MCMC corrector step, we could only conduct evaluation over $5,000$ images, and thus the results are not comparable to the main VQCIFAR10 result in Table \ref{tab:vqcifar10}.  We set the number of generation steps to be $100$ and use MCMC corrector on the last $10$, $20$ generation timesteps, respectively. 

From the results shown in Table \ref{tab:mcmc_result} and Table \ref{tab:mcmc_result_cont}, we can see that MCMC corrector can significantly improve the quality of the generated samples. Specifically, for discrete-time generation, IS is showing a significant improvement than the sampling process without the MCMC corrector. For continuous-time case, both IS and FID scores are improving from the baseline (without MCMC).

\noindent 
\begin{minipage}{0.48\textwidth}
\begin{table}[H]
\centering
\caption{Image gen. quality w.r.t. Inception Score (IS) and the Frechet Inception Dist. (FID) over 5,000 samples unconditionally generated by \methodsim in discrete-time case. MCMC corrector is conducted for the last 10,20 timesteps over 100 sampling steps.}
\scalebox{0.8}{
\begin{tabular}{lcc}
\toprule
 MCMC Configuration & IS (\textuparrow) & FID (\textdownarrow) \\
 \midrule
 Without MCMC & 9.01 & 19.79\\ 
 \midrule
 $\Delta n=0.001,N=2$, Start Steps:10 & 9.28 & 18.46\\
 $\Delta n=0.001,N=2$, Start Steps:20 & \textbf{9.59} & 18.36 \\
 $\Delta n=0.005,N=2$, Start Steps:10 & 9.43 & 18.26 \\
 $\Delta n=0.005,N=2$, Start Steps:20 & 9.43 & 20.37 \\
 $\Delta n=0.001,N=5$, Start Steps:10 & 9.29 & \textbf{18.02}\\
 $\Delta n=0.001,N=5$, Start Steps:20 & 9.47 & 18.56 \\
 $\Delta n=0.002,N=5$, Start Steps:10 & 9.35 & 18.18 \\
 $\Delta n=0.002,N=5$, Start Steps:20 & 9.48 & 20.37 \\
\bottomrule
\end{tabular}
}
\label{tab:mcmc_result}
\end{table}
\end{minipage}%
\hfill 
\begin{minipage}{0.48\textwidth}
\begin{table}[H]
\centering
\caption{Image gen. quality w.r.t. Inception Score (IS) and the Frechet Inception Dist. (FID) over 5,000 samples  generated by \methodce in continuous-time case. MCMC corrector is conducted for the last 10,20 timesteps over 100 sampling steps.}
\scalebox{0.8}{
\begin{tabular}{lcc}
\toprule
 MCMC Configuration & IS (\textuparrow) & FID (\textdownarrow) \\
 \midrule
 Without MCMC & 8.98 & 19.19 \\ 
 \midrule
 $\Delta n=0.001,N=2$, Start Steps:10 & 9.12 & 17.62 \\
 $\Delta n=0.001,N=2$, Start Steps:20 & \textbf{9.23} & 17.35 \\
 $\Delta n=0.005,N=2$, Start Steps:10 & 9.01 & 17.71 \\
 $\Delta n=0.005,N=2$, Start Steps:20 & \textbf{9.23} & 17.58 \\
 $\Delta n=0.001,N=5$, Start Steps:10 & 9.03 & \textbf{17.26} \\
 $\Delta n=0.001,N=5$, Start Steps:20 & 9.00 & 17.42 \\
 $\Delta n=0.002,N=5$, Start Steps:10 & 9.16 & 17.98 \\
 $\Delta n=0.002,N=5$, Start Steps:20 & 8.97 & 17.83 \\
\bottomrule
\end{tabular}
}
\label{tab:mcmc_result_cont}
\end{table}
\end{minipage}

\subsubsection{Sampling Time Discrete-time and Continuous-time Discrete Diffusion}
\label{ssec:sampling_time}
In Table \ref{tab:sampling_time}, we have calculated the time required for baselines and our models to sample 50,000 VQ-encoded images with 1000 timesteps on a single A6000 GPU. In addition, we calculate the time required when \method and \tauldr need MCMC corrector steps (=5) in sampling. We compare our results in the following tables. In terms of sampling time, \rdm yields the most sampling time due to its re-routing sampling scheme. \sedd and \tauldr are similar in sampling time, due to the tau-leaping scheme. \method falls short of \sddm, which requires a special NN architecture and different code base(jaxlib) than other methods. But when considering MCMC sampling, \method brings faster sampling time than tau-leaping in \tauldr, while allowing sampling for both continuous-time and discrete-time models.

\begin{table}[h]
\centering
\caption{Sample times and MCMC sampling steps for different models.}
\vspace{0.05in}
\scalebox{0.9}{
\begin{tabular}{l|l|l}
\hline
\textbf{Model} & \textbf{Sample Time} & \textbf{MCMC Sampling Step=5} \\ 
\hline
\rdm & $\sim$23 hours & - \\ 
\dpm & $\sim$8 hours & - \\ 
\tauldr & $\sim$9 hours & $\sim$2 days, 20 hours \\ 
\sddm & $\sim$7 hours & - \\
\sedd & $\sim$8 hours & - \\ \hline
\method (continuous version) & $\sim$8 hours & $\sim$2 days, 12 hours \\ 
\method (discrete version) & $\sim$8 hours & $\sim$2 days, 12 hours \\ \hline
\end{tabular}
}
\label{tab:sampling_time}
\end{table}


\end{document}